\documentclass{article}

\usepackage{arxiv}

\usepackage[utf8]{inputenc} % allow utf-8 input
\usepackage[T1]{fontenc}    % use 8-bit T1 fonts
\usepackage{hyperref}       % hyperlinks
\usepackage{url}            % simple URL typesetting
\usepackage{booktabs}       % professional-quality tables
\usepackage{amsfonts}       % blackboard math symbols
\usepackage{nicefrac}       % compact symbols for 1/2, etc.
\usepackage{microtype}      % microtypography
\usepackage{amsthm}
\usepackage{mathtools}
\usepackage{amssymb}
\usepackage{todonotes}

\usepackage[shortlabels]{enumitem}
\usepackage{color}
\usepackage{comment}
\usepackage{subcaption}
\usepackage{graphicx}
\usepackage{makecell}
\usepackage{amsmath}
\usepackage{tikz-cd}
\usepackage[sort&compress,square,comma,authoryear]{natbib}
\usepackage[capitalise]{cleveref}
\usepackage{float} 
\usepackage{wrapfig}

\newtheorem{theorem}{Theorem}
\newtheorem{definition}[theorem]{Definition}

\newtheorem{lemma}[theorem]{Lemma}
\newtheorem{remark}[theorem]{Remark}

\numberwithin{theorem}{section}

\graphicspath{{./pics/}}

\renewcommand{\arraystretch}{1.55}

\newcommand{\R}{\mathbb{R}}

\newcommand{\Lip}{\text{Lip}}

\newcommand{\req}[1]{Eq.\,(\ref{#1})}

\title{Structure-preserving GANs}

\author{  Jeremiah Birrell\\
    Department of Mathematics and Statistics\\
  University of Massachusetts Amherst\\
  Amherst, MA 01003,  USA \\
  \texttt{birrell@math.umass.edu} \\
   \And
    Markos A. Katsoulakis\\
    Department of Mathematics and Statistics\\
  University of Massachusetts Amherst\\
  Amherst, MA 01003,  USA \\
  \texttt{markos@math.umass.edu} \\
\And
    Luc Rey-Bellet\\
    Department of Mathematics and Statistics\\
  University of Massachusetts Amherst\\
  Amherst, MA 01003,  USA \\
  \texttt{luc@math.umass.edu} 
  \And
    Wei Zhu\\
    Department of Mathematics and Statistics\\
  University of Massachusetts Amherst\\
  Amherst, MA 01003,  USA \\
  \texttt{zhu@math.umass.edu} 
}

\begin{document}
\maketitle

\begin{abstract}

Generative adversarial networks (GANs), a class of distribution-learning methods based on a two-player game between a generator and a discriminator, can generally be formulated as a minmax problem based on the  variational representation of a divergence between the unknown  and the generated distributions. We introduce   structure-preserving GANs as a data-efficient framework for  learning distributions with additional structure such as group symmetry, by developing new variational representations for divergences. Our theory shows that we can reduce the discriminator space to its projection on the invariant discriminator space, using the conditional expectation with respect to the $\sigma$-algebra associated to the underlying structure. 
In addition, we prove that the discriminator space reduction must be accompanied by a careful design of  structured generators, as flawed designs  may easily  lead to a catastrophic “mode collapse” of the learned distribution. We contextualize our framework by building  symmetry-preserving GANs for   distributions with intrinsic group symmetry, and  demonstrate that both players, namely the equivariant generator and invariant discriminator, play important but distinct  roles in the learning process. Empirical experiments and ablation studies across a broad range of data sets, including real-world medical imaging, validate our theory, and show our proposed methods achieve significantly improved sample fidelity and diversity---almost an order of magnitude measured in Fr\'echet Inception Distance---especially in the small data regime.

%  We develop the theory of $(f,\Gamma)$-divergences when either the test-function space or the probability distributions are invariant under a group of symmetries.  
\end{abstract}

% keywords can be removed

\section{Introduction}
Since their introduction by  \citet{goodfellow2014generative}, generative adversarial networks (GANs) have become a burgeoning domain in distribution learning with a diverse range of innovative applications \cite{karras2019style,Zhu_2019_CVPR,mustafa_cosmogan_2019,yi2019generative}. Mathematically, the minmax game between a generator and a discriminator in GAN can typically be formulated as minimizing a divergence--- or other notions of ``distance"---with a variational representation between the unknown and the generated distributions. Such formulations, however, do not make prior \textit{structural} assumptions on the probability measures, making them sub-optimal in sample efficiency when learning distributions with intrinsic structures, such as the (rotation) group symmetry for medical images without preferred orientation; see Figure~\ref{fig:anhir_images_small}.

\begin{wrapfigure}{R}{0.5\textwidth}
  \begin{center}
    \includegraphics[width=0.48\textwidth]{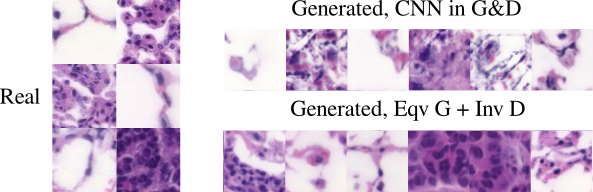}
  \end{center}
  \caption{Real and GAN generated ANHIR images dyed with the H\&E stain [cf.~\cref{sec:medical_data}]. Left panel: real images. Right panels: randomly selected $D_2^L$-GAN generated samples after 40,000 generator iterations. Top right panel: \texttt{CNN G\&D}, i.e., the baseline model. Bottom right panel: \texttt{Eqv G} + \texttt{Inv D}, i.e., our proposed framework contextualized in learning group-invariant distributions. More images are available in Appendix \ref{app:additional_results}.}\label{fig:anhir_images_small}
    \begin{center}
    \includegraphics[width=0.48\textwidth]{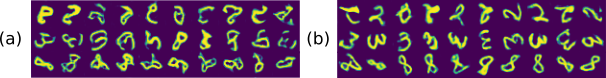}
  \end{center}
  \caption{Randomly generated digits 2, 3 and 8 by GANs trained on the rotated MNIST images using \textbf{1\% (600)} training samples. (a): the baseline CNN model. (b): our proposed framework %contextualized in 
for learning group-invariant distributions.}
\label{fig:rotmnist_digits_2_intro}
\end{wrapfigure}

We introduce, in this work, the \textit{structure-preserving GANs}, a data-efficient framework for learning probability measures with embedded structures, by developing new variational representations for divergences between structured distributions. We demonstrate that efficient adversarial learning can be achieved by reducing the discriminator space to its projection onto its invariant subspace, using the conditional expectation with respect to the $\sigma$-algebra associated to the underlying structure; such practice, which is rigorously justified by our theory and generally applicable to a broad range of variational divergences, acts effectively as an unbiased regularization to prevent discriminator overfitting, a common challenge for GAN optimization in the limited data regime  \cite{NEURIPS2020_55479c55}. Furthermore, our theory suggests  that the discriminator space reduction must be accompanied by  \textit{correctly} building generators sharing the same probabilistic structure, as the lack of which may easily lead to ``mode collapse" in the trained model, i.e., the generated distribution samples only a subset of the support of the data source [cf. Figure~\ref{fig:toy_200_alpha_2_2d} (2nd row)].

As an example, we contextualize our framework by building symmetry-preserving GANs for learning distributions with group symmetry. Unlike  prior empirical work, our choice of equivariant generators and invariant discriminators is theoretically founded, and we show (theoretically and empirically) how flawed design of equivariant generators  results easily in the aforementioned mode collapse [cf.~Figure~\ref{fig:toy_200_alpha_2_2d} (4th row)]. Experiments and ablation studies over synthetic and real-world data sets validate our theory, disentangle the contribution of the structural priors on generators and discriminators, and demonstrate the significant outperformance of our framework in terms of both sample quality and diversity---in some cases almost by an order of magnitude measured in Fr\'echet Inception Distance; see Figure~\ref{fig:anhir_images_small} and \ref{fig:rotmnist_digits_2_intro} for a visual illustration.

In Section \ref{sec:related} we will discuss several related approaches to equivariant GANs.  We provide background on GANs, variational representations of divergences, and group equivariance in Section \ref{sec:background}. Section \ref{sec:theory} contains our main theoretical results regarding divergences between structured distributions. Section \ref{app:more_f_gamma} contains additional theoretical results specific to a primal formulation of $(f,\Gamma)$-divergences for structured distributions, building on the inf-convolution formulation of general $(f,\Gamma)$-divergences in  \citet{Birrell:f-Gamma}. Finally, our experiments on synthetic and real-world data sets are found in Section \ref{sec:experiments}.

\section{Related Work}
\label{sec:related}
Neural generation of group-invariant distributions has mainly been proposed in a flow-based framework \cite{kohler2019equivariant,kohler2020equivariant, rezende2019equivariant, liu2019graph,bilovs2021scalable,boyda2021sampling,garcia2021n}. Such models typically use an equivariant normalizing-flow to  push-forward a group-invariant prior distribution to a complex  invariant target. In the context of GANs,  \citet{EquivariantGAN} intuitively replace the 2D convolutions with group convolutions \cite{cohen2016group} to build group-equivariant GANs; however, their empirical study has not been justified by theory, and their incomplete design of the equivariant generator may easily lead to a ``mode collapse" of the learned model; see the discussion of Theorem~\ref{thm:Gamma_invariant}. {  The existence of symmetry can often be deduced from prior or domain knowledge of the distribution, e.g., the rotation symmetry for medical images without preferred orientation. Symmetry detection from data has also been studied in recent works such as \cite{NEURIPS2021_148148d6}. }  When extended from group symmetry to probability structures induced from other operators, our work is also related to GAN-assisted coarse-graining (CG) for molecular dynamics \cite{voth_durumeric2019adversarial} and cosmology \cite{mustafa_cosmogan_2019, feder2020nonlinearcosmo}; see the end of \cref{sec:discriminator} for a detailed discussion.

\section{Background and Motivation}\label{sec:background}
\subsection{Generative adversarial networks}
\label{sec:gan_background}
Generative adversarial networks are a class of methods in learning a probability distribution via a zero-sum game between a generator and a discriminator \cite{goodfellow2014generative,arjovsky2017wasserstein,nowozin2016f,NIPS2017_892c3b1c}. Specifically, let $(X,\mathcal{M})$ be a measurable space, and $\mathcal{P}(X)$ be the set of probability measures on $X$; given a target distribution $Q\in \mathcal{P}(X)$, the original GAN proposed by  \citet{goodfellow2014generative} learns $Q$ by solving
\begin{align}
\label{eq:original_gan}
    \inf_{g\in G} D(Q\|P_g) & =\inf_{g\in G}\sup_{\gamma\in \Gamma} H(\gamma; Q, P_g),
    %\nonumber
    %& =\inf_{g\in G}\sup_{\gamma\in \Gamma} E_Q[\log \gamma]+E_{P_g}[\log(1-\gamma)],
\end{align}
where $H(\gamma; Q, P_g) = E_Q[\log \gamma]+E_{P_g}[\log(1-\gamma)]$. The map $g:Z\to X$ in \req{eq:original_gan} is called a \textit{generator}, which maps a random vector $z\in Z$ to a generated sample $g(z)\in X$, pushing forward the noise distribution $P\in \mathcal{P}(Z)$ (typically a Gaussian) to a probability measure $P_g\in \mathcal{P}(X)$, i.e., 
\begin{align}
    P_g \coloneqq g_*P \coloneqq P\circ g^{-1}\,.
\end{align} 
The test function $\gamma:X\to \R$ is called a \textit{discriminator}, which aims to differentiate the source distribution $Q$ and the generated probability measure $P_g$ by maximizing $H(\gamma;Q, P_g)$. The spaces $G$ and $\Gamma$, respectively, of generators and discriminators are both parametrized by neural networks (NNs), and the solution of model \eqref{eq:original_gan} is the best generator $g\in G$ that is able to ``fool" all discriminators $\gamma\in\Gamma$ by achieving the smallest $D(Q\|P_g)$, which measures the ``dissimilarity" between $Q$ and $P_g$.

\subsection{Variational representations for divergences}
\label{sec:variational_divergence_background}
Mathematically, most GANs can be formulated as minimizing the ``distance" between the probability measures $Q$ and $P_g$ according to some divergence or probability metric with a variational representation $\sup_{\gamma\in \Gamma}H(\gamma; Q, P_g)$ as in \eqref{eq:original_gan}. We hereby recast these formulations in a unified but flexible mathematical framework that will prove essential in \cref{sec:discriminator}. Let $\mathcal{M}(X)$ be the space of measurable functions on $X$ and  $\mathcal{M}_b(X)$ be the subspace of bounded measurable functions. Given an objective functional $H:\mathcal{M}(X)^n \times \mathcal{P}(X)\times\mathcal{P}(X)\to [-\infty,\infty]$ and a test function space $\Gamma\subset \mathcal{M}(X)^n$, $n\in\mathbb{Z}^+$, we define 
\begin{align}\label{eq:general_divergence_IPM}
    D_H^\Gamma(Q\|P)=\sup_{\gamma\in\Gamma}H(\gamma;Q,P)\,.
\end{align}
$D_H^\Gamma$ is called a \textit{divergence} if $D_H^\Gamma\geq 0$ and $D_H^\Gamma(Q\|P)=0$ if and only if $Q=P$, hence providing a notion of  ``distance" between probability measures. Variational representations of the form \eqref{eq:general_divergence_IPM} have been widely used, including in GANs \cite{goodfellow2014generative, nowozin2016f, arjovsky2017wasserstein}, divergence estimation \cite{Nguyen_2007, Nguyen_Full_2010, Ruderman,doi:10.1137/20M1368926}, determining independence through mutual information estimation \cite{MINE_paper},  uncertainty quantification of stochastic processes \cite{chowdhary_dupuis_2013,DKPP}, bounding risk  in  probably approximately correct (PAC) learning \cite{10.1145/307400.307435,10.1145/267460.267466,catoni2008pac},  parameter estimation \cite{Broniatowski_Keziou}, statistical mechanics and interacting particles  \cite{Kipnis:99}, and large deviations \cite{dupuis2011weak}. It is known that formula~\eqref{eq:general_divergence_IPM} includes, through suitable choices of  functional $H(\gamma;Q,P)$ and function space $\Gamma$, many divergences and probability metrics. Below we list several classes of examples.

\textbf{(a) $f$-divergences.} Let $f:[0, \infty) \to \mathbb{R}$ be convex and lower semi-continuous (LSC), with  $f(1)=0$ and $f$ strictly convex at  $x=1$. The $f$-divergence between $Q$ and $P$ is
\begin{align}
\label{eq:f_divergence}
D_f(Q\|P)=&\sup_{\gamma\in \mathcal{M}_b(X)}\{ E_Q[\gamma]-E_P[f^*(\gamma)]\},
% = & \sup_{\gamma\in \mathcal{M}_b(X)}\{ E_Q[\gamma]-\Lambda_f^P[\gamma]\},
\end{align}
where $f^*$ denotes the Legendre transform of $f$. Some notable examples of the $f$-divergences include the Kullback-Leibler (KL) divergence and the family of $\alpha$-divergences, which are constructed, respectively, from
\begin{align}
\label{eq:KL_alpha}
    f_{KL} = x\log x, ~~f_\alpha(x) = \frac{x^{\alpha}-1}{\alpha(\alpha-1)}, \alpha>0, \alpha\neq 1.
\end{align}
The flexibility of $f$ allows one to tailor the divergence to the data source, e.g., for heavy tailed data. However, the formula \eqref{eq:f_divergence} becomes $D_f(Q\|P) = \infty$ when $Q$ is not absolutely continuous with respect to $P$, limiting its efficacy in comparing distributions with low-dimensional support.

\textbf{(b) $\Gamma$-Integral Probability Metrics (IPMs).} Given $\Gamma\subset \mathcal{M}_b(X)$, the $\Gamma$-IPM between $Q$ and $P$ is defined as
\begin{align}\label{eq:IPM:def}
    W^\Gamma(Q,P)=\sup_{\gamma\in\Gamma}\{E_Q[\gamma]-E_P[\gamma]\}.
\end{align}
Apart from the Wasserstein metric when $\Gamma=\Lip^1(X)$ (the space of 1-Lipschitz functions), examples of IPMs also include the total variation metric, the Dudley metric, and maximum mean discrepancy (MMD) \cite{muller_1997, Gretton_review_IPM}. With suitable choices of $\Gamma$, IPMs
are able to meaningfully compare not-absolutely continuous distributions, but they could potentially fail at comparing distributions with heavy tails \cite{Birrell:f-Gamma}.

\textbf{(c) $(f,\Gamma)$-divergences.} This class of divergences was  introduced by  \citet{Birrell:f-Gamma} and they subsume both $f$-divergences and $\Gamma$-IPMs. Given a function $f$ satisfying the same condition as in the definition of the $f$-divergence and $\Gamma\subset \mathcal{M}_b(X)$, the $(f,\Gamma)$-divergence is defined as
\begin{align}
\label{eq:Df_Gamma_def1}
D_f^\Gamma(Q\|P)=&\sup_{\gamma\in\Gamma}\left\{E_Q[\gamma]-\Lambda_f^P[\gamma]\right\},
\end{align}
where $\Lambda_f^P[\gamma]=\inf_{\nu\in\mathbb{R}}\left\{\nu+E_P[f^*(\gamma-\nu)]\right\}$. One can verify that \eqref{eq:Df_Gamma_def1} includes as a special case the $f$-divergence \eqref{eq:f_divergence} when $\Gamma = \mathcal{M}_b(X)$, and it is demonstrated in \cite{Birrell:f-Gamma} that under suitable assumptions on $\Gamma$  we have
\begin{align}\label{eq:Df_Gamma_nonAC}
  0\le  D_f^\Gamma(Q\|P)\le \min\{D_f(Q\|P), W^\Gamma(Q, P) \}\,,
\end{align}
making $D_f^\Gamma$ suitable to compare not-absolutely continuous distributions with heavy tails. An example of the $(f,\Gamma)$-divergence is the Lipschitz $\alpha$-divergence,
\begin{align}
\label{eq:lip_alpha_divergence}
    D_\alpha^L(Q\|P) = \sup_{\gamma\in \Lip_b^L(X)}\{E_Q[\gamma] - \Lambda_{f_\alpha}^P[\gamma] \},
\end{align}
where $f=f_\alpha$ as in \req{eq:KL_alpha}, and $\Gamma = \Lip_b^L(X)$ is the space of bounded $L$-Lipschitz functions.

\textbf{(d) Sinkhorn divergences.} The Wasserstein  metric associated with a cost function $c:X^2\to \R^+$ has the variational representation
\begin{align}
    W_c^\Gamma(Q, P)= \sup_{\gamma = (\gamma_1, \gamma_2)\in \Gamma}\{E_P[\gamma_1] + E_Q[\gamma_2]\}\,,
\end{align}
where $\Gamma = \{(\gamma_1, \gamma_2)\in C(X)^2: \gamma_1(x)+\gamma_2(y)\le c(x, y) , \forall x,y\in X\}$, and $C(X)$ is the space of continuous functions on $X$. The Sinkhorn divergence is given by
\begin{align}\label{eq:sinkhorn_def}
    \mathcal{SD}^\Gamma_{c,\epsilon}(Q,P)= W^\Gamma_{c,\epsilon}(Q,P)- \frac{W^\Gamma_{c,\epsilon}(Q,Q)
    +W^\Gamma_{c,\epsilon}(P,P)}{2}\,,
\end{align}
where $W^\Gamma_{c,\epsilon}(Q,P)$ is the entropic regularization of the Wasserstein metrics \cite{Cuturi:stoch_optim_OT:2016},
\begin{align}
W^\Gamma_{c,\epsilon}(Q,P)= \sup_{\gamma=(\gamma_1, \gamma_2)\in \Gamma} \left\{ E_P[\gamma_1] + E_Q[\gamma_2] - \epsilon E_{P\times Q}\left[\exp\left(\frac{\gamma_1 \oplus\gamma_2 - c}{\epsilon}\right) \right] + \epsilon \right\} \,,
\end{align}
where   $\gamma_1\oplus\gamma_2(x,y)\coloneqq\gamma_1(x)+\gamma_2(y)$ and  $\Gamma=C_b(X)\times C_b(X)$ ($C_b(X)$ denotes the space of bounded continuous functions on $X$).

We refer to \cref{app:variational_divergence} for a more detailed discussion of the variational divergences introduced above. In all the aforementioned examples, the choice of the discriminator space, $\Gamma$, is a defining characteristic of the divergence. We will explain, in \cref{sec:discriminator}, a general framework, i.e., the structure-preserving GANs, for incorporating added structural knowledge of the probability distributions or data sets into the choice of $\Gamma$, leading to enhanced performance and data efficiency in adversarial learning of structured distributions.

\subsection{Group invariance and equivariance}
\label{sec:group_background}
We  first introduce the structure-preserving GAN framework in the context of learning distributions with group symmetry. Here we explain the necessary background and notations.  {  We emphasize that the focus of this work is not to discuss the group-invariance properties of probability measures (which can be found in, e.g., \cite{schindler2003measures}), but to understand how to incorporate such structural
information into the generator/discriminator of GANs such that invariant probability distributions can be learned more efficiently. However, we first  require the following background and notations.} 

\textbf{Groups and group actions.} A \textit{group} is a set $\Sigma$ equipped with a binary operator, the group product, satisfying the axioms of associativity, identity, and invertibility. Given a group $\Sigma$ and a set $X$, a map $T:\Sigma\times X\to X$ is called a \textit{group action} if, for all $\sigma\in\Sigma$, $T_\sigma\coloneqq T(\sigma,\cdot):X\to X$ is an automorphism on $X$, and $T_{\sigma_1}\circ T_{\sigma_2} = T_{\sigma_1\cdot\sigma_2}, \forall \sigma_1, \sigma_2\in\Sigma$. In this paper, we will consider mainly the 2D rotation group $SO(2)=\{R_{\theta}\in \R^{2\times 2}: \theta\in \R\}$ and roto-reflection group $O(2)=\{R_{m, \theta}\in \R^{2\times 2}: m\in \mathbb{Z}, \theta\in \R\}$, where $R_\theta$ is the 2D rotation matrix of angle $\theta$, and $R_{m,\theta}$ has a further reflection if $m\equiv 1~(\text{mod}~2)$. The natural actions of $SO(2)$ and $O(2)$ on $\R^2$ are matrix multiplications, which can be lifted to actions on the space of ($k$-channel) planar signals $L^2(\R^2,\R^k)$, e.g., RGB images. More specifically, when $\Sigma$ is  $SO(2)$ or $O(2)$  let $T_\sigma f(x) \coloneqq f(\sigma^{-1}x), ~\forall \sigma\in \Sigma, \forall f\in L^2(\R^2,\R^k).$ We will also consider the finite subgroups $C_n$, $D_n$, respectively, of $SO(2)$ and $O(2)$, with the rotation angles $\theta$ restricted to integer multiples of $2\pi/n$.

\textbf{Group equivariance and invariance.} Let $T^Z$ and $T^X$, respectively, be $\Sigma$-actions on the spaces $Z$ and $X$. A map $g:Z\to X$ is called $\Sigma$-\textit{equivariant} if 
\begin{align}
    T_\sigma^X\circ g = g\circ T_\sigma^Z, \forall \sigma\in \Sigma\,.
\end{align} 
A map $\gamma:X\to Y$ is called $\Sigma$-\textit{invariant} if \begin{align}
    \gamma\circ T_\sigma^X = \gamma, \forall \sigma\in \Sigma\,.
\end{align}
Invariance is thus a special case of equivariance after equipping $Y$ with the action $T^Y_\sigma y \equiv y, \forall \sigma\in \Sigma$. In the context of NNs, achieving equivariance/invariance via group-equivariant CNNs (G-CNNs) has been well-studied, and we refer the reader to \cite{NEURIPS2019_b9cfe8b6,NEURIPS2019_45d6637b} for a complete theory of G-CNNs.

Let $G$ be a collection of measurable maps $g:Z\to X$. We denote its subset of $\Sigma$-equivariant maps as \begin{align}
G_{\Sigma}^{\text{eqv}}\coloneqq\{g \in G: T_\sigma^X\circ g = g\circ T_\sigma^Z, ~\forall \sigma\in \Sigma \}\,.    
\end{align}
Similarly, let $\Gamma$ be a set of measurable functions $\gamma :X\to Y$; its subset, $\Gamma_\Sigma^{\text{inv}}$, of $\Sigma$-invariant functions is defined as
\begin{align}
    \label{eq:sigmainvariantGamma}
    \Gamma^{\text{inv}}_\Sigma\coloneqq\{\gamma \in \Gamma: \gamma \circ T_\sigma^X=\gamma,  ~\forall \sigma\in \Sigma \}\,.
\end{align}
The function space $\Gamma$ is called \textit{closed under $\Sigma$} if
\begin{align}\label{def:Sigma_closed_Gamma}
    \gamma\circ T_\sigma^X \in \Gamma, ~~\forall \sigma\in \Sigma, ~~\forall \gamma\in\Gamma\,.
\end{align}
Finally, a probability measure $P\in\mathcal{P}(X)$ is called \textit{$\Sigma$-invariant} if $P = P\circ (T_\sigma^X)^{-1}$ for all $\sigma\in\Sigma$. For instance, the distribution of medical images without orientation preference should be $SO(2)$-invariant; see Figure~\ref{fig:anhir_images_small}. The set of all $\Sigma$-invariant distributions on $X$ is denoted as
\begin{equation}\label{eq:sigma_inv_meas}
    \mathcal{P}_{\Sigma}(X)\coloneqq\{P \in \mathcal{P}(X): P~\text{is}~ \Sigma\text{-invariant}\}.
\end{equation}

\subsection{Definition of Haar measure on $\Sigma$ and the symmetrization operators $S_\Sigma$ and
$S^\Sigma$}
We will make frequent use of the symmetrization operators, on both functions and probability distributions, that are induced by a group action on $X$. These are  constructed using the unique Haar probability measure, $\mu_\Sigma$, of a compact Hausdorff topological group $\Sigma$ (see, e.g., Chapter 11 in \citet{folland2013real}). Intuitively the Haar measure is the  uniform probability measure on $\Sigma$. Mathematically, this is expressed via the invariance of Haar measure under  group multiplication, $ \mu_\Sigma(\sigma \cdot E)=\mu_\Sigma(E\cdot \sigma)=\mu_\Sigma(E)$ for all $\sigma \in \Sigma$ and all Borel sets $E\subset \Sigma$. This is a generalization of the invariance of Lebesgue measure  under translations and rotations. The Haar measure can be used to define symmetrization operators on both functions and probability measures as follows (going forward, we assume the group action is measurable).

\textbf{Symmetrization of functions:} $S_\Sigma:\mathcal{M}_b(X)\to\mathcal{M}_b(X)$,
\begin{align}\label{eq:symmetroperator:def}    &S_\Sigma[\gamma](x)\coloneqq\int_\Sigma \gamma (T_{\sigma'}(x)) \mu_\Sigma(d\sigma')=E_{\mu_\Sigma}[\gamma \circ T_{\sigma'}(x)]\,.
\end{align}

\textbf{Symmetrization  of probability measures (dual operator):} $S^{\Sigma}: \mathcal{P}(X)\to \mathcal{P}(X)$, defined for  $\gamma \in \mathcal{M}_b(X)$ by
\begin{align}
    &E_{S^\Sigma[P]}\gamma\coloneqq\int_XS_\Sigma[\gamma](x)dP(x)=E_PS_\Sigma[\gamma]\,.
\end{align}
\begin{remark}{Sampling from $S^\Sigma[P]$:} If $x_i, i=1,..., N$ are samples from $P$, 
and $\sigma_j, j=1,...,M$ are samples from the Haar probability measure $\mu_\Sigma$ (all independent) then $T_{\sigma_j}(x_i)$ are   samples from $S^\Sigma[P]$. If $P$ is $\Sigma$-invariant then the use of $T_{\sigma_j}(x_i)$  can be viewed as a form of data augmentation.
\end{remark}
The following lemma provides several key properties of the symmetrization operators.  \begin{lemma}\label{lemma:symmetr_ops}
(a) The symmetrization operator  $S_\Sigma:\mathcal{M}_b(X)\to\mathcal{M}_b(X)$ is a projection operator onto the subspace of $\Sigma$-invariant bounded measurable functions
\begin{equation}\label{eq:sigma_inv_funct_bounded_app}
    \mathcal{M}_{b,\Sigma}^{\text{inv}}(X)\coloneqq\{\gamma \in \mathcal{M}_b(X): \gamma \circ T_\sigma=\gamma\,\, \mbox{ for all } \,\, \sigma \in \Sigma\}\, ,
\end{equation}

in the sense that 
\begin{enumerate}
    \item $S_\Sigma[\mathcal{M}_b(X)]=\mathcal{M}_{b,\Sigma}^{\text{inv}}(X)$,
    \item $S_\Sigma \circ S_\Sigma=S_\Sigma$.
\end{enumerate}
Moreover, 
\begin{align}\label{eq:S_gamma_T}
S_\Sigma[\gamma\circ T_\sigma]=S_\Sigma[\gamma]    
\end{align}
for all $\gamma\in \mathcal{M}_b(X)$, $\sigma\in\Sigma$.
\smallskip

\noindent(b) The symmetrization operator  $S^\Sigma:\mathcal{P}(X)\to\mathcal{P}(X)$ is a projection operator onto the subset of $\Sigma$-invariant probability measures 
\begin{equation}\label{eq:sigma_inv_meas_app}
    \mathcal{P}_{\Sigma}(X)\coloneqq\{P \in \mathcal{P}(X): P \circ T^{-1}_\sigma=P\,\, \mbox{ for all } \,\, \sigma \in \Sigma\}\, ,
\end{equation}
in the sense that 
\begin{enumerate}
    \item $S^\Sigma[\mathcal{P}(X)]=\mathcal{P}_{\Sigma}(X)$,
    \item $S^\Sigma \circ S^\Sigma=S^\Sigma$.
\end{enumerate}
\smallskip
\noindent(c) $S_\Sigma$ is the conditional expectation operator with respect to the $\sigma$-algebra of $\Sigma$-invariant sets.  More specifically, for all $\gamma\in\mathcal{M}_b(X)$, $P\in\mathcal{P}_\Sigma(X)$ we have
\begin{equation}\label{eq:symmetroperator:conditionalexp}
    S_\Sigma[\gamma]=E_P[\gamma|\mathcal{M}_\Sigma]\,,
\end{equation} where $\mathcal{M}_\Sigma$ is the $\sigma$-algebra of $\Sigma$-invariant sets, 
\begin{align}
    \mathcal{M}_\Sigma\coloneqq\{\text{Measurable sets }B\subset X: T_\sigma (B)=B\,\,\, \mbox{ for all $\sigma \in \Sigma$}\}\,.
\end{align}

\end{lemma}
\begin{proof} We will need the following invariance property of integrals with respect to Haar measure, which can be proven using the invariance of Haar measure under left and right group multiplication:
\begin{align}\label{eq:Haar_int_inv}
  \int_\Sigma h(\sigma \cdot \sigma')d\mu_\Sigma(\sigma') 
 =&\int_\Sigma h(\sigma' \cdot \sigma)d\mu_\Sigma(\sigma') = \int_\Sigma h(\sigma')d\mu_\Sigma(\sigma')\,.
\end{align}

(a)  If $\gamma \in \mathcal{M}_b(X)$ then $\gamma'=S_\Sigma[\gamma]
\in \mathcal{M}_{b,\Sigma}^{\text{inv}}(X)$ by applying \eqref{eq:Haar_int_inv}  with $h(\sigma)\coloneqq \gamma \circ T_\sigma(x)$, $x\in X$.
Indeed we have 
\begin{align*}
\gamma'\circ T_\sigma(x) =& \int \gamma ( T_{\sigma'} ( T_\sigma(x))) d\mu_\Sigma(\sigma') =\int h(\sigma^\prime\cdot\sigma)\mu_\Sigma(d\sigma^\prime)=\int h(\sigma^\prime)\mu_\Sigma(d\sigma^\prime)
= \gamma'(x)\, .
\end{align*}
Furthermore any $\gamma \in \mathcal{M}_{b,\Sigma}^{\text{inv}}(X)$ belongs to the range of $S_\Sigma$ since $\gamma \circ T_\sigma=\gamma$ for all $\sigma \in\Sigma$ implies that $\gamma=S_\Sigma[\gamma]$. This also shows that
$S_\Sigma\circ S_\Sigma=S_\Sigma$. Finally, for $\gamma\in\mathcal{M}_b(X)$, $\sigma\in\Sigma$, $x\in X$ we can compute
\begin{align*}
&S_\Sigma[\gamma\circ T_\sigma](x)= \int \gamma(T_{\sigma\cdot \sigma^\prime}(x))\mu_\Sigma(d\sigma^\prime)
=\int \gamma(T_\sigma^\prime(x))\mu_\Sigma(d\sigma^\prime)=S_\Sigma[\gamma](x)\,,
\end{align*}
where we again used the  invariance property of integrals with respect to Haar measure \eqref{eq:Haar_int_inv}.

\noindent(b) For $P\in\mathcal{P}(X)$, $\gamma\in\mathcal{M}_b(X)$, and $\sigma\in\Sigma$ we can use \eqref{eq:S_gamma_T} to compute
\begin{align*}
    \int \gamma dS^\Sigma[P]\circ T_\sigma^{-1}=\int \gamma\circ T_\sigma dS^\Sigma[P]=\int S_\Sigma[\gamma\circ T_\sigma] dP=\int  S_\Sigma[\gamma] dP=\int \gamma dS^\Sigma[P]\,.\end{align*}
This holds for all $\gamma\in\mathcal{M}_b(X)$, hence $S^\Sigma[P]\circ T_\sigma^{-1}=S^\Sigma[P]$ for all $\sigma\in\Sigma$.  Therefore $S^\Sigma[P]\in \mathcal{P}_\Sigma(X)$.  Conversely, if $P\in\mathcal{P}_\Sigma(X)$
then $E_P[\gamma\circ T_\sigma]=E_P[\gamma]$ for all $\sigma \in \Sigma$ and $\gamma \in \mathcal{M}_b(X)$ and thus, by Fubini's theorem, $E_P[S_\Sigma[\gamma]]=E_P[\gamma]$. Hence $S^\Sigma[P]=P$ and so $P\in S^\Sigma[\mathcal{P}]$. This completes the proof that $S^\Sigma[\mathcal{P}(X)]=\mathcal{P}_\Sigma(X)$. Combining these calculations it is also clear that $S^\Sigma\circ S^\Sigma=S^\Sigma$.

\noindent(c)  Let $\gamma\in\mathcal{M}_b(X)$ and $P\in\mathcal{P}_\Sigma(X)$. From part (a) we know that $S_\Sigma[\gamma]\in\mathcal{M}_{b,\Sigma}^{\text{inv}}(X)$ and from this it is straightforward to show that $S_\Sigma[\gamma]$ is $\mathcal{M}_\Sigma$-measurable.  Now fix $A\in\mathcal{M}_\Sigma$ and note that $1_A\circ T_\sigma=1_A$ for all $\sigma\in\Sigma$ (where $1_A$ denotes the indicator function for $A$). Using this fact together with $S^\Sigma[P]=P$ (see part (b)) we can compute
\begin{align*}
    \int  S_\Sigma[\gamma] 1_AdP=&\int \int \gamma\circ T_{\sigma^\prime} 1_A \mu_\Sigma(d\sigma^\prime)dP=\int \int ( \gamma1_A)\circ T_{\sigma^\prime} \mu_\Sigma(d\sigma^\prime)dP=\int  S_\Sigma[ \gamma1_A] dP    =\int \gamma 1_A  dS^\Sigma[P]\\=&\int  \gamma 1_AdP\,.
\end{align*}
This proves $S_\Sigma[\gamma]=E_P[\gamma|\mathcal{M}_\Sigma]$ by the definition of conditional expectation.
\end{proof}
 Lemma~\ref{lemma:symmetr_ops}  implies that since  $S_\Sigma, S^\Sigma$ are projections onto  $\mathcal{M}_{b,\Sigma}^{\text{inv}}$, 
$\mathcal{P}_{\Sigma}(X)$ respectively, they are necessarily \textit{structure-preserving}, namely here  symmetry-preserving. We discuss a general concept of structure-preserving operators at the end of Section~\ref{sec:discriminator}.

%%%%%%%%%%%%%%%%%%%%%%%%%%%%%%%

\section{Theory}\label{sec:theory}

We present in this section our theory for structure-preserving GANs. The results are first stated for the special case of learning group-invariant distributions. We then extend the theory to a general class of structure-preserving operators.

\subsection{Invariant discriminator theorem }\label{sec:discriminator}
We demonstrate under assumptions outlined below and for broad  classes of divergences and probability metrics that for $\Sigma$-invariant  probability measures $P, Q$ we can restrict the test function space $\Gamma$ (discriminator space in GANs) in \eqref{eq:general_divergence_IPM}  to the subset of $\Sigma$-invariant functions, $\Gamma^{\text{inv}}_\Sigma$ [cf.~Eq.~\eqref{eq:sigmainvariantGamma}], without changing the divergence/probability metric, i.e.,
  \begin{align}
     D_H^\Gamma(Q\|P)= D_H^{\Gamma^{\text{inv}}_\Sigma}(Q\|P) \quad \text{for all }\,Q,P\in\mathcal{P}_\Sigma\,.
 \end{align}
     The space $\Gamma^{\text{inv}}_\Sigma$ is
   a much ``smaller" and more efficient  discriminator  space to optimize over  in the proposed GANs. We rigorously formulate our results in the following theorem, which first considers the $(f, \Gamma)$ divergence \eqref{eq:Df_Gamma_def1}, the $\Gamma$-IPM \eqref{eq:IPM:def}, and the Sinkhorn divergence \eqref{eq:sinkhorn_def}.
   %though the proof  can be adapted to other divergences that are expressed in terms of an appropriate variational formula. 
\begin{theorem}\label{thm:invariant_discriminator1}
 If $S_\Sigma[\Gamma] \subset \Gamma$ and the probability measures $P, Q$ are $\Sigma$-invariant then 
     \begin{equation}\label{eq:main_thm:fgammadiv}
         D^\Gamma(Q\|P)=D^{\Gamma^{\text{inv}}_\Sigma}(Q\|P)\,,
     \end{equation}
     where $D^\Gamma$ is an $(f,\Gamma$)-divergence or a $\Gamma$-IPM. \req{eq:main_thm:fgammadiv} also holds for  Sinkhorn divergences if the cost is $\Sigma$-invariant (i.e.,  $c(T_\sigma(x),T_\sigma(y))=c(x,y)$ for all $\sigma\in\Sigma$, $x,y\in X$).
\end{theorem}
\begin{proof} We first prove the Theorem for $(f,\Gamma)$-divergences.  Start by using Jensen's inequality and the convexity of the Legendre transform $f^*$ to obtain
\begin{align*}
    &f^*(S_\Sigma[\gamma](x)-\nu)=f^*\left(\int \big(\gamma(T_\sigma(x))-\nu\big)\mu_\Sigma(d\sigma)\right)\\
    \leq& \int f^*(\gamma(T_\sigma(x))-\nu)\mu_\Sigma(d\sigma)=S_\Sigma[f^*(\gamma(x)-\nu)]
\end{align*}
for all $\gamma \in\mathcal{M}_b(X)$. Therefore
\begin{align*}
    D_f^{S_\Sigma[\Gamma]}(Q\|P)
    =&\sup_{\gamma\in\Gamma,\nu\in\mathbb{R}}\{ E_Q[S_\Sigma[\gamma]]-\nu-E_P[f^*(S_\Sigma[\gamma]-\nu)]\}\\
    \geq &\sup_{\gamma\in \Gamma,\nu\in\mathbb{R}}\{ E_Q[S_\Sigma[\gamma]-\nu]-E_P[S_\Sigma[f^*(\gamma-\nu)]]\} \notag\\
         =&\sup_{\gamma\in \Gamma,\nu\in\mathbb{R}}\left\{ E_Q[\gamma] -\nu- E_P[ f^*(\gamma-\nu)]\right\}=D_f^\Gamma(Q\|P)\,,\notag
\end{align*}
where in the next to last equality we use Lemma \ref{lemma:symmetr_ops}(c) together with the assumptions  $P,Q\in\mathcal{P}_\Sigma(X)$ to conclude   %$S^\Sigma[P]=P$ 
$E_P[S_\Sigma[f^*(\gamma-\nu)]]=E_P[ f^*(\gamma-\nu)]$ and $E_Q[S_\Sigma[\gamma]]=E_Q[\gamma]$. Hence we obtain  $D_f^\Gamma(Q\|P)\le D_f^{S_\Sigma[\Gamma]}(Q\|P)$. Combining this with $S_\Sigma[\Gamma] \subset \Gamma$ and \eqref{eq:Df_Gamma_def1} we obtain $D_f^{S_\Sigma[\Gamma]}(Q\|P)=D_f^\Gamma(Q\|P)$.  We conclude  by showing that $S_\Sigma[\Gamma] \subset \Gamma$  implies $S_\Sigma[\Gamma]= \Gamma^{\text{inv}}_\Sigma$.
First, if $\gamma \in \Gamma^{\text{inv}}_\Sigma$ then $S_\Sigma[\gamma]=\gamma$, therefore $\Gamma^{\text{inv}}_\Sigma \subset S_\Sigma[\Gamma]$.
Conversely, since $\Gamma \subset \mathcal{M}_b(X)$, the functions in $S_\Sigma[\Gamma]$ are $\Sigma$-invariant (see Lemma~\ref{lemma:symmetr_ops}). We assumed $S_\Sigma[\Gamma] \subset \Gamma$, hence $S_\Sigma[\Gamma] \subset \Gamma^{\text{inv}}_\Sigma$.

The proof for $\Gamma$-IPMs is similar, but does not require Jensen's inequality due to the linearity of the objective functional in $\gamma$. Hence the hypothesis  $S_\Sigma[\Gamma] \subset \Gamma$ is not necessary to obtain $W^\Gamma(Q,P)=W^{S_\Sigma[\Gamma]}(Q,P)$.  The proof for Sinkhorn divergences follows similar steps as for the $(f,\Gamma)$-divergences; see  Appendix \ref{app:Sinkhorn_proof} for details.
\end{proof}
 
{  Theorem \ref{thm:invariant_discriminator1} suggests that the discriminator space reduction effectively acts as an unbiased regularization
to prevent discriminator overfitting, a common challenge
for GAN optimization  in the small data regime. Using invariant discriminators can thus improve the data-efficiency of the model; this will be empirically verified in Tables \ref{tab:fid_rotmnist} - \ref{table:3}.}

\textbf{Examples satisfying the key condition $S_\Sigma[\Gamma] \subset \Gamma$ of Theorem~\ref{thm:invariant_discriminator1}
}
\begin{enumerate}[itemsep=0mm]
\item First we consider the standard $f$-divergence \eqref{eq:f_divergence} between two $\Sigma$-invariant probability measures $P$ and $Q$. The identity $S_\Sigma[\mathcal{M}_b(X)]=\mathcal{M}_{b,\Sigma}^{\text{inv}}(X)$ from Lemma~\ref{lemma:symmetr_ops} implies that the functions space  can be restricted  to the $\Sigma$-invariant bounded functions $\mathcal{M}_{b,\Sigma}^{\text{inv}}(X)$, giving rise to an
$(f, \Gamma)$-divergence \eqref{eq:Df_Gamma_def1} with $\Gamma=\mathcal{M}_{b,\Sigma}^{\text{inv}}(X)$, i.e., $D_f(Q\|P)=D_f^{\mathcal{M}_{b,\Sigma}^{\text{inv}}(X)}(Q\|P)$.

\item If the group $\Sigma$ is finite and the function space $\Gamma \subset \mathcal{M}_b(X)$ is convex and  closed  under $\Sigma$ in the sense of \eqref{def:Sigma_closed_Gamma}, then 
$
S_\Sigma[\Gamma] \subset \Gamma\, ,
$
as readily follows from the definition   \eqref{eq:symmetroperator:def}. Our implemented examples in Section \ref{sec:experiments} fall under this category.

\item The space of 1-Lipschitz functions on a metric space $(X,d)$, assuming the action is  $1$-Lipschitz, i.e., $d(T_\sigma(x),T_\sigma(y))\leq d(x,y)$ for all $\sigma\in\Sigma$, $x,y\in X$.

\item  The unit ball in an appropriate RKHS; see Lemma \ref{lemma:RKHS_SH}.

\item  More generally, if $\Gamma$ is convex and closed in the weak topology on $\Gamma$ induced by integration against finite signed measures; see Lemma \ref{lemma:closed_S_Sigma} for a proof.

%\item If possible, but not necessary: Comments on closures, limiting cases, ``maximal" $\Gamma$? Spectral normalization, ReLU.

\end{enumerate}

\subsubsection{Extension to general objective functionals}
Next we show how the  proof of Theorem \ref{thm:invariant_discriminator1} can be generalized to a wider variety of objective functionals. This result will utilize a certain topology on the space of bounded measurable functions which we describe in the following definition.
\begin{definition}\label{def:weak_topology}
Let $V$ be a subspace of $\mathcal{M}_b(X)^n$, $n\in\mathbb{Z}^+$, and $M(X)$ be the set of finite signed  measures on $X$.  For $\nu\in M(X)^n$  we define $\tau_\nu:V\to\mathbb{R}$ by $\tau_\nu(\gamma)\coloneqq\sum_{i=1}^n\int \gamma^i d\nu_i$ and we let $\mathcal{T}=\{\tau_\nu:\nu\in M(X)^n\}$.  $\mathcal{T}$ is a separating vector space of linear functionals on $V$ and we equip $V$ with the weak topology from $\mathcal{T}$ (i.e., the weakest topology on $V$ for which every $\tau\in \mathcal{T}$ is continuous). This makes $V$ a locally convex topological vector space with dual space $V^*=\mathcal{T}$; see Theorem 3.10 in \cite{rudin2006functional}.  In the following we will abbreviate this by saying that {\em $V$ has the $M(X)$-topology}.
\end{definition}
\begin{theorem}\label{thm:general}
Let $V$ be a subspace of $\mathcal{M}_b(X)^n$, $n\in\mathbb{Z}^+$, that is closed under $\Sigma$ in the sense of \eqref{def:Sigma_closed_Gamma} and satisfies $S_\Sigma[V]\subset V$.   Given an objective functional $H:V\times \mathcal{P}(X)\times\mathcal{P}(X)\to[-\infty,\infty)$ and a test function space $\Gamma\subset V$ we define
\begin{align}
    D_H^\Gamma(Q\|P)\coloneqq \sup_{\gamma\in\Gamma}H(\gamma;Q,P)\,.
\end{align}
If  $H(\cdot;Q,P)$ is concave and upper semi-continuous (USC) in the $M(X)$-topology on $V$ (see Definition \ref{def:weak_topology}) and \begin{align}
     H(\gamma\circ T_\sigma;Q,P)=H(\gamma;Q\circ T_\sigma^{-1},P\circ T_\sigma^{-1})
 \end{align}
 for all $\sigma\in\Sigma$, $\gamma\in V$, and $Q,P\in\mathcal{P}(X)$ then for all $\Sigma$-invariant $Q,P$ we have
 \begin{align}\label{eq:inv_DH_bound}
     D_H^\Gamma(Q\|P)\leq D_H^{S_\Sigma[\Gamma]}(Q\|P)\,.
 \end{align}
 If, in addition, $S_\Sigma[\Gamma]\subset\Gamma$ then $S_\Sigma[\Gamma]= \Gamma^{\text{inv}}_\Sigma$ and
  \begin{align}
     D_H^\Gamma(Q\|P)= D_H^{\Gamma^{\text{inv}}_\Sigma}(Q\|P)\,.
 \end{align}
\end{theorem}
\begin{remark}
See Section \ref{sec:SGamma_Gamma} for conditions implying $S_\Sigma[\Gamma]\subset\Gamma$.
\end{remark}
\begin{proof}
Fix $\gamma\in\Gamma$ and $\Sigma$-invariant $Q,P$. Define $G\coloneqq -H(\cdot;Q,P)$ and note that $G:V\to(-\infty,\infty]$ is LSC and convex. Convex conjugate duality (see the Fenchel-Moreau Theorem, e.g., Theorem 2.3.6 in \citet{bot2009duality}) and Fubini's theorem then imply
\begin{align*}
    G(S_\Sigma[\gamma])
    =&\sup_{\nu\in M(X)^n}\{\tau_\nu(S_\Sigma[\gamma])-G^*(\tau_\nu)\}\notag\\
        =&\sup_{\nu\in M(X)^n}\{\sum_i\int S_\Sigma[\gamma^i]d\nu_i-G^*(\tau_\nu)\}\notag\\   =&\sup_{\nu\in M(X)^n}\{\int\sum_i \int \gamma^i\circ T_\sigma d\nu_i-G^*(\tau_\nu)\mu_\Sigma(d\sigma)\}\notag\\
        =&\sup_{\nu\in M(X)^n}\{\int\tau_\nu(\gamma\circ T_\sigma)-G^*(\tau_\nu)\mu_\Sigma(d\sigma)\}\notag        \leq \int G(\gamma\circ T_\sigma)\mu_\Sigma(d\sigma)\,.
\end{align*}
We can use our assumptions to compute
\begin{align*}
G(\gamma\circ T_\sigma)=&-H(\gamma\circ T_\sigma;Q,P)\\
=&-H(\gamma;Q\circ T_\sigma^{-1},P\circ T_\sigma^{-1})\notag\\
=&-H(\gamma;Q,P)\notag
\end{align*}
and hence we obtain
\begin{align*}
    H(S_\Sigma[\gamma];Q,P)\geq H(\gamma;Q,P)\,. 
\end{align*}
Taking the supremum over $\gamma\in\Gamma$ gives \eqref{eq:inv_DH_bound}.  If   $S_\Sigma[\Gamma]\subset\Gamma$ then we clearly have the bound $D_H^{S_\Sigma[\Gamma]}\leq D_H^\Gamma$ and hence $D_H^{S_\Sigma[\Gamma]}= D_H^\Gamma$.  The equality $S_\Sigma[\Gamma]= \Gamma^{\text{inv}}_\Sigma$ was shown in the proof of Theorem \ref{thm:invariant_discriminator1} and so we are done.
\end{proof}

Theorem \ref{thm:general} applies to many classes of divergences, some of which we have not yet discussed. For example:
\begin{enumerate}
\item Integral probability metrics and MMD \eqref{eq:IPM:def}; see \cite{muller_1997, Gretton_review_IPM}.
\item $(f,\Gamma)$ divergences \eqref{eq:Df_Gamma_def1}; concavity and USC of the objective functional follows Proposition B.8 in \cite{Birrell:f-Gamma}.
\item Sinkhorn divergences \eqref{eq:sinkhorn_def};  concavity and USC of the objective functional follows Lemma B.7 in \cite{Birrell:f-Gamma}.
\item R{\'e}nyi divergence for $\alpha\in(0,1)$; see Theorem 3.1 in \cite{doi:10.1137/20M1368926}.
\item The Kullback-Leibler Approximate Lower bound Estimator (KALE); see Definition 1 in \cite{2021arXiv210608929G}.
\end{enumerate}

\subsubsection{Extension to other structure-preserving operators}

Let  $K_x(dx^\prime)$ be a probability kernel from $X$ to $X$ and define  $S_K:\mathcal{M}_b(X)\mapsto \mathcal{M}_{b}(X)$ by $S_K[f](x)\coloneqq\int f(x^\prime)K_x(dx^\prime)$.   $K$ also defines a dual map $S^K:\mathcal{P}(X)\to \mathcal{P}(X)$, $S^K[P]\coloneqq \int K_x(\cdot) P(dx)$. Let $\mathcal{P}_K(X)$ be the set of $K$-invariant probability measures, i.e., 
\begin{align}
\mathcal{P}_K(X)=\{P\in\mathcal{P}(X):S^K[P]=P\}\,.
\end{align}
In this setting we have the following generalization of Theorem \ref{thm:invariant_discriminator1}.
\begin{theorem}\label{thm:inv_disc_gen_K}
If $\Gamma\subset\mathcal{M}_b(X)$ such that $S_K[\Gamma]\subset\Gamma$ and $Q,P\in\mathcal{P}_K(X)$ then
\begin{align}\label{app:eq:S_K_general}
D^\Gamma(Q\|P)=D^{S_K[\Gamma]}(Q\|P)\,,
\end{align}
where $D^\Gamma$ is an $(f,\Gamma)$-divergence or a $\Gamma$-IPM. It also holds for the Sinkhorn divergence if $S_K[c(\cdot,y)]=c(\cdot,y)$ and $S_K[c(x,\cdot)]=c(x,\cdot)$ for all $x,y\in X$.

In addition, if $S_K$ is a projection (i.e., $S_K\circ S_K=S_K$) then $S_K[\Gamma]=\Gamma_K^{\text{inv}}$ where
where $\Gamma_K^{\text{inv}}\coloneqq\{\gamma\in\Gamma:S_K[\gamma]=\gamma\}$.
\end{theorem}
\begin{remark}
Theorem \ref{thm:inv_disc_gen_K} is an instance of the data processing inequality; see Theorem 2.21 in \cite{Birrell:f-Gamma}.
\end{remark}
\begin{proof}
We prove \eqref{app:eq:S_K_general} for $(f,\Gamma)$-divergences.  The proofs for $\Gamma$-IPMs and Sinkhorn divergences are similar. We note that for $\Gamma$-IPMs, \eqref{app:eq:S_K_general} does not require the assumption $S_K[\Gamma]\subset\Gamma$.

Fix $Q,P\in\mathcal{P}_K(X)$ and use Jensen's inequality along with the $K$-invariance of $Q$ and $P$ to compute
\begin{align*}
D^{S_K[\Gamma]}_f(Q\|P)
=&\sup_{\gamma\in\Gamma,\nu\in\mathbb{R}}\{E_Q[S_K[\gamma]-\nu]-E_P[f^*(S_K[\gamma]-\nu)]\}\\
=&\sup_{\gamma\in\Gamma,\nu\in\mathbb{R}}\{E_Q[S_K[\gamma-\nu]]-E_P[f^*(\int (\gamma(x^\prime)-\nu)K_x(dx^\prime))]\}\\
\geq &  \sup_{\gamma\in\Gamma,\nu\in\mathbb{R}}\{E_Q[S_K[\gamma-\nu]] -E_P[\int f^*(\gamma(x^\prime)-\nu)K_x(dx^\prime))]\}\\
= &  \sup_{\gamma\in\Gamma,\nu\in\mathbb{R}}\{E_{S^K[Q]}[\gamma-\nu] -E_{S^K[P]}[ f^*(\gamma-\nu)]\}\\
= &  \sup_{\gamma\in\Gamma,\nu\in\mathbb{R}}\{E_{Q}[\gamma-\nu] -E_{P}[ f^*(\gamma-\nu)]\}=D_f^\Gamma(Q\|P)\,.
\end{align*}
Therefore $D^{S_K[\Gamma]}_f(Q\|P)\geq D_f^\Gamma(Q\|P)$. Note that this computation is a special case of the proof of the  data processing inequality for $(f,\Gamma$)-divergences; see  Theorem 2.21 in \cite{Birrell:f-Gamma}. The assumption  $S_K[\Gamma]\subset\Gamma$  implies the reverse inequality, hence we conclude $D^{S_K[\Gamma]}_f(Q\|P)= D_f^\Gamma(Q\|P)$.  

Now suppose $S_K\circ S_K=S_K$.   If $\gamma=S_K[\gamma^\prime]\in S_K[\Gamma]$ then $S_K[\gamma]=S_K[S_K[\gamma^\prime]]=S_K[\gamma^\prime]=\gamma$.  This, together with the assumption that $S_K[\Gamma]\subset\Gamma$ implies  $\gamma\in\Gamma_K^{\text{inv}}$.  Conversely, if $\gamma\in \Gamma_K^{\text{inv}}$ then $\gamma=S_K[\gamma]\in S_K[\Gamma]$ by the definition of $\Gamma_K^{\text{inv}}$.  This completes the proof.
\end{proof}
Conditional expectations, $S_{K}[f]\coloneqq E_P[f|\mathcal{A}]$, are a special case of Theorem \ref{thm:inv_disc_gen_K} with the kernel being a regular conditional probability, $K=P(\cdot|\mathcal{A})$. Here $\Gamma_K^{\text{inv}}$ is the set of $\mathcal{A}$-measurable functions in $\Gamma$, which can be significantly ``smaller" than $\Gamma$. The case where $\mathcal{A}=\sigma(\xi)$ for some random variable $\xi$ has particular importance in coarse graining of molecular dynamics \cite{Noid:Review:CG, Voth2018:Review}, as we will detail below. The result for $\Sigma$-invariant measures, Theorem \ref{thm:invariant_discriminator1}, is also special case of Theorem \ref{thm:inv_disc_gen_K}, where the kernel is $K_x=\mu_\Sigma\circ R_x^{-1}$, $R_x(\sigma)\coloneqq T_\sigma(x)$. Alternatively, Lemma \ref{lemma:symmetr_ops} (c) shows $S_\Sigma$ can  be written as a conditional expectation.

\paragraph{Coarse-graining and structure-preserving operators}

Here we show how to apply our structure preserving formalism, Theorem \ref{thm:inv_disc_gen_K}, in the context of coarse-graining.  We refer to the reviews \cite{Noid:Review:CG, Voth2018:Review} for fundamental concepts in  the coarse-graining of molecular systems.
Mathematically, a coarse-graining of the state space $X$ is given by a measurable (non-invertible) map 
$$
\xi: X \to Y
$$
where $y=\xi(x)$ are thought of as the coarse variables and $Y$ as a space of significantly less complexity than $X$. If $\mathcal{A}=\sigma(\xi)$ is the $\sigma$-algebra generated by the coarse-graining map $\xi$ then a function is measurable with respect to $\mathcal{A}$ if it is constant on every level set $\xi^{-1}(y)$. 

To complete the description of the coarse-graining one selects a kernel $K_y(dx)$, which in the coarse-graining literature is called the back-mapping.  The kernel $K_y(dx)$ describes the conditional distribution of the fully resolved state $x \in \xi^{-1}(y)$, conditioned on the coarse-grained  state $y=\xi(x)$, namely $K_y(dx)=P(dx|y)$; in particular $K_y(dx)$ is supported on the set $\xi^{-1}(y)$. The kernel induces naturally a projection $S_K:\mathcal{M}_b(X)\to \mathcal{M}_b(X)$ given by

% 2/ Add refs to the CG review \cite{Noid:Review:CG} for fundamental concepts and 
% for \textit{conditional GAN}-based reconstructions (called back-mapping in the CG literature) refer to \cite{li2020backmapping,stieffenhofer_adversarial_2021}]}
$$
S_K[f](x)\,=\, \int_{\xi^{-1}(y)} f(x') K_y(dx')\, \quad \textrm{ for any } x\in \xi^{-1}(y)
$$
and, by construction, $S_K[f](x)$ is $\mathcal{A}$-measurable. If a measure is $S^K$-invariant, i.e., $S^K[P]=P$, then it is uniquely determined by its value on $\mathcal{A}$, in other words it is completely specified by a probability measure $Q \in \mathcal{P}(Y)$ on the coarse variable $y=\xi(x)$. We refer to such a $Q$ as a ``coarse-grained" probability measure. Once a coarse-grained measure is constructed on $Y$, see \cite{Noid:Review:CG, Voth2018:Review} for a rich  array of such methods, it can be  then ``reconstructed" as a measure on $X$ by the kernel $K_y(dx)$ as $P(dx)=K_y(dx)Q(dy)$. For example, if we take $X$ and $Y$ to be  discrete sets we can chose the trivial (uniform) reconstruction kernel with density $k_y(x)=\delta_x(\xi^{-1}(y))\frac{1}{|\xi^{-1}{(y)}|}$ and any coarse-grained measure with density $q(y)$ on the coarse variables $y$ is reconstructed on $X$ as a probability density on $X$:
$$
p(x)=\delta_{x}(\xi^{-1}(y))\frac{1}{|\xi^{-1}(y)|}q(y)\, , \quad \mbox{where}  \quad y=\xi(x)\, , x \in X\, .
$$

Finally, we note that %the choice of a 
back-mappings $K_y(dx)=P(dx|y)$ in coarse-graining---being probabilities conditioned on the coarse variables---can be constructed, to great accuracy,  as generative models  using \textit{conditional GANs}, see \cite{li2020backmapping,stieffenhofer_adversarial_2021}.

\subsection{Equivariant generator theorem}

Theorem~\ref{thm:invariant_discriminator1} provides the theoretical justification for reducing the discriminator space $\Gamma$ to its $\Sigma$-invariant subset $\Gamma_\Sigma^\text{inv}$  when the source $Q$ and the generated measure $P_g$ are \textbf{both}  $\Sigma$-invariant. Our next theorem, however, shows that such practice could easily lead to ``mode collapse" if  one of the two distributions is \textbf{not} $\Sigma$-invariant, see  Figure~\ref{fig:toy_200_alpha_2_2d}; the proof is deferred to \cref{app:proofs}.

\begin{theorem}
\label{thm:Gamma_invariant}
Let $S_\Sigma[\Gamma]\subset \Gamma$ and $P, Q\in \mathcal{P}(X)$, i.e., not necessarily $\Sigma$-invariant. We have
 \begin{align}\label{eq:collapse_thm:fgammadiv}
     &D^{\Gamma_\Sigma^\text{inv}}(Q\|P)=D^{\Gamma}(S^\Sigma[Q]\|S^\Sigma[P])\,,
    \end{align}
 where $D^\Gamma$ is an $(f,\Gamma)$-divergence or a $\Gamma$-IPM.
 \end{theorem}
 \begin{remark}
The analogous result for the Sinkhorn divergences also holds if the cost is separately $\Sigma$-invariant in each variable, i.e., $c(T_\sigma(x),y)=c(x,y)$ and $c(x,T_\sigma(y))=c(x,y)$ for all $\sigma\in\Sigma$, $x,y\in X$. Though this is not satisfied by most   commonly used cost functions and actions one can always enforce it by replacing the cost function $c$ with the symmetrized cost
\begin{align}
     c_\Sigma(x,y)\coloneqq\int\int c(T_\sigma(x),T_{\sigma^\prime}(y))\mu_\Sigma(d\sigma)\mu_\Sigma(d\sigma^\prime)\,.
\end{align}
\end{remark}
\begin{proof}
We prove the result for $(f,\Gamma)$-divergences; the proof for $\Gamma$-IPMs is similar.
\begin{align*}
    D_f^\Gamma(S^\Sigma[Q]\| S^\Sigma [P]) & = D_f^{\Gamma_\Sigma^\text{inv}}(S^\Sigma[Q]\|S^\Sigma[P])\\
    & = \sup_{\gamma\in \Gamma_\Sigma^{\text{inv}}, \nu\in\R}\left\{ E_{S^\Sigma[Q]}[\gamma-\nu] - E_{S^\Sigma[P]}[f^*(\gamma-\nu)]\right\}\\
    & = \sup_{\gamma\in \Gamma_\Sigma^{\text{inv}}, \nu\in\R}\left\{ E_{Q}[\gamma-\nu] - E_{P}[f^*(\gamma-\nu)]\right\}\\
    & = D_f^{\Gamma_\Sigma^\text{inv}}(Q\|P)\,,
\end{align*}
where the first equality is due to Theorem~\ref{thm:invariant_discriminator1}, and the third equality holds as $\gamma-\nu$ and $f^*(\gamma-\nu)$ are both $\Sigma$-invariant when $\gamma\in\Gamma_\Sigma^{\text{inv}}$.
\end{proof}

Theorem~\ref{thm:Gamma_invariant} has the following implications: {  If one uses a $\Sigma$-invariant GAN (i.e., invariant discriminators and equivariant generators) to learn a non-invariant data source $Q$ then one will in fact learn the symmetrized version $S^\Sigma[Q]$.} On the other hand,  if the data source $Q$ is $\Sigma$-invariant (i.e., $S^\Sigma[Q]= Q$, cf.~Lemma~\ref{lemma:symmetr_ops}) but the GAN generated distribution $P_g$ is not then discriminators from  $\Gamma_\Sigma^\text{inv}$ alone can not differentiate $Q$ and $P_g$, i.e., $D^{\Gamma_\Sigma^\text{inv}}(Q\|P_g)=0$, as long as $Q = S^\Sigma[P_g]$. This suggests that $P_g$ can easily suffer from ``mode collapse", as it only needs to equal $Q$ after $\Sigma$-symmetrization; we refer readers to Figure~\ref{fig:toy_200_alpha_2_2d} (2nd and 4th rows) for a visual illustration, where a unimodal $P_g$ can be erroneously selected as the ``best" fitting model, even though its $\Sigma$-symmetrization $S^\Sigma[P_g]$ should be the ``correct" one.

To prevent this from happening, one needs to ensure the generator produces a $\Sigma$-invariant distribution $P_g$; this is guaranteed by the following Theorem.
\begin{theorem}
\label{thm:generator}
If $P_Z\in \mathcal{P}(Z)$ is $\Sigma$-invariant and $g:Z\to X$ is $\Sigma$-equivariant then 
the push-forward measure $P_g\coloneqq P_Z\circ g^{-1}$ is $\Sigma$-invariant, i.e., $P_g\in \mathcal{P}_\Sigma(X)$.
\end{theorem}
\begin{proof}
The proof is based on the equivalence of the following commutative diagrams:
\begin{equation}
\begin{tikzcd}
    Z\arrow[d, "T_\sigma^Z"']\arrow[r, "g"] &  X\arrow[d, "T_\sigma^X"']\\
    Z\arrow[r, "g", swap] & X
\end{tikzcd} \iff 
\begin{tikzcd}
    \mathcal{P}(Z)\arrow[d, "\circ (T_\sigma^Z)^{-1}"']\arrow[r, "\circ g^{-1}"] &  \mathcal{P}(X)\arrow[d, "\circ (T_\sigma^X)^{-1}"']\\
    \mathcal{P}(Z)\arrow[r, "\circ g^{-1}", swap] & \mathcal{P}(X)
\end{tikzcd}
\end{equation}
More specifically,
\begin{align*}
    &P_g\circ (T^X_\sigma)^{-1}=P_Z\circ g^{-1}\circ (T^X_\sigma)^{-1} =P_Z\circ  (T^X_\sigma\circ g)^{-1}\\
    =& P_Z\circ  ( g\circ T^Z_\sigma)^{-1} =P_Z\circ   (T^Z_\sigma)^{-1} \circ g^{-1}=P_Z\circ g^{-1}\\
     =&P_g\,,\notag
\end{align*}
where the third and fifth equalities are due to the equivariance and invariance, respectively, of $g$ and $P_Z$.
\end{proof}
We note that equivariant flow-based methods have also been proposed based on a similar strategy to Theorem~\ref{thm:generator}. We refer readers to Section \ref{sec:related} for  a discussion of related works.

\begin{figure}[h]
    \centering
    \includegraphics[width=\columnwidth]{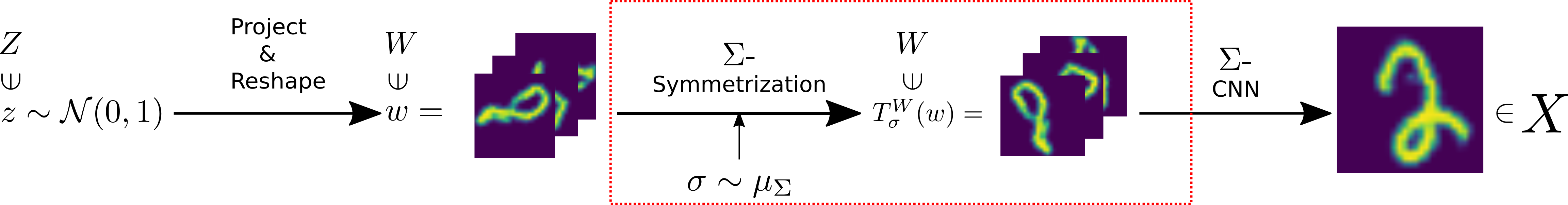}
    \caption{The $\Sigma$-symmetrization layer (enclosed in the red rectangle), which is missing in \cite{EquivariantGAN}, ensures generator equivariance, which is critical in preventing GAN ``mode collapse" [cf.~Remark~\ref{rmk:equivariant_generator_2}].}
    \label{fig:symmetrization}
\end{figure}
\begin{remark}
\label{rmk:equivariant_generator_1}
Suppose $g = \gamma_2\circ \gamma_1$ is a composition of two maps, $\gamma_1:Z\to W$ and $\gamma_2:W\to X$. Even if $\gamma_1$ is not $\Sigma$-equivariant (in fact, $Z$ does not even need to be equipped with a $\Sigma$-action $T_\sigma^Z$), as long as $P_{\gamma_1}\in\mathcal{P}(W)$ is $\Sigma$-invariant and $\gamma_2$ is $\Sigma$-equivariant, the push-forward measure $P_g\in \mathcal{P}(X)$ is still $\Sigma$-invariant.
\end{remark}
{  
To construct the $\Sigma$-invariant noise source  required in Theorem \ref{thm:generator} (or Remark \ref{rmk:equivariant_generator_1}) one can begin with an arbitrary noise source and use a {\bf $\Sigma$-symmetrization layer}, as described by the following theorem.
\begin{theorem}\label{thm:inv_noise}
Let $W\sim \mu_\Sigma$ and  $N$ be a $Z$-valued random variable (i.e., an arbitrary noise source). If $N$ and $W$ are independent then the distribution of $T^Z(W,N)$ is $\Sigma$-invariant.
\end{theorem}
\begin{proof}
Let $P_Z$ denote the distribution of $N$. Independence of $W$ and $N$ implies $(W,N)\sim \mu_\Sigma\times P_Z$. Therefore $T^Z(W,N)\sim (\mu_\Sigma\times P_Z)\circ (T^Z)^{-1}\coloneqq P^\Sigma_Z$. We need to show that $P^\Sigma_Z$ is $\Sigma$-invariant: For $\sigma\in\Sigma$ we can compute
\begin{align}
  P^\Sigma_Z\circ (T^Z_\sigma)^{-1}=&(\mu_\Sigma\times P_Z)\circ (T^Z)^{-1}\circ (T^Z_\sigma)^{-1}\\
  =&  (\mu_\Sigma\times P_Z)\circ (T^Z_\sigma\circ T^Z)^{-1}\notag\\
  =&(\mu_\Sigma\times P_Z)\circ (  T^Z\circ (T_\sigma^\Sigma\times id))^{-1}\notag\\
  =&(\mu_\Sigma\times P_Z)\circ (T_\sigma^\Sigma\times id)^{-1}\circ ( T^Z)^{-1}\,,\notag
\end{align}
where $T^\Sigma$ is the left-multiplication action of $\Sigma$ on itself.  Invariance of $\mu_\Sigma$ implies
\begin{align}
    (\mu_\Sigma\times P_Z)\circ (T_\sigma^\Sigma\times id)^{-1}=&(\mu_\Sigma\circ (T_\sigma^\Sigma)^{-1})\times P_Z    =\mu_\Sigma\times P_Z\,.
\end{align}
Therefore
\begin{align}
  &P^\Sigma_Z\circ T_\sigma^{-1}=(\mu_\Sigma\times P_Z)\circ (  T^Z)^{-1}=P^\Sigma_Z\,.
  \end{align}
  This proves $P^\Sigma_Z$ is $\Sigma$-invariant as claimed.
\end{proof}
}

\begin{remark}
\label{rmk:equivariant_generator_2}
\citet{EquivariantGAN} also proposed to use G-CNNs to generate images with $C_4/D_4$-invariant distributions. However, the first step in their model, i.e., the ``Project \& Reshape" step [cf.~Figure~\ref{fig:symmetrization}], uses a fully-connected layer which destroys the group symmetry in the noise source, leading to non-invariant final distribution $P_g$ even if the subsequent layers are all $\Sigma$-equivariant. This easily leads to ``mode collapse" [cf.~Theorem~\ref{thm:Gamma_invariant}], which we will empirically demonstrate in \cref{sec:experiments}; see, e.g., Figure~\ref{fig:toy_200_alpha_2_2d} (4th row). An easy remedy for this is to add a $\Sigma$-symmetrization layer: let $w$ be the output of ``Project \& Reshape"; the  $\Sigma$-symmetrization layer draws a random $\sigma\sim\mu_{\Sigma}$ and transforms $w$ into $T_\sigma^W(w)$, producing a $\Sigma$-invariant distribution on the layer output {  (see Theorem \ref{thm:inv_noise})}. The final distribution $P_g$ is thus $\Sigma$-invariant if subsequent layers are all $\Sigma$-equivariant by Remark~\ref{rmk:equivariant_generator_1}. See Figure~\ref{fig:symmetrization} for a visual illustration.
\end{remark}

\subsection{Conditions Ensuring $S_\Sigma[\Gamma]\subset\Gamma$}\label{sec:SGamma_Gamma}
In this section we provide conditions under which the test function space $\Gamma$ is closed under symmetrization, that being a key assumption in our main results in Section \ref{sec:theory}. First we show that $S_\Sigma[\Gamma]\subset\Gamma$ when $\Gamma$ is   the unit ball  in an appropriate RKHS. 
\begin{lemma}\label{lemma:RKHS_SH}
Let $V\subset \mathcal{M}_b(X)$ be a separable RKHS with  reproducing-kernel $k:X\times X\to\mathbb{R}$. Let $\Gamma=\{\gamma\in V:\|\gamma\|_V\leq 1\}$ be the unit ball in $V$.  Suppose we have a measurable group action $T:\Sigma\times X\to X$ and $k$ is $\Sigma$-invariant under this action (i.e., $k(T_\sigma(x),T_\sigma(y))=k(x,y)$ for all $\sigma\in \Sigma$, $x,y\in X$). Then $S_\Sigma[\Gamma]\subset\Gamma$.
\end{lemma}
\begin{remark}
The proof will  use many standard properties of a RKHS. In particular, recall that the assumption  $X\subset \mathcal{M}_b(X)$ implies $k$ is bounded  and jointly measurable. See  Chapter 4 in \cite{steinwart2008support} for this and further background. See \cite{JMLR:v12:sriperumbudur11a} and references therein for more discussion of characteristic kernels as well as the related topic of universal kernels.
\end{remark}
\begin{proof}
The $\Sigma$-invariance of $k$ implies
\begin{align}\label{eq:k_h}
    k(T_\sigma(x),y)=k(T_\sigma(x),T_\sigma( T_{\sigma^{-1}}(y)))=k(x,T_{\sigma^{-1}}(y))
\end{align}
and
\begin{align}\label{eq:inner_product_inv}
    &\langle k(\cdot,T_\sigma(x)),k(\cdot, T_\sigma(y))\rangle_V=k(T_\sigma(x),T_\sigma(y))    =k(x,y)    =\langle k(\cdot,x),k(\cdot,y)\rangle_V
\end{align}
for all $\sigma\in \Sigma$ and $x,y\in X$.   Next we will show that the map $U_\sigma:\gamma\mapsto \gamma\circ T_\sigma$ is an isometry on $V$  for all $\sigma\in \Sigma$, $\gamma\in V$:  It is clearly a linear map. To show its range is contained in $V$, first recall that the span of $\{k(\cdot,x)\}_{x\in X}$ is dense in $V$. Therefore, given $\gamma\in V$ there is a sequence $\gamma_n\to \gamma$ having the form
\begin{align*}
    \gamma_n= \sum_{i=1}^{N_n} a_{n,i}k(\cdot,x_{n,i})
\end{align*}
for some $a_{n,i}\in\mathbb{R}$, $x_{n,i}\in X$. Equation \eqref{eq:k_h} implies
\begin{align*}\label{eq:g_n_Lh}
    \gamma_n\circ T_\sigma= \sum_{i=1}^{N_n} a_{n,i}k(T_\sigma(\cdot),x_{n,i})=\sum_{i=1}^{N_n} a_{n,i}k(\cdot,T_{\sigma^{-1}}(x_{n,i}))\,.
\end{align*}
Combining \req{eq:g_n_Lh} with \req{eq:inner_product_inv} we can conclude that $\|\gamma_n\circ T_\sigma\|_V=\|\gamma_n\|_V$ and $\|\gamma_n\circ T_\sigma-\gamma_m\circ T_\sigma\|_V=\|\gamma_n-\gamma_m\|_V$.  $\gamma_n$ converges in $V$, hence is Cauchy, therefore $\gamma_n\circ T_\sigma$ is Cauchy as well. We have assumed $V$ is complete, therefore $\gamma_n\circ T_\sigma\to \tilde \gamma$ for some $\tilde \gamma\in V$.   $V$ is a RKHS, hence the evaluation maps are continuous and we find $\tilde \gamma(x)=\lim_n \gamma_n(T_\sigma(x))=\gamma(T_\sigma(x))$ for all $x$.  Therefore $\gamma\circ T_\sigma=\tilde \gamma\in V$ and
\begin{align*}
\|\gamma\circ T_\sigma\|_V =\lim_n\|\gamma_n\circ T_\sigma\|_V=\lim_n\|\gamma_n\|_V=\|\gamma\|_V\,.
\end{align*}
This proves $U_\sigma$ is an isometry on $V$. 

Now fix $\gamma\in\Gamma$. We will show that the map $\sigma\to U_\sigma[\gamma]$ is Bochner integrable (see, e.g., Appendix E in \citet{cohn2013measure}): It clearly has has separable range since $V$ was assumed to be separable. By the same reasoning as above, given $\tilde \gamma\in V$ we   have a sequence $\tilde \gamma_n\to \tilde \gamma$ where
\begin{align*}
    \tilde \gamma_n= \sum_{i=1}^{N_n} a_{n,i}k(\cdot,x_{n,i})\,.
\end{align*}
Hence
\begin{align*}
    \langle \tilde \gamma, U_\sigma[\gamma]\rangle_V=&\lim_n \sum_{i=1}^{N_n} a_{n,i}\langle k(\cdot,x_{n,i}), U_\sigma[\gamma]\rangle_V=\lim_n\sum_{i=1}^{N_n} a_{n,i}, U_\sigma[\gamma]( x_{n,i})\\
     =&\lim_n\sum_{i=1}^{N_n} a_{n,i}, \gamma(T_\sigma(x_{n,i}))\notag\,,
\end{align*}
which is now clearly measurable in $\sigma$ due to the measurability of the action.  Therefore $\sigma\mapsto U_\sigma[\gamma]$ is strongly measurable.  $\|U_\sigma[\gamma]\|_V= \|\gamma\|_V\leq 1$, therefore the Bochner integral  $\int U_\sigma[\gamma]\mu_\Sigma(d\sigma)$ exists in $V$ and satisfies
\begin{align*}
   \| \int U_\sigma[\gamma]\mu_\Sigma(d\sigma)\|_V\leq  \int \|U_\sigma[\gamma]\|_V\mu_\Sigma(d\sigma)\leq 1\,.
\end{align*}
 This proves $\int U_\sigma[\gamma]\mu_\Sigma(d\sigma)\in \Gamma$.  Finally, $V$ is a RKHS and so the evaluation maps are in $V^*$. Therefore evaluation commutes with the Bochner integral and we find
\begin{align*}
    &(\int U_\sigma[\gamma]\mu_\Sigma(d\sigma))(x)=\int U_\sigma[\gamma](x)\mu_\Sigma(d\sigma)    = \int \gamma(T_\sigma(x))\mu_\Sigma(d\sigma)=S_\Sigma[\gamma](x)\,.
\end{align*}
Hence we can conclude $S_\Sigma[\gamma]\in\Gamma$ for all $\gamma\in\Gamma$ as claimed.
\end{proof}

The next result provides a general framework for proving $S_\Sigma[\Gamma]\subset\Gamma$.
\begin{lemma}\label{lemma:closed_S_Sigma}
Let $V\subset \mathcal{M}_b(X)^n$, $n\in\mathbb{Z}^+$, be a subspace equipped with the $M(X)$-topology (see Definition \ref{def:weak_topology}) and $\Gamma\subset V$. If  $\Gamma$ is convex and closed, the group action $T:\Sigma\times X\to X$ is measurable,  $S_\Sigma[V]\subset V$, and $\Gamma$ is closed under $\Sigma$ (i.e., $\gamma\circ T_\sigma\in\Gamma$ for all $\gamma\in\Gamma$, $\sigma\in\Sigma$) then $S_\Sigma[\Gamma]\subset\Gamma$. 
\end{lemma}
\begin{proof}
Suppose we have $\gamma\in\Gamma$ with $S_\Sigma[\gamma]\not\in\Gamma$. As noted in Definition \ref{def:weak_topology}, $V$ is a locally convex topological vector space with $V^*=\{\tau_\nu:\nu\in M(X)^n\}$, $\tau_\nu(\gamma)\coloneqq \sum_{i=1}^n\int \gamma^id\nu_i$. The separating hyperplane theorem (see Theorem 3.4(b) in \citet{rudin2006functional}) applied to $A=\{S_\Sigma[\gamma]\}$ and $B=\Gamma$ therefore implies the existence of $\nu\in M(X)^n$ such that
\begin{align}
\tau_\nu(\tilde\gamma)>\tau_\nu( S_\Sigma[\gamma])
\end{align}
for all $\tilde\gamma\in\Gamma$.  We have assumed $\Gamma$ is closed under $\Sigma$ and so we can let $\tilde\gamma=\gamma\circ T_{\sigma}$  to get
\begin{align}
\sum_{i=1}^n \int\gamma^i\circ T_{\sigma}d\nu_i-\sum_{i=1}^n\int S_\Sigma[\gamma^i]d\nu_i>0
\end{align}
for all   $\sigma\in\Sigma$. Integrating with respect to $\mu_\Sigma(d\sigma)$ and using Fubini's theorem to change the order of integration we obtain a contradiction.  Therefore $S_\Sigma[\gamma]\in\Gamma$ as claimed.
\end{proof}
We end this section with several examples of function spaces, $V$, that are useful in conjunction with Lemma \ref{lemma:closed_S_Sigma}:
\begin{enumerate}
    \item $V=\mathcal{M}_b(X)^n$, $n\in\mathbb{Z}^+$, in which case $S_\Sigma[V]\subset V$ follows from measurability of the action.
    \item $X$ is a metric space, the action $T:\Sigma\times X\to X$ is continuous, and $V=C_b(X)^n$, $n\in\mathbb{Z}^+$. In this case, $S_\Sigma[V]\subset V$ follows from the dominated convergence theorem.
    \item $X$ is a metric space, the action $T:\Sigma\times X\to X$ is continuous, $T_\sigma$ is $1$-Lipschitz for all $\sigma\in\Sigma$, and $V=\Lip_b^1(X)^n$, $n\in\mathbb{Z}^+$. In this case, $S_\Sigma[V]\subset V$ follows  from the following calculation:
    \begin{align*}
        |S_\Sigma[\gamma](x)-S_\Sigma[\gamma](y)|\leq&\int |\gamma(T_\sigma(x))-\gamma(T_\sigma(y))|\mu_\Sigma(d\sigma)
        \leq \int d(T_\sigma(x),T_\sigma(y))\mu_\Sigma(d\sigma)\\
        \leq &\int d(x,y)\mu_\Sigma(d\sigma)=d(x,y)
    \end{align*}
    for all $\gamma\in\Lip_b^1(X)$.
\end{enumerate}

\section{Primal Formulation of $\Sigma$-invariant $(f,\Gamma)$-Divergences}\label{app:more_f_gamma}
In this section  we will derive a primal formulation of the $(f,\Gamma)$-divergence between $\Sigma$-invariant distributions when $\Gamma$ consists of $\Sigma$-invariant discriminators; this will take the form of an infimal convolution formula. Our result is reminiscent of the results in \citet{Cuturi:stoch_optim_OT:2016} for Sinkhorn divergences. Under appropriate assumptions, we will show that the primal optimization problem has a unique solution and  will prove  the divergence property for $(f,\Gamma)$-divergences on $\mathcal{P}_\Sigma(X)$.

In this section we will assume that $X$ is a complete separable metric space (with metric $d$).  Our analysis will require the following notion of a determining set of functions.
\begin{definition}
 Given $\mathcal{Q}\subset\mathcal{P}(X)$, a subset $\Psi\subset \mathcal{M}_b(X)$ will be called {\bf $\mathcal{Q}$-determining} if for all $Q,P\in\mathcal{Q}$, $E_Q[ \psi]=E_P[\psi]$ for all $\psi\in \Psi$ implies $Q=P$. 
\end{definition}

We will also  need   $f$ and $\Gamma$ to satisfy one of the following admissibility criteria, as introduced in \cite{Birrell:f-Gamma}.
\begin{definition}\label{def:admissible}
For  $a,b$ with $-\infty\leq a<1<b\leq\infty$  we define   $\mathcal{F}_1(a,b)$ to be the set of convex functions $f:(a,b)\to\mathbb{R}$ with $f(1)=0$. For $f\in\mathcal{F}_1(a,b)$, if $b$ is finite we extend the definition of $f$ by $f(b)\coloneqq\lim_{x\nearrow b}f(x)$. Similarly, if $a$ is finite we define  $f(a)\coloneqq\lim_{x\searrow a}f(x)$ (convexity implies  these limits exist in $(-\infty,\infty]$). Finally, extend $f$ to $x\not\in [a,b]$ by $f(x)=\infty$.  The resulting function $f:\mathbb{R}\to(-\infty,\infty]$ is convex and  LSC. 

We will call  $f\in\mathcal{F}_1(a,b)$ {\bf admissible} if  
$\{f^*<\infty\}=\mathbb{R}$ and $\lim_{y\to-\infty}f^*(y)<\infty$ (note that this limit always exists by convexity).  If $f$ is also strictly convex at $1$ then we will call $f$  {\bf strictly admissible}. We will call $\Gamma\subset C_b(X)$ {\bf admissible} if $0\in\Gamma$, $\Gamma$ is convex, and $\Gamma$ is closed  in the  $M(X)$-topology on $C_b(X)$ (see Definition \ref{def:weak_topology}).    $\Gamma$ will be called {\bf strictly admissible} if it also satisfies the following property: There exists a $\mathcal{P}(X)$-determining set $\Psi\subset C_b(X)$ such that for all $\psi\in\Psi$ there exists $c\in\mathbb{R}$, $\epsilon>0$ such that $c\pm\epsilon \psi\in\Gamma$. Finally, an admissible $\Gamma\subset C_{b,\Sigma}^{\text{inv}}(X)$ (the set of $\Sigma$-invariant bounded continuous functions) will be called {\bf $\Sigma$--strictly admissible} if there exists a $\mathcal{P}_\Sigma(X)$-determining set $\Psi\subset C_b(X)$ such that for all $\psi\in\Psi$ there exists $c\in\mathbb{R}$, $\epsilon>0$ such that $c\pm\epsilon \psi\in\Gamma$. 
\end{definition}

One way to construct a $\Sigma$-strictly admissible set is to start with an appropriate strictly admissible set and then restrict  to the subset of $\Sigma$-invariant functions; see Appendix \ref{app:admissibility} for a proof.
\begin{lemma}\label{lemma:Gamma_H_admissible}
Let $\Gamma\subset C_b(X)$.
\begin{enumerate}
    \item If $\Gamma$ is admissible then $\Gamma_\Sigma^{\text{inv}}$ is admissible.
    \item If  $\Gamma$ is strictly admissible and $S_\Sigma[\Gamma]\subset\Gamma$ then $\Gamma_\Sigma^{\text{inv}}$ is $\Sigma$-strictly admissible.
\end{enumerate}
\end{lemma}

Below are several useful examples of strictly admissible $\Gamma$ that satisfy $S_\Sigma[\Gamma]\subset\Gamma$.
\begin{enumerate}
    \item $\Gamma\coloneqq C_b(X)$, if the action is continuous in $x$, i.e., if $T_\sigma:X\to X$ is continuous for all $\sigma\in \Sigma$.
    \item $\Gamma\coloneqq\{g\in C_b(X):|g|\leq C\}$ for any $C>0$ and assuming the action is continuous in $x$,
    \item $\Gamma\coloneqq\Lip_b^L(X)$ for any $L>0$ and assuming the action is $1$-Lipschitz, i.e., $d(T_\sigma( x),T_\sigma(y))\leq d(x,y)$ for all $\sigma\in\Sigma$, $x,y\in X$.
    \item $\Gamma\coloneqq\{g\in\Lip_b^L(X):|g|\leq C\}$ for any $C,L>0$ and assuming the action is $1$-Lipschitz.
    \item The unit ball in an appropriate RKHS $V$, $\Gamma\coloneqq\{g\in V:\|g\|_V\leq 1\}$, assuming   the kernel   is $\Sigma$-invariant; see Lemma  \ref{lemma:RKHS_admissibility}  for details.
\end{enumerate}

The following result extends the infimal convolution formula and divergence properties from \cite{Birrell:f-Gamma} to the case where the models and test-function space are $\Sigma$-invariant. 
\begin{theorem}
 Suppose   $f$ and $\Gamma$ are admissible and $\Gamma\subset C_{b,\Sigma}^{\text{inv}}(X)$. For $Q,P\in\mathcal{P}_\Sigma(X)$ we have the following properties:
\begin{enumerate}
\item Infimal Convolution Formula on $\mathcal{P}_\Sigma(X)$: 
\begin{align}\label{eq:inf_conv}
D_f^{\Gamma}(Q\|P)=\inf_{\eta\in \mathcal{P}_\Sigma(X)}\{D_f(\eta\|P)+W^{\Gamma}(Q,\eta)\}\,.
\end{align}
In particular, $D_f^\Gamma(Q\|P)\leq \min\{D_f(Q\|P), W^\Gamma(Q,P)\}$.
\item Existence of an Optimizer: If $D_f^{\Gamma}(Q\|P)<\infty$ then there exists $\eta_*\in\mathcal{P}_\Sigma(X)$ such that
\begin{align}\label{eq:inf_conv_existence}
D_f^{\Gamma}(Q\|P)=D_f(\eta_*\|P)+W^{\Gamma}(Q,\eta_*)\,.
\end{align}
If $f$ is strictly convex then there is a unique such $\eta_*$.
\item $\mathcal{P}_\Sigma(X)$-Divergence Property for $W^\Gamma$: $W^{\Gamma}(Q,P)\geq 0$ and $W^{\Gamma}(Q,P)=0$ if $Q=P$.  If $\Gamma$ is $\Sigma$-strictly admissible then $W^{\Gamma}(Q,P)=0$  implies $Q=P$.
\item $\mathcal{P}_\Sigma(X)$-Divergence Property for $D_f^\Gamma$: $D_f^\Gamma(Q\|P)\geq 0$ and $D_f^\Gamma(Q\|P)=0$ if $Q=P$. If $f$ is strictly admissible  and $\Gamma$ is $\Sigma$-strictly admissible then $D_f^{\Gamma}(Q\|P)= 0$ implies  $Q=P$.
\end{enumerate}
\end{theorem}
\begin{proof}
\begin{enumerate}
    \item Part 1 of Theorem 2.15 from \cite{Birrell:f-Gamma} implies an infimal convolution formula on $\mathcal{P}(X)$, hence
\begin{align}
D_f^\Gamma(Q\|P)=\inf_{\eta\in \mathcal{P}(X)}\{D_f(\eta\|P)+W^\Gamma(Q,\eta)\}\leq \inf_{\eta\in \mathcal{P}_\Sigma(X)}\{D_f(\eta\|P)+W^\Gamma(Q,\eta)\}\,.
\end{align}
To prove the reverse inequality, we use the bound $D_f\geq D_f^{S_\Sigma[\mathcal{M}_b(X)]}$, the equality $S_\Sigma[\Gamma]=\Gamma$, and then Theorem \ref{thm:Gamma_invariant} to compute
\begin{align}
D_f^\Gamma(Q\|P)\geq&\inf_{\eta\in \mathcal{P}(X)}\{D_f^{S_\Sigma[\mathcal{M}_b(X)]}(\eta\|P)+W^{S_\Sigma[\Gamma]}(Q,\eta)\}\\
=&\inf_{\eta\in \mathcal{P}(X)}\{D_f(S^\Sigma[\eta]\|P)+W^\Gamma(Q,S^\Sigma[\eta])\}\notag\\
= & \inf_{\eta\in\mathcal{P}_\Sigma(X)}\{D_f(\eta\|P)+W^\Gamma(Q,\eta)\}\,.\notag
\end{align}
This proves the infimal convolution formula on $\mathcal{P}_\Sigma(X)$.
\item Now suppose $D_f^\Gamma(Q\|P)<\infty$.  Part 2 of Theorem 2.15 from \cite{Birrell:f-Gamma} implies there exists $\eta_*\in \mathcal{P}(X)$ such that
\begin{align}
    D_f^\Gamma(Q\|P)=D_f(\eta_*\|P)+W^\Gamma(Q,\eta_*)\,.
\end{align}
We need to show that $\eta_*$ can be taken to be $\Sigma$-invariant.  To do this, first use the infimal convolution formula to bound
\begin{align}
  D_f^\Gamma(Q\|P)\leq  D_f(S^\Sigma[\eta_*]\|P)+W^\Gamma(Q,S^\Sigma[\eta_*])\,.
\end{align}

The $\Sigma$-invariance of $Q$ and $P$ together with Theorem \ref{thm:Gamma_invariant} imply
\begin{align}
W^\Gamma(Q,S^\Sigma[\eta_*])=W^\Gamma(Q,\eta_*)\,.
\end{align}
and 
\begin{align}
  D_f(S^\Sigma[\eta_*]\|P)=  D_f^{\mathcal{M}_{b,\Sigma}^{\text{inv}}(X)}(\eta_*\|P)\leq D_f(\eta_*\|P)\,.
\end{align}
Therefore
\begin{align}
  D_f^\Gamma(Q\|P)\leq  D_f(S^\Sigma[\eta_*]\|P)+W^\Gamma(Q,S^\Sigma[\eta_*])\leq  D_f(\eta_*\|P)+W^\Gamma(Q,\eta_*)=D_f^\Gamma(Q\|P)\,.
\end{align}
Hence 
\begin{align}
    D_f^\Gamma(Q\|P)=D_f(S^\Sigma[\eta_*]\|P)+W^\Gamma(Q,S^\Sigma[\eta_*])
\end{align}
with $S^\Sigma[\eta_*]\in\mathcal{P}_\Sigma(X)$ as claimed.

If $f$ is strictly convex then uniqueness is a corollary of the corresponding uniqueness result from Part 2 of   Theorem 2.15 in \cite{Birrell:f-Gamma}.
\item Admissibility of $\Gamma$ implies $0\in\Gamma$, hence $W^\Gamma(Q\|P)\geq E_Q[0]-E_P[0]= 0$.  If $Q=P$ then the definition clearly implies $W^\Gamma(Q,P)=0$. If $\Gamma$ is $\Sigma$-strictly admissible and $W^{\Gamma}(Q,P)=0$ then $0\geq E_Q[g]-E_P[g]$ for all $g\in\Gamma$.  Letting $g=c\pm\epsilon\psi$ as in the definition of $\Sigma$-strict admissiblity we see that $0\geq \pm ( E_Q[\psi]- E_P[\psi])$.  Hence $E_Q[\psi]=E_P[\psi]$ for all $\psi\in\Psi$. $\Psi$ is  a $\mathcal{P}_\Sigma(X)$-determining set and $Q,P\in\mathcal{P}_\Sigma(X)$, hence we can conclude that $Q=P$.
\item We know that $D_f\geq 0$ and $W^\Gamma\geq 0$, therefore the infimal convolution formula implies $D_f^\Gamma\geq 0$. If $Q=P$   we can bound
\begin{align}
    0\leq D_f^\Gamma(Q\|P)\leq D_f(Q\|P)=0\,,
\end{align}
hence $D_f^\Gamma(Q\|P)=0$.  Finally, suppose $f$ is strictly admissible, $\Gamma$ is $\Sigma$-strictly admissible, and $D_f^\Gamma(Q\|P)=0$.  Then Part 2 of this theorem implies
\begin{align}
    0=D_f^\Gamma(Q\|P)=D_f(\eta_*\|P)+W^\Gamma(Q,\eta_*)
\end{align}
for some $\eta_*\in\mathcal{P}_\Sigma(X)$.
Both terms are non-negative, hence \begin{align}
    D_f(\eta_*\|P)=W^\Gamma(Q,\eta_*)=0\,.
\end{align} 
The $\mathcal{P}_\Sigma(X)$-divergence property for $W^\Gamma$ then implies $Q=\eta_*$.  $f$ being strictly admissible implies that $D_f$ has the divergence property, hence $\eta_*=P$.  Therefore $Q=P$ as claimed.
\end{enumerate}
\end{proof}

\section{Experiments}\label{sec:experiments}
We now present experiments on both synthetic and real-world data sets with embedded group symmetry to empirically verify our theory for structure-preserving GANs from Section \ref{sec:theory}.

{  \subsection{Algorithmic Feasibility}

Theorems \ref{thm:invariant_discriminator1} and \ref{thm:generator}  imply  that one can build invariant GANs by using $\Sigma$-
invariant discriminators, $\Sigma$-equivariant generators, and a $\Sigma$-invariant noise source. Equivariant networks for arbitrary group symmetry (and gauge invariance) have been studied in recent works such as \cite{pmlr-v48-cohenc16}. Invariant noise sources can be constructed as shown in  Theorem \ref{thm:inv_noise}. We note that the symmetrization operators $S^\Sigma$, $S_\Sigma$ are only
used in the proofs of theoretical properties of the proposed
GANs and are not needed in practical implementations. The necessary invariance/equivariance is built into
the discriminator/generator via the structure of the layers; see Appendix \ref{app:architecture}.
}

\subsection{Data sets and common experimental setups}

%The experiments are conducted on the following data sets:

\textit{Toy example.} Following \cite{Birrell:f-Gamma}, this synthetic data source is a mixture of four 2D t-distributions with $0.5$ degrees of freedom, embedded in a plane in $\R^{12}$. The four centers of the t-distributions are located (in the supporting plane) at coordinates $(\pm 10, \pm 10)$, exhibiting $C_4$-symmetry [cf.~Figure~\ref{fig:toy_200_alpha_2_2d}].

\textit{RotMNIST} is built by randomly rotating the original 10-class $28\times 28$ MNIST digits \cite{lecun1998gradient}, resulting in an $SO(2)$-invariant distribution. We use different portions of the 60,000 training images for experiments in \cref{sec:rotmnist}.

\textit{ANHIR} consists of pathology slides stained with 5 distinct dyes for the study of cellular compositions \cite{borovec2020anhir}. Following \cite{EquivariantGAN}, we extract from the original images 28,407 foreground patches of size $64\times 64$. The staining dye is used as the class label for conditioned image synthesis. As the images have no preferred orientation/reflection, the distribution is $O(2)$-invariant.

\textit{LYSTO} contains 20,000 patches extracted from whole-slide images of breast, colon and prostate cancer stained with immunohistochemical markers \cite{francesco_ciompi_2019_3513571}. The images are classified into 3 categories based on the organ source, and we downsize the images to $64\times 64$. Similar to ANHIR, this data set is also $O(2)$-invariant.

\textbf{Common experimental setups.} To verify our theory in \cref{sec:theory}, and to quantify and disentangle the contributions of the structure-preserving discriminator (\texttt{D}) and generator (\texttt{G}) (Theorem~\ref{thm:invariant_discriminator1} and Theorem~\ref{thm:Gamma_invariant}), we  replace the baseline \texttt{G}  and/or  \texttt{D} by their group-equivariant/invariant counterparts, \texttt{Eqv G} and \texttt{Inv D}, while adjusting the number of filters according to the group size to ensure a similar number of trainable parameters. We also consider the incomplete attempt by  \citet{EquivariantGAN} at building equivariant generators (\texttt{(I)Eqv G}), wherein the first fully-connected layer destroys the symmetry in the noise source, resulting in non-equivariant \texttt{G} even if subsequent layers are all equivariant [cf.~Remark~\ref{rmk:equivariant_generator_2}].
We use the Fr\'echet Inception Distance (FID) \cite{heusel2017gans} to evaluate the quality and diversity of the GAN generated samples after embedding them in the feature space of a pre-trained Inception-v3 network \cite{szegedy2016rethinking}. Due to the simplicity of RotMNIST, we replace the inception-featurization by the encoding feature space of an autoencoder trained on the \textit{rotated} digits. We note that, compared to classifiers, autoencoders are guaranteed to produce \textit{different} features for rotated versions of the same digit; they are thus more suitable to measure sample diversity in rotation.

\begin{figure}[h]
    \centering
    \begin{subfigure}[b]{.64\columnwidth}
        \includegraphics[width=\textwidth]{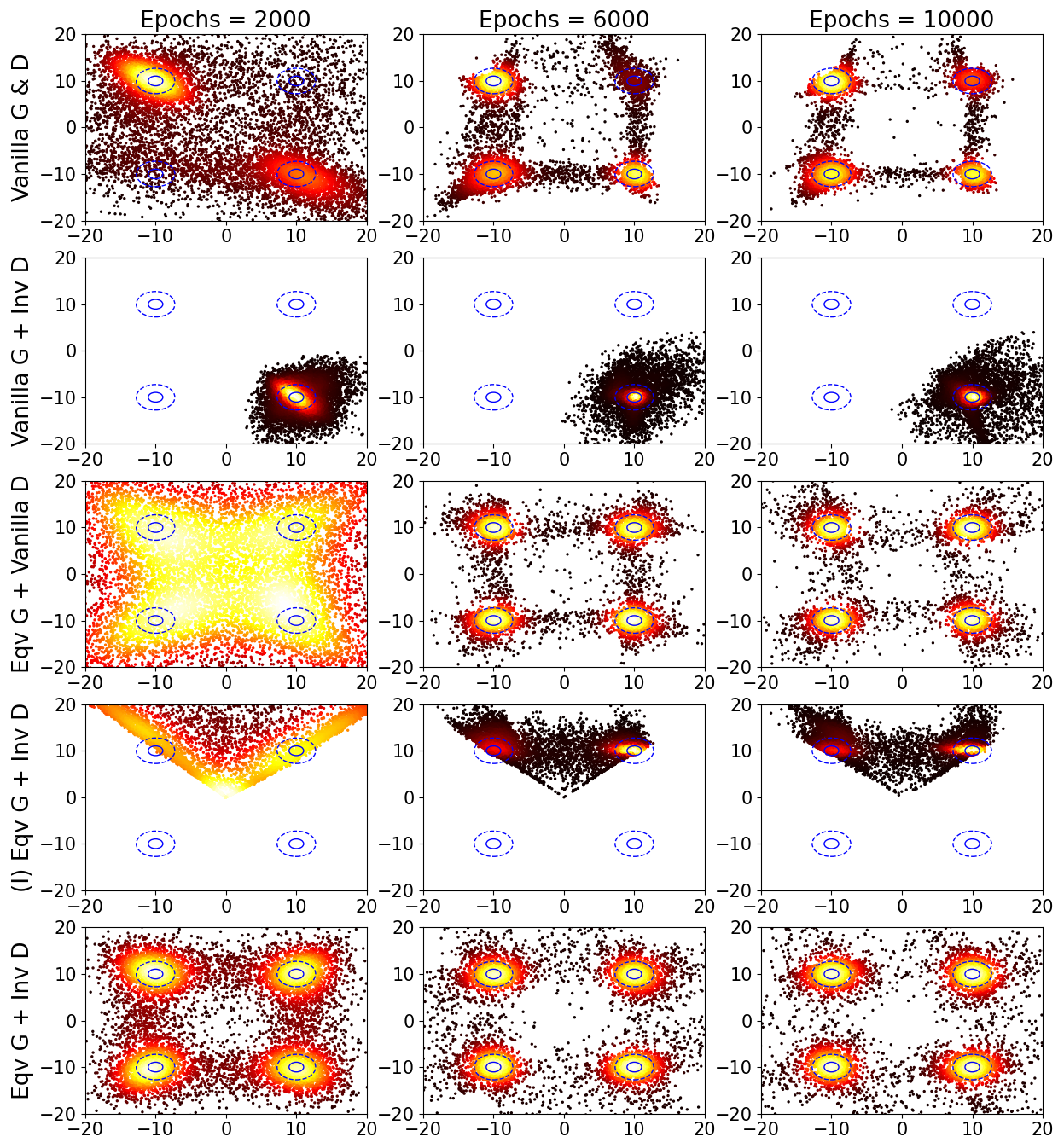}
        \caption{2D projection of the generated samples.}
        \label{fig:toy_200_alpha_2_2d}
    \end{subfigure}
    \begin{subfigure}[b]{.35\columnwidth}
        \centering
        \includegraphics[width=\textwidth]{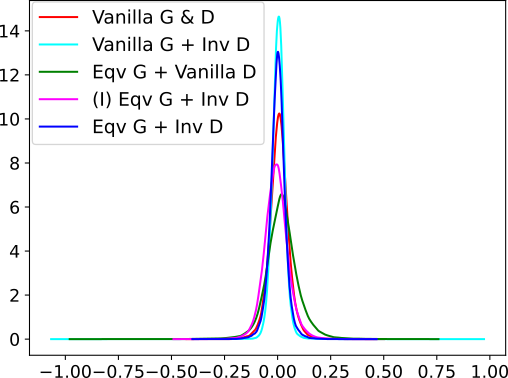}
        \caption{$D_2^L$-GANs.}
        \label{fig:alpha_2_complement}
        \vspace{5em}
        \includegraphics[width=\textwidth]{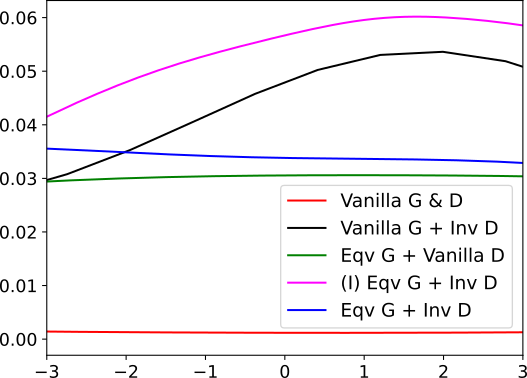}
        \caption{WGANs}
        \label{fig:wgan_complement}
    \end{subfigure}
    \caption{{  This figure illustrates how our method can simultaneously
handle heavy tails and low-dimensional support.} Panel (a): 2D projection of the $D_2^L$-GAN generated samples onto the support plane of the source $Q$ [cf. \cref{sec:toy}]. Each column shows the result after a given number of training epochs. The rows correspond to different settings for the generators (\texttt{G}) and discriminators (\texttt{D}); in particular, the 2nd and 4th rows use invariant \texttt{D} accompanied by, respectively, a baseline \texttt{G}  and an incorrectly constructed equivariant \texttt{G}, leading to mode collapse [cf.~Theorem~\ref{thm:Gamma_invariant}]. The blue ovals mark the 25\% and 50\% probability regions of the data source $Q$, while the heat-map shows the generator samples.
    Panel (b) and (c): Generator distribution, projected onto components orthogonal to the support plane of $Q$. Values concentrated around zero indicate convergence to the sub-manifold. Models are trained on \textbf{200} training points.}
    \label{fig:toy_200_alpha_2}
\end{figure}

\subsection{Toy Example}
\label{sec:toy}

We test the performance of different GANs (and their equivariant versions) based on 3 types of divergences, namely the Wasserstein-GAN (WGAN) based on the $\Gamma$-IPM \req{eq:general_divergence_IPM}, the $D_{f_\alpha}$-GAN based on the classifical $f$-divergence \req{eq:f_divergence} and \eqref{eq:KL_alpha}, and the $D_\alpha^L$-GAN based on the $(f,\Gamma)$-divergence \req{eq:lip_alpha_divergence}, in learning the $C_4$-invariant mixture $Q$. We use fully-connected networks with 3 hidden layers for the baseline \texttt{G} and \texttt{D} (\texttt{Vanilla G\&D}). The generator pushes forward a 10D Gaussian noise source, which is itself $C_4$-invariant after prescribing a proper group action, e.g., $\pi/2$-rotations in the first two dimensions. Equivariant \texttt{G} (\texttt{Eqv G}) and invariant \texttt{D} (\texttt{Inv D}) are built by replacing fully-connected layers with $C_4$-convolutional layers based on Theorem~\ref{thm:generator} due to the $C_4$-invariance of the noise source. We also mimic the incomplete attempt by  \citet{EquivariantGAN} in building equivariant generators (\texttt{(I)Eqv G}) by leaving the first fully-connected layer unchanged and replacing only the subsequent layers by $C_4$-convolutions.

Figure~\ref{fig:toy_200_alpha_2_2d} displays the 2D projection of the generated samples learned by the $D_{\alpha=2}^L$-GAN (and its equivariant versions) on 200 training samples. It is clear that the baseline model without structural prior (\texttt{Vanilla G\&D}) has difficulty in learning $Q$ in such small data regime. Using an $\texttt{Inv D}$ alone without an \texttt{Eqv G} (\texttt{Vanilla G} + \texttt{Inv D}) or with an incorrectly imposed \texttt{Eqv G} (\texttt{(I)Eqv G} + \texttt{Inv D}) leads easily to  ``mode collapse", validating Theorem~\ref{thm:Gamma_invariant}. On the other hand, $D_{\alpha}^L$-GAN with an \texttt{Eqv G} (even without an \texttt{Inv D}) is able to learn all 4 modes of $Q$. We omit the results of (equivariant) $D_{f_{\alpha}}$-GANs and WGANs from Figure~\ref{fig:toy_200_alpha_2_2d}, as both fail to learn the data source $Q$; this is unsurprising due to the lack of absolute continuity between $Q$ and $P_g$ (the former is supported on a plane, while the latter is the entire 12D space) and the fact that $Q$ is heavy-tailed (as the mean does not exist.) This demonstrates the importance of our framework's broad applicability to a variety of variational divergences, as an improper choice of the divergence---even with structural prior---can fail to learn the source distribution.

Figure~\ref{fig:toy_200_alpha_2}~(b) and (c) show the generated distribution  projected onto components orthogonal to the support plane of $Q$. Values concentrated around zero indicate successful learning of the low-dimensional source distribution, i.e., generating high-fidelity samples. Figure~\ref{fig:alpha_2_complement} indicates that an \texttt{Inv D} in the $D_{\alpha}^L$-GAN helps produce a distribution with sharper support, whereas \texttt{Eqv G} alone without \texttt{Inv D} tends to generate relatively low-quality samples away from the supporting plane. In contrast, Figure~\ref{fig:wgan_complement} indicates that WGAN (even with symmetry prior) fails to learn the support plane due to $Q$ being heavy-tailed. Results with different numbers of training samples and $\alpha$'s are shown in \cref{app:additional_results}, and the conclusions are similar.

\subsection{RotMNIST}
\label{sec:rotmnist}
We adopt a similar setup to \citet{EquivariantGAN}. Specifically, in the baseline \texttt{G}, a fully-connected layer first projects and reshapes the concatenated Gaussian noise and class embedding into a 2D feature map (see Figure~\ref{fig:symmetrization}); spectrally-normalized convolutions \cite{miyato2018spectral},  interspersed with pointwise-nonlinearities, class-conditional batch-normalizations, and upsamplings, are subsequently used to increase the spatial dimension. We note again that replacing 2D convolutions with $C_n$-convolutions does not simply lead to \texttt{Eqv G}, as the distribution after the ``project and reshape" layer is no longer $C_n$-invariant. This can be fixed by adding a $C_n$-symmetrization layer after the first linear embedding; see Remark~\ref{rmk:equivariant_generator_2}. We consider GANs with the relative average loss (RA-GANs) \cite{jolicoeur-martineau2018} in addition to the $D_\alpha^L$-GANs for this experiment. All configurations are trained with a batch size of 64 for 20,000 generator iterations. Implementation details are available in \cref{app:implementation_details}.

  \begin{table}[H]
    \caption{The median of the FIDs (lower is better), calculated every 1,000 generator update for 20,000 iterations, averaged over three independent trials. The number of the training samples used for experiments varies from 1\% (600) to 10\% (6,000) of the RotMNIST training set. See Appendix \ref{app:additional_results} for further results. } \label{tab:fid_rotmnist}  
  \centering
  \begin{tabular}{m{1em}ccccccc}
    \toprule
    Loss & Architecture &   1\%& 5\%& 10\% & 50\% & 100\% \\
    \midrule
    \rotatebox{90}{RA-GAN }& \makecell{\texttt{CNN G\&D} \\
    \texttt{Eqv G} + \texttt{CNN D}, $\Sigma = C_4$\\    
    \texttt{CNN G} + \texttt{Inv D}, $\Sigma = C_4$\\
    \texttt{(I)Eqv G} + \texttt{Inv D}, $\Sigma = C_4$\\
    \texttt{Eqv G} + \texttt{Inv D}, $\Sigma = C_4$ \\
    \texttt{Eqv G} + \texttt{Inv D}, $\Sigma = C_8$}&
\makecell{295\\389\\223\\173\\\textbf{98}\\123}&\makecell{357\\333\\181\\141\\78\\\textbf{52}}&
    \makecell{348\\355\\188\\132\\89\\\textbf{51}}&
    \makecell{403\\380\\177\\135\\84\\\textbf{52}}&
    \makecell{392\\393\\176\\130\\82\\\textbf{57}}
    \\
    \midrule
    \rotatebox{90}{ $D_{\alpha=2}^\Gamma$-GAN }& \makecell{\texttt{CNN G\&D} \\
    \texttt{Eqv G} + \texttt{CNN D}, $\Sigma = C_4$\\    
    \texttt{CNN G} + \texttt{Inv D}, $\Sigma = C_4$\\
    \texttt{(I)Eqv G} + \texttt{Inv D}, $\Sigma = C_4$\\
    \texttt{Eqv G} + \texttt{Inv D}, $\Sigma = C_4$ \\
    \texttt{Eqv G} + \texttt{Inv D}, $\Sigma = C_8$}&\makecell{280\\253\\330\\273\\149\\\textbf{122}}&\makecell{261\\271\\208\\147\\99\\\textbf{55}}&\makecell{283\\251\\192\\133\\88\\\textbf{57}}&\makecell{297\\274\\183\\124\\80\\\textbf{53}}&\makecell{293\\275\\173\\126\\81\\\textbf{51}}\\
    \bottomrule
  \end{tabular}
\end{table}

Table~\ref{tab:fid_rotmnist}  shows the median of the FIDs, calculated every 1,000 generator update, averaged over three independent trials.  It is clear that our proposed models (\texttt{Eqv G} + \texttt{Inv D}) consistently achieve significantly improved results compared to the baseline \texttt{CNN G\&D} and the prior approach (\texttt{(I)Eqv G} + \texttt{Inv D}); the out-performance is even more pronounced when increasing the group size from $\Sigma=C_4$ to $C_8$.  { We note that, similar to RotMNIST, one can also use a custom autoencoder featurization for FID evaluation, and the superiority of our model (\texttt{Eqv G} + \texttt{Inv D}) is even more prominent under such metric: for instance, on ANHIR, the median FIDs calculated through autoencoder featurization  of the
three comparing models are, respectively, 1221 (\texttt{CNN G\&D}), 936 ((\texttt{(I)Eqv G} + \texttt{Inv D})), and 329 (\texttt{Eqv G} + \texttt{Inv D})}. See Figure~\ref{fig:rotmnist_digits_2}  for randomly generated samples by RA-GANs trained with 1\% training data. More results are available in \cref{app:additional_results}.

\begin{figure}[h]
    \centering
    \begin{subfigure}[b]{.3\columnwidth}
    \includegraphics[width=1\columnwidth]{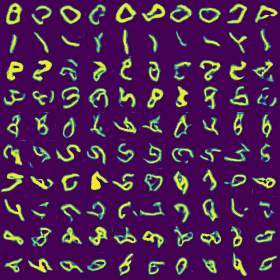}  
    \caption{\texttt{CNN G\&D}}
        
    \end{subfigure}
    ~
    \begin{subfigure}[b]{.3\columnwidth}
    \includegraphics[width=1\columnwidth]{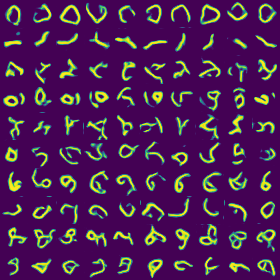}  
    \caption{\texttt{Eqv G} + \texttt{CNN D}, $\Sigma=C_4$}
        
    \end{subfigure}    
    ~
    \begin{subfigure}[b]{.3\columnwidth}
    \includegraphics[width=1\columnwidth]{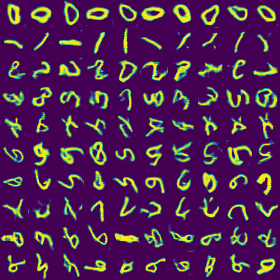}  
    \caption{\texttt{CNN G} + \texttt{Inv D}, $\Sigma=C_4$}
        
    \end{subfigure}   
    ~
    \begin{subfigure}[b]{.3\columnwidth}
    \includegraphics[width=1\columnwidth]{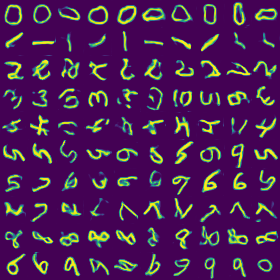}  
    \caption{\texttt{(I)Eqv G} + \texttt{Inv D}, $\Sigma=C_4$}
        
    \end{subfigure}    
    ~
    \begin{subfigure}[b]{.3\columnwidth}
    \includegraphics[width=1\columnwidth]{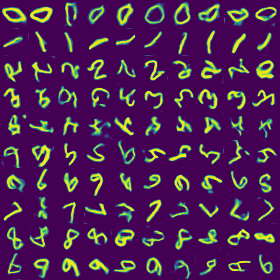}  
    \caption{\texttt{Eqv G} + \texttt{Inv D}, $\Sigma=C_4$}
        
    \end{subfigure}    
    ~
    \begin{subfigure}[b]{.3\columnwidth}
    \includegraphics[width=1\columnwidth]{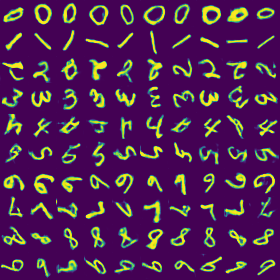}  
    \caption{\texttt{Eqv G} + \texttt{Inv D}, $\Sigma=C_8$}
        
    \end{subfigure}        
    \caption{Randomly generated digits by the RA-GANs trained on RotMNIST after 20K generator iterations with 1\% (600) training data. More images are available in \cref{app:additional_results}.}
\label{fig:rotmnist_digits_2}
\end{figure}

\subsection{ANHIR and LYSTO}
\label{sec:medical_data}

Compared to RotMNIST, ResNet and its $D_4$-equivariant counterpart are used instead of CNNs for \texttt{G}  and \texttt{D}. All models are trained for 40,000 generator iterations with a batch size of 32. Implementation details are available in \cref{app:implementation_details}.

Table~\ref{tab:fid_medical} displays the minimum and median of the FIDs, calculated every 2,000 generator update, averaged over three independent trials. The plus sign ``+" after the data set, e.g., ANHIR+, denotes the presence of data augmentation (random $90^\circ$ rotations and reflection) during training. It is clear that augmentation usually (but not always) has a positive effect on the results evaluated by the FID; however, our proposed model even without data augmentation still consistently and significantly outperforms the baseline model (\texttt{CNN G\&D}) and the prior approach (\texttt{(I)Eqv G} + \texttt{Inv D}) \cite{EquivariantGAN} with augmentation. Figure~\ref{fig:anhir_images_large_main} presents a random collection of real and generated LYSTO images, visually verifying the improved sample fidelity of our model over the baseline. More results are available in \cref{app:additional_results}.

\begin{figure}[h]
    \centering
    \includegraphics[width=.9\textwidth]{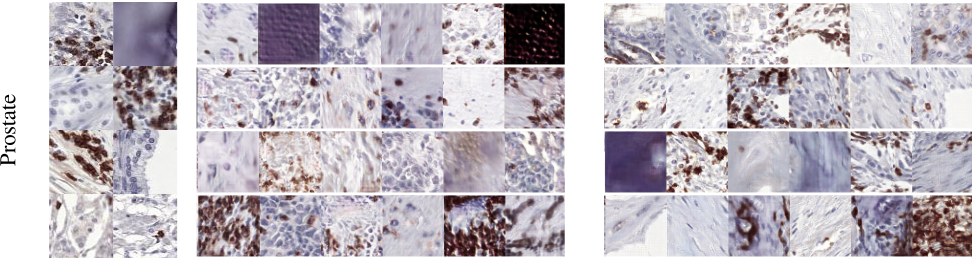}
    \caption{Real and GAN generated LYSTO images of breast, colon, and prostate cancer. Left panel: real images. Middle and right panels: randomly selected $D_2^L$-GANs' generated samples after 40,000 generator iterations. Middle panel: \texttt{CNN G\&D}. Right panel: \texttt{Eqv G} + \texttt{Inv D}. More images are available in \cref{app:additional_results}.}
    \label{fig:anhir_images_large_main}
\end{figure}

\begin{table}[h]
  \caption{The (min, median) of the FIDs over the course of training, averaged over three independent trials on the medical images, where the plus sign ``+" after the data set, e.g., ANHIR+, denotes the presence of data augmentation during training.} \label{tab:fid_medical}  
  \centering
  \begin{tabular}{m{1em}ccc}
    \toprule
    Loss & Architecture   & ANHIR & ANHIR+\\
    \midrule
    $D_{2}^L$& \makecell{\texttt{CNN G\&D}\\
    \texttt{(I)Eqv G} + \texttt{Inv D} \\    
    \texttt{Eqv G} + \texttt{Inv D}}
    &\makecell{(313, 485)\\
    (120, 176) \\    
    \textbf{(97, 157)}}
    &\makecell{(347, 539)\\
    (119, 177) \\    
    \textbf{(90, 128)}}
    \\
    \bottomrule
    \toprule
    Loss & Architecture   & LYSTO & LYSTO+\\
    \midrule
    $D_{2}^L$& \makecell{\texttt{CNN G\&D}\\
    \texttt{(I)Eqv G} + \texttt{Inv D} \\    
    \texttt{Eqv G} + \texttt{Inv D}}
    & \makecell{(289, 410)\\
    (253, 343) \\    
    \textbf{(205, 259)}}
    & \makecell{(265, 376)\\
    (244, 329)\\    
    \textbf{(192, 259)}}
    \\
    \bottomrule    
  \end{tabular}
\end{table}

\subsection{Discussion of empirical findings}
Consistently across all experiments, our proposed structure-preserving GAN outperforms prior approaches in generating high-fidelity and diverse samples by a significant margin, in some cases almost an order of magnitude measured in FID. The results also show that, compared to data-augmentation (a common strategy for learning from limited data), building theoretically-guided structural probabilistic priors directly into the two GAN players  achieves substantially improved performance and data efficiency in adversarial learning.

\section*{Acknowledgements}
The research of J.B., M.K. and L.R.-B. was partially supported by  the Air Force Office of Scientific Research (AFOSR) under the grant FA9550-21-1-0354.
The research of M. K. and L.R.-B. was partially supported by the National Science Foundation (NSF) under the grants DMS-2008970 and TRIPODS CISE-1934846. The research of W.Z. was  partially supported by NSF under DMS-2052525 and DMS-2140982. We thank Neel Dey for sharing the pre-processed ANHIR data set. This work was performed in part using high performance computing equipment obtained under a grant from the Collaborative R\&D Fund managed by the Massachusetts Technology Collaborative.

%{ [A (very) trimmed version of conclusions but only on Sec 5 findings---sort of a PR short paragraph, useful to the reviewers as a summary of the most *striking* aspects of our  empirical findings]}

%\section{Conclusions}
%We propose the structure-preserving GANs, a general data-efficient framework of learning distributions with embedded structures, by developing new variational representations for divergences between structured probability measures. Our theory shows how to reduce the discriminator space according to the underlying structure to achieve efficient distribution learning; it also explains why one needs to \textit{correctly} build structured generators when reducing the discriminator space, as its flawed design may easily lead to a ``mode collapse" of the trained model. When contextualized in learning group-invariant distributions, our proposed model, which is rigorously guided by the theory, significantly outperforms prior approaches consistently across a variety of data sets, producing samples with both improved fidelity and diversity. In future work, we will enrich our proposed framework and study its application for other structure preservation cases, such as coarse-graining for molecular dynamics and cosmology.

%%%%%%%%%%%%%%%%
\bibliography{Structure-preserving_GANs.bbl}
%\bibliography{refs}
%\bibliographystyle{icml2022}

%%%%%%%%%%%%%%%%%%%%%%%%%%%%%%%%%%%%%%%%%%%%%%%%%%%%%%%%%%%%%%%%%%%%%%%%%%%%%%%
%%%%%%%%%%%%%%%%%%%%%%%%%%%%%%%%%%%%%%%%%%%%%%%%%%%%%%%%%%%%%%%%%%%%%%%%%%%%%%%
% APPENDIX
%%%%%%%%%%%%%%%%%%%%%%%%%%%%%%%%%%%%%%%%%%%%%%%%%%%%%%%%%%%%%%%%%%%%%%%%%%%%%%%
%%%%%%%%%%%%%%%%%%%%%%%%%%%%%%%%%%%%%%%%%%%%%%%%%%%%%%%%%%%%%%%%%%%%%%%%%%%%%%%
\newpage
\appendix
\onecolumn

\appendix

\section{More details on variational representations of divergences and probability metrics}
\label{app:variational_divergence}

We provide, in this appendix, more details on variational representations of the divergences and probability metrics discussed in \cref{sec:variational_divergence_background}. Recall the notation introduced in the main paper: let $(X,\mathcal{M})$ be a measurable space, $\mathcal{M}(X)$ be the space of measurable functions on $X$, and $\mathcal{M}_b(X)$  be the subspace of bounded measurable functions. We denote $\mathcal{P}(X)$ as the set of probability measures on $X$.
Given an objective functional $H:\mathcal{M}^n(X) \times \mathcal{P}(X)\times\mathcal{P}(X)\to [-\infty,\infty]$ and a test function space $\Gamma\subset \mathcal{M}(X)^n, n\in\mathbb{Z}^+$, we define 
\begin{align}\label{eq:general_divergence_IPM_app}
    D_H^\Gamma(Q\|P)=\sup_{\gamma\in\Gamma}H(\gamma;Q,P)\,.
\end{align}
$D_H^\Gamma$ is called a \textit{divergence} if $D_H^\Gamma\geq 0$ and $D_H^\Gamma(Q\|P)=0$ if and only if $Q=P$, hence providing a notion of  ``distance" between probability measures. $D_H^\Gamma$ is further called a \textit{probability metric} if it satisfies the triangle inequality (i.e., $D_H^\Gamma(Q\|P)\leq D_H^\Gamma(Q\|\nu)+D_H^\Gamma(\nu\|P)$ for all $Q,P,\nu\in\mathcal{P}(X)$) and is symmetric (i.e., $D_H^{\Gamma}(Q\|P) = D_H^\Gamma(P\|Q)$ for all $P, Q\in \mathcal{P}(X)$). It is well known that formula \eqref{eq:general_divergence_IPM_app} includes, through suitable choices of objective functional $H(\gamma;Q,P)$ and function space $\Gamma$, many divergences and probability metrics. Below we further elaborate on the examples discussed in \cref{sec:variational_divergence_background}.

%%%%%%%%%%%%%%%%%%%

% and let $\Gamma$ be a collection of functions on $X$. 
\textbf{(a) $f$-divergences.} Let $f:[0, \infty) \to \mathbb{R}$ be convex and lower semi-continuous (LSC), with  $f(1)=0$ and $f$ strictly convex at  $x=1$. The $f$-divergence between $Q$ and $P$ can be defined based on two equivalent variational representations \cite{Birrell:f-Gamma}, namely
\begin{align}
D_f(Q\|P)=&\sup_{\gamma\in \mathcal{M}_b(X)}\{ E_Q[\gamma]-E_P[f^*(\gamma)]\}\label{eq:f-div:variational:LT_app}\\
=&\sup_{\gamma\in \mathcal{M}_b(X)}\{ E_Q[\gamma]-\Lambda_f^P[\gamma]\}\, ,\label{eq:f-div:variational:DV_app}
\end{align}
% where the second equality follows from \eqref{eq:Lambda_f_def} and \eqref{eq:Df_var_formula} due to the invariance of $\mathcal{M}_b(\Omega)$ under the shift map $ g\mapsto g-\nu$ for $\nu\in\mathbb{R}$
where $f^*$ in the first representation \eqref{eq:f-div:variational:LT_app} denotes the Legendre transform (LT) of $f$,
\begin{align}\label{eq:LT:def_app}
    f^*(y)=\sup_{x\in\mathbb{R}}\{yx-f(x)\}, \quad \forall y\in \R,
\end{align}
and $\Lambda_f^P[\gamma]$ in the second representation \eqref{eq:f-div:variational:DV_app} is defined as
\begin{align}\label{eq:Lambda_f_def_app}
    \Lambda_f^P[\gamma]\coloneqq\inf_{\nu\in\mathbb{R}}\{\nu+E_P[f^*(\gamma-\nu)]\}\,,\,\,\,\,\,\gamma\in\mathcal{M}_b(\Omega)\,.
\end{align}

The two variational representations \req{eq:f-div:variational:LT_app} and \req{eq:f-div:variational:DV_app} share the same $\Gamma=\mathcal{M}_b(X)$, and their equivalence is due to $\mathcal{M}_b(\Omega)$ being closed under the shift map $\gamma\mapsto \gamma-\nu$ for $\nu\in\mathbb{R}$. Examples of the $f$-divergences include the Kullback-Leibler (KL) divergence \cite{kullback1951information}, the total variation distance, the $\chi^2$-divergence, the Hellinger distance, the Jensen-Shannon divergence, and the family of $\alpha$-divergences \cite{nowozin2016f}. For instance, the KL-divergence is constructed from
\begin{align}
\label{eq:KL_alpha_app}
    f_{KL} = x\log x, \quad \forall x\ge 0.
\end{align}

A key element in the second variational representation for $D_f$ [\req{eq:f-div:variational:DV_app}] is the functional $\Lambda_f^P[\gamma]$, which is a generalization of  the cumulant generating function from the KL-divergence case to the $f$-divergence case. Indeed, for the  KL-divergence  where  $f(x)=f_{KL}(x)=x\log x$, it is straightforward to show that $\Lambda_f^P$  becomes the standard cumulant generating function,
$\Lambda_{f_{KL}}^P[\gamma]=\log E_P[e^\gamma]$,
and \req{eq:f-div:variational:DV_app} becomes the Donsker-Varadhan variational formula; see Appendix C.2 in \cite{dupuis2011weak}.
The flexibility of $f$ allows one to tailor the divergence to the data source, e.g., for heavy tailed data. Moreover, the strict concavity of $f$ in $\gamma$  can result in improved statistical learning, estimation, and convergence performance. However, the variational representations \eqref{eq:f-div:variational:LT_app} and \eqref{eq:f-div:variational:DV_app} both result in $D_f(Q\|P) = \infty$ if $Q$ is not absolutely continuous with respect to $P$, limiting their efficacy in comparing distributions with low-dimensional support.

\textbf{(b) $\Gamma$-Integral Probability Metrics (IPMs).} Given $\Gamma\subset \mathcal{M}_b(X)$, the $\Gamma$-IPM between $Q$ and $P$ is defined as
\begin{align}\label{eq:IPM:def_app}
    W^\Gamma(Q,P)=\sup_{\gamma\in\Gamma}\{E_Q[\gamma]-E_P[\gamma]\}.
\end{align}
We refer to \cite{muller_1997, Gretton_review_IPM} for a complete theory and  conditions on $\Gamma$ ensuring that  $W^\Gamma(Q,P)$ is a metric. Apart from the Wasserstein metric when $\Gamma=\Lip^1(X)$ is the space of 1-Lipschitz functions, examples of IPMs also include: the total variation metric, where $\Gamma$ is the unit ball in $\mathcal{M}_b(X)$; the Dudley metric, where $\Gamma$ is the unit ball in the space of bounded and Lipschitz continuous functions; and maximum mean discrepancy (MMD), where $\Gamma$ is the unit ball in an RKHS \cite{muller_1997, Gretton_review_IPM}. With suitable choices of $\Gamma$, IPMs
are able to meaningfully compare not-absolutely continuous distributions, but they could potentially fail at comparing distributions with heavy tails \cite{Birrell:f-Gamma}.

\textbf{(c) $(f,\Gamma)$-divergences.} This class of divergences were  introduced in \cite{Birrell:f-Gamma} and they subsume both $f$-divergences and $\Gamma$-IPMs. Given a function $f$ satisfying the same condition as in the definition of the $f$-divergence and $\Gamma\subset \mathcal{M}_b(X)$, the $(f,\Gamma)$-divergence is defined as
\begin{align}
\label{eq:Df_Gamma_def1_app}
D_f^\Gamma(Q\|P)=&\sup_{\gamma\in\Gamma}\left\{E_Q[\gamma]-\Lambda_f^P[\gamma]\right\},
\end{align}
where $\Lambda_f^P[\gamma]$ is again given by \req{eq:Lambda_f_def_app}, implying that \req{eq:Df_Gamma_def1} includes as a special case the $f$-divergence \eqref{eq:f_divergence} when $\Gamma = \mathcal{M}_b(X)$ and the $\Gamma\subset\mathcal{M}_b(X)$ implies
\begin{align}\label{eq:Df_geq_Df_Gamma_app}
    D_f^\Gamma(Q\|P)\leq D_f(Q\|P)
\end{align}
for any $\Gamma\subset\mathcal{M}_b(X)$.  It is demonstrated in \cite{Birrell:f-Gamma} that  one also has
\begin{align}\label{eq:IPM_geq_Df_Gamma_app}
    D_f^\Gamma(Q\|P)\leq W^{\Gamma}(Q, P)\, .
\end{align}
Some notable examples of such $\Gamma$'s can be found in  \cite{Birrell:f-Gamma}, for instance  the 1-Lipschitz functions $\Lip^1(X)$, the RKHS unit ball, ReLU neural networks, ReLU neural networks with spectral normalizations, etc. The property \eqref{eq:IPM_geq_Df_Gamma_app} readily  implies that $(f, \Gamma)$ divergences can be defined for non-absolutely continuous probability distributions. If $X$ is further assumed to be a complete separable metric space   then, under stronger assumptions on $f$ and $\Gamma$, one has the following Infimal Convolution Formula:
\begin{align}
    D_f^\Gamma(Q\|P) = \inf_{\eta\in \mathcal{P}(X)}\left\{D_f(\eta\|P) + W^\Gamma(Q, \eta) \right\},
\end{align}
which implies, in particular, $0\le D_f^\Gamma(Q\|P)\le \min \{D_f(Q\|P) , W^\Gamma(Q, P)\}$, i.e., \req{eq:Df_geq_Df_Gamma_app} and \req{eq:IPM_geq_Df_Gamma_app}.

\textbf{(d) Sinkhorn divergences.} The Wasserstein (or ``earth-mover") metric associated with a cost function $c:X\times X\to \R^+$ has the variational representation
\begin{align}
    W_c^\Gamma(Q, P)= \inf_{\pi \in {\rm Co}(Q,P)} E_\pi[c(x,y)] = \sup_{\gamma = (\gamma_1, \gamma_2)\in \Gamma}\{E_P[\gamma_1] + E_Q[\gamma_2]\}\,,
\end{align}
where ${\rm Co}(Q,P)$ is the set  of all couplings of $P$ and $Q$ and  $\Gamma=\{\gamma=(\gamma_1, \gamma_2) \in C(X)\times C(X): \gamma_1(x) + \gamma_2(y) \le c(x,y)\, , x, y \in X\}$, with $C(X)$ being the space of continuous functions on $X$ ($C_b(X)$ will denote the subspace of bounded continuous functions).  The Sinkhorn divergence is given by
\begin{align}\label{eq:sinkhorn_def_app}
    \mathcal{SD}^\Gamma_{c,\epsilon}(Q,P)\,=&\, W^\Gamma_{c,\epsilon}(Q,P) - \frac{1}{2}W^\Gamma_{c,\epsilon}(Q,Q) -\frac{1}{2}W^\Gamma_{c,\epsilon}(P,P),
\end{align}
with $W^\Gamma_{c,\epsilon}(Q,P)$ being the entropic regularization of the Wasserstein metrics \cite{Cuturi:stoch_optim_OT:2016},
\begin{align}
W^\Gamma_{c,\epsilon}(Q,P)=& \inf_{\pi \in {\rm Co}(Q,P)} \left\{ E_\pi[c(x,y)] + \epsilon R(\pi\|P\times Q)\right\}\\
=& \sup_{\gamma=(\gamma_1, \gamma_2)\in \Gamma} \left\{ E_P[\gamma_1] + E_Q[\gamma_2] - \epsilon E_{P\times Q}\left[\exp\left(\frac{\gamma_1 \oplus\gamma_2 - c}{\epsilon}\right) \right] + \epsilon \right\} \,,
\end{align}
where now $\Gamma=C_b(X)\times C_b(X)$ and $\gamma_1\oplus\gamma_2(x,y)\coloneqq\gamma_1(x)+\gamma_2(y)$.

\section{Proofs}\label{app:proofs}
In this appendix we provide proofs of several  results that were stated without proof in the main text. 
\subsection{Proof of    Theorem \ref{thm:invariant_discriminator1} for Sinkhorn Divergences. }\label{app:Sinkhorn_proof}
\begin{theorem}\label{thm:sinkhorn_proof}
 If $Q,P\in\mathcal{P}_\Sigma(X)$, $\Gamma\subset\mathcal{M}_b(X)^2$ with $S_\Sigma[\Gamma] \subset \Gamma$ ($S_\Sigma[\gamma]^i\coloneqq S_\Sigma[\gamma^i]$), and $c(T_\sigma(x),T_\sigma(y))=c(x,y)$ for all $\sigma\in\Sigma$, $x,y\in X$ then
\begin{equation}\label{eq:main_thm:Sinkhorn_app}
     \mathcal{SD}_{c, \epsilon}^\Gamma(Q,P)= \mathcal{SD}_{c, \epsilon}^{\Gamma^{\text{inv}}_\Sigma}(Q,P)\, .
     \end{equation}
\end{theorem}
\begin{remark}
Note that the classical Sinkhorn divergence is obtained when $\Gamma=C_b(X)\times C_b(X)$ but the proof of this theorem applies to any $\Gamma\subset \mathcal{M}_b(X)^2$ with $S_\Sigma[\Gamma]\subset\Gamma$.
\end{remark}
\begin{proof}
Equation \eqref{eq:sinkhorn_def_app} implies that it suffices to show $W^\Gamma_{c,\epsilon}(Q,P)=W^{\Gamma^{\text{inv}}_\Sigma}_{c,\epsilon}(Q,P)$: From the proof of Theorem \ref{thm:invariant_discriminator1} we know that $\Gamma^{\text{inv}}_\Sigma=S_\Sigma[\Gamma]$, therefore
\begin{align*}
    &W^{\Gamma_\Sigma^{\text{inv}}}_{c,\epsilon}(Q,P)\\
    =&W^{S_\Sigma[\Gamma]}_{c,\epsilon}(Q,P)=\sup_{(\gamma_1, \gamma_2)\in \Gamma}\left\{E_P[S_\Sigma[\gamma_1]] + E_Q[S_\Sigma[\gamma_2]] - \epsilon E_{P\times Q}\left[\exp\left(\frac{S_\Sigma[\gamma_1] \oplus S_\Sigma[\gamma_2] - c}{\epsilon}\right) \right] + \epsilon\right\}\\
    =&\sup_{(\gamma_1, \gamma_2)\in \Gamma}\left\{E_{S^\Sigma[P]}[\gamma_1] + E_{S^\Sigma[Q]}[\gamma_2] - \epsilon E_{P\times Q}\left[\exp\left(\frac{\int \gamma_1(T_\sigma(x))+\gamma_2(T_\sigma(y))-c(x,y)\mu_\Sigma(d\sigma)}{\epsilon}\right) \right] + \epsilon\right\}\,.
\end{align*} 
Using Jensen's inequality followed by Fubini's theorem on the third term we obtain
\begin{align*}
    &W^{\Gamma_\Sigma^{\text{inv}}}_{c,\epsilon}(Q,P)\\
    \geq&\sup_{(\gamma_1, \gamma_2)\in \Gamma}\left\{E_{S^\Sigma[P]}[\gamma_1] + E_{S^\Sigma[Q]}[\gamma_2] - \epsilon \int E_{P\times Q}\left[ \exp\left(\frac{\gamma_1(T_\sigma(x))+\gamma_2(T_\sigma(y))-c(x,y)}{\epsilon}\right) \right]\mu_\Sigma(d\sigma) + \epsilon\right\}\,.
\end{align*} 
Finally, the $\Sigma$-invariance of $Q$, $P$, and $c$ imply $S^\Sigma[P]=P$, $S^\Sigma[Q]=Q$, and
\begin{align*}
&\int E_{P\times Q}\left[ \exp\left(\frac{\gamma_1(T_\sigma(x))+\gamma_2(T_\sigma(y))-c(x,y)}{\epsilon}\right) \right]\mu_\Sigma(d\sigma)\\
=     &\int E_{P\times Q}\left[ \exp\left(\frac{\gamma_1(T_\sigma(x))+\gamma_2(T_\sigma(y))-c(T_\sigma(x),T_\sigma(y))}{\epsilon}\right) \right]\mu_\Sigma(d\sigma)\\
    =&\int \int \int \exp\left(\frac{\gamma_1(x)+\gamma_2(y)-c(x,y)}{\epsilon}\right) Q\circ T_\sigma^{-1}(dx) P\circ T_\sigma^{-1}(dy)\mu_\Sigma(d\sigma)\\
      =& \int \int \exp\left(\frac{\gamma_1(x)+\gamma_2(y)-c(x,y)}{\epsilon}\right)Q(dx) P(dy)\,.
\end{align*} 
Therefore
\begin{align*}
    W^{\Gamma_\Sigma^{\text{inv}}}_{c,\epsilon}(Q,P)
    \geq&\sup_{(\gamma_1, \gamma_2)\in \Gamma}\left\{E_{P}[\gamma_1] + E_{Q}[\gamma_2] - \epsilon E_{P\times Q}\left[ \exp\left(\frac{\gamma_1\oplus\gamma_2-c}{\epsilon}\right) \right] +\epsilon\right\}=W_{c,\epsilon}^\Gamma(Q,P)\,.
\end{align*} 
The reverse inequality follows from $\Gamma_\Sigma^{\text{inv}}\subset\Gamma$ and so the proof is complete.
\end{proof}

\subsection{Admissibility Lemmas}\label{app:admissibility}
In this appendix we prove several lemmas regarding admissible test function spaces. First we prove the admissibility properties of $\Gamma_\Sigma^{\text{inv}}$ from Lemma \ref{lemma:Gamma_H_admissible}.
\begin{lemma}
Let $\Gamma\subset C_b(X)$.
\begin{enumerate}
    \item If $\Gamma$ is admissible then $\Gamma_\Sigma^{\text{inv}}$ is admissible.
    \item If  $\Gamma$ is strictly admissible and $S_\Sigma[\Gamma]\subset\Gamma$ then $\Gamma_\Sigma^{\text{inv}}$ is $\Sigma$-strictly admissible.
\end{enumerate}
\end{lemma}
\begin{proof}
\begin{enumerate}
    \item The zero function is $\Sigma$-invariant, hence is in $\Gamma_\Sigma^{\text{inv}}$.  If $\gamma_1,\gamma_2\in\Gamma_\Sigma^{\text{inv}}$ and $t\in[0,1]$ then convexity of $\Gamma$ implies $t\gamma_1+(1-t)\gamma_2\in\Gamma$. We have $(t\gamma_1+(1-t)\gamma_2)\circ T_\sigma=t\gamma_1\circ T_\sigma+(1-t)\gamma_2\circ T_\sigma=t\gamma_1+(1-t)\gamma_2$, hence we  conclude that  $\Gamma_\Sigma^{\text{inv}}$ is convex.  Finally, we can write
    \begin{align*}
        \Gamma_\Sigma^{\text{inv}}=&\Gamma\bigcap_{\sigma\in \Sigma,x\in X}\{\gamma\in C_b(X):\gamma(T_\sigma(x))=\gamma(x)\}\\
        =&\Gamma\bigcap_{\sigma\in \Sigma,x\in X}\{\gamma\in C_b(X):\tau_{\delta_{T_\sigma( x)}}[\gamma]=\tau_{\delta_x}[\gamma]\}\,.
    \end{align*}
We have assumed $\Gamma$ is admissible, hence it is closed. The maps $\tau_\nu$, $\nu\in M(X)$ are continuous on $C_b(X)$,  hence the sets $\{\gamma\in C_b(X):\tau_{\delta_{T_\sigma( x)}}[\gamma]=\tau_{\delta_x}[\gamma]\}$ are also closed.  Therefore $\Gamma_\Sigma^{\text{inv}}$ is closed.  This proves $\Gamma_\Sigma^{\text{inv}}$ is admissible.
\item Now suppose $\Gamma$ is strictly admissible and $S_\Sigma[\Gamma]\subset \Gamma$.  In particular, $\Gamma$ is admissible and so Part 1 implies $\Gamma_\Sigma^{\text{inv}}$ is admissible.  Let $\Psi$ be as in the definition of strict admissibility. For every $\psi\in \Psi$ there exists $c\in\mathbb{R}$, $\epsilon>0$ such that $c\pm \epsilon\psi\in\Gamma$.  Hence $c\pm \epsilon S_\Sigma[\psi]=S_\Sigma[c\pm \epsilon \psi]\in S_\Sigma[\Gamma]=\Gamma_\Sigma^{\text{inv}}$ (see the proof of Theorem \ref{thm:invariant_discriminator1}) and $S_\Sigma[\Psi]\subset C_b(X)$.  Finally, suppose $Q,P\in\mathcal{P}_\Sigma(X)$ such that $E_Q[S_\Sigma[\psi]]=E_P[S_\Sigma[\psi]]$ for all $\psi\in\Psi$. Part (b) of Lemma \ref{lemma:symmetr_ops} then implies $E_Q[\psi]=E_P[\psi]$ for all $\psi\in \Psi$. $\Psi$ is $\mathcal{P}(X)$-determining, hence $Q=P$. Therefore $S_\Sigma[\Psi]$ is a $\mathcal{P}_\Sigma(X)$-determining set and we conclude that $\Gamma_\Sigma^{\text{inv}}$ is $\Sigma$-strictly admissible.  
\end{enumerate}
\end{proof}

Next we provide assumptions under which the unit ball  in a RKHS  is   closed under $S_\Sigma$ and is (strictly) admissible.
\begin{lemma}\label{lemma:RKHS_admissibility}
Let $V\subset C_b(X)$ be a separable RKHS with  reproducing-kernel $k:X\times X\to\mathbb{R}$. Let $\Gamma=\{\gamma\in V:\|\gamma\|_V\leq 1\}$ be the unit ball in $V$.  Then:
\begin{enumerate}
    \item $\Gamma$ is admissible. 
    \item If the kernel is characteristic (i.e., the map  $P\in \mathcal{P}(X)\mapsto \int k(\cdot,x)P(dx)\in V$ is one-to-one) then $\Gamma$ is strictly admissible.
    \item If $k$ is $\Sigma$-invariant the $S_\Sigma[\Gamma]\subset \Gamma$.
\end{enumerate} 
\end{lemma}
\begin{proof}
\begin{enumerate}
\item Admissibility was shown in Lemma C.9 in \cite{Birrell:f-Gamma}.
\item Now suppose the kernel is characteristic.  Let $P,Q\in\mathcal{P}(X)$ with $\int \gamma dP=\int \gamma dQ$ for all $\gamma\in\Gamma$ (and hence for all $\gamma \in V$). Therefore
\begin{align}
0=\int \gamma dQ-\int \gamma dP=\langle \gamma,\int k(\cdot,x)Q(dx)-\int k(\cdot,x)P(dx)\rangle_V  
\end{align}
for all $\gamma\in V$.  Therefore $\int k(\cdot,x)Q(dx)=\int k(\cdot,x)P(dx)$. We have assumed the kernel is characteristic, hence we  conclude that $Q=P$. This proves  $\Gamma$ is $\mathcal{P}(X)$-determining.  We also have $-\Gamma\subset \Gamma$, hence $\Gamma$ is strictly admissible.
\item This was shown in Lemma \ref{lemma:RKHS_SH} above.
\end{enumerate}
\end{proof}

\newpage

\section{Additional Experiments}
\label{app:additional_results}

\begin{figure}[H]
  \centering
  \begin{subfigure}[b]{1\textwidth}
    \centering
    \includegraphics[width=\textwidth]{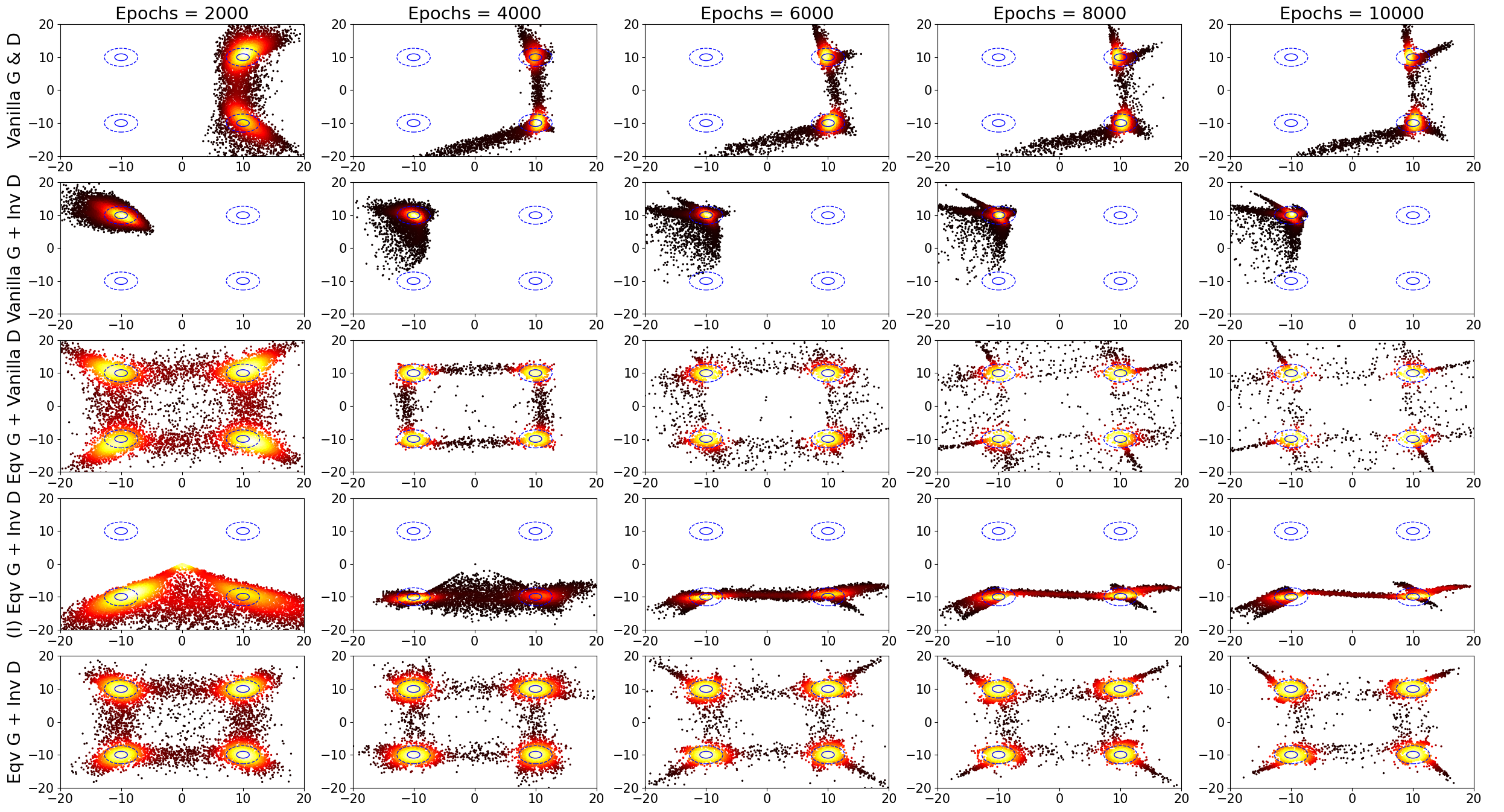}
    \caption{Models trained with 50 training samples.}
    \label{fig:toy_50_alpha_2_2d}
  \end{subfigure}
  \begin{subfigure}[b]{1\textwidth}
    \centering
    \includegraphics[width=\textwidth]{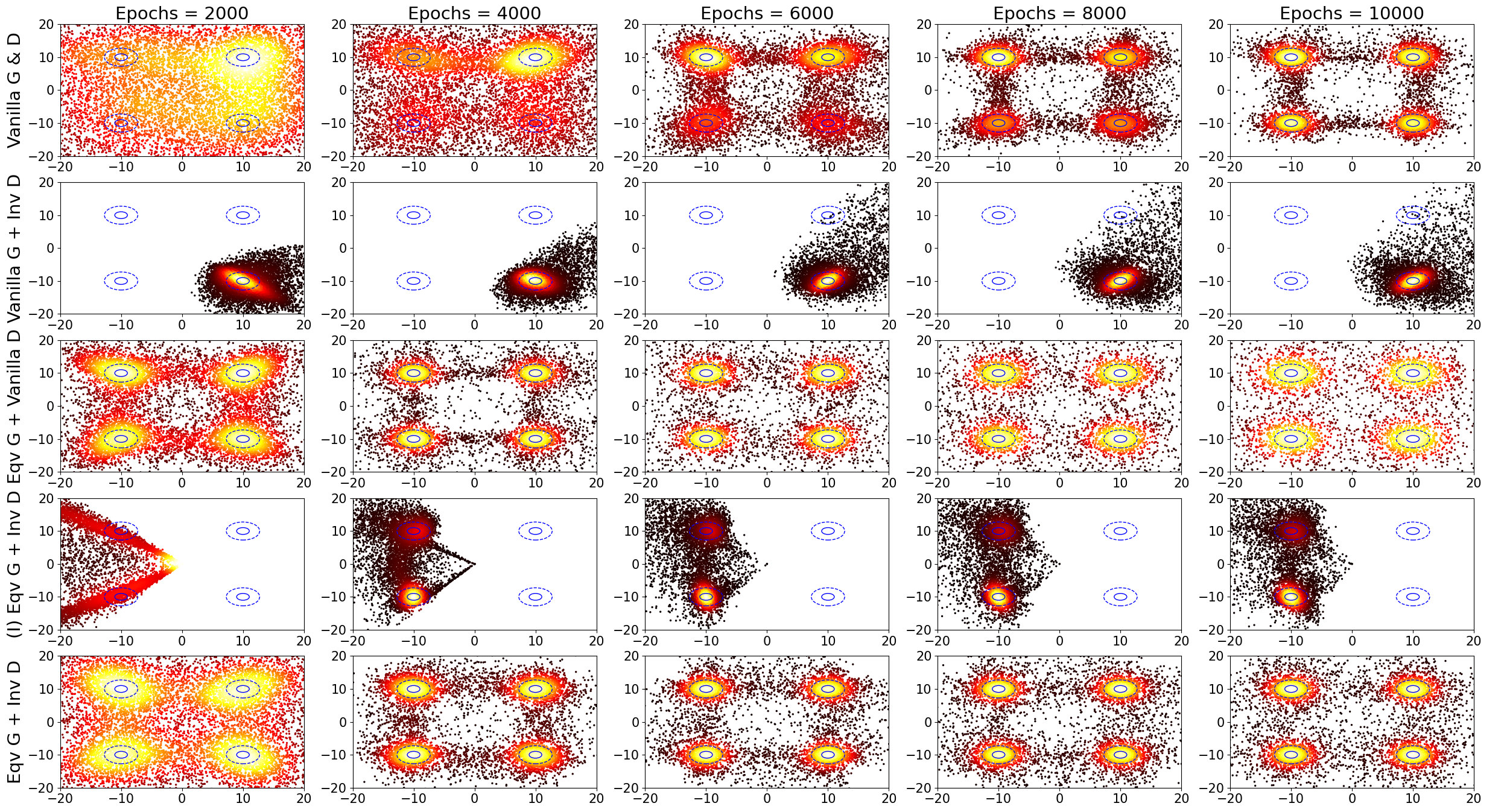}
    \caption{Models trained with 5000 training samples.}
    \label{fig:toy_5000_alpha_2_2d}
  \end{subfigure}
  \caption{2D projection of the $D_2^{L}$-GAN generated samples onto the support plane of the source distribution $Q$ [cf. \cref{sec:toy}]. Each column shows the result after a given number of training epochs. The rows correspond to different settings for the generators and discriminators. The solid and dashed blue ovals mark the 25\% and 50\% probability regions, respectively, of the data source $Q$, while the heat-map shows the generator samples. Panel (a): models are trained with \textbf{50} training samples. Panel (b): models are trained with \textbf{5000} training samples.}
  \label{fig:toy_alpha_2_2d_different_Nsample}
\end{figure}

\begin{figure}[H]
  \centering
  \begin{subfigure}[b]{1\textwidth}
    \centering
    \includegraphics[width=\textwidth]{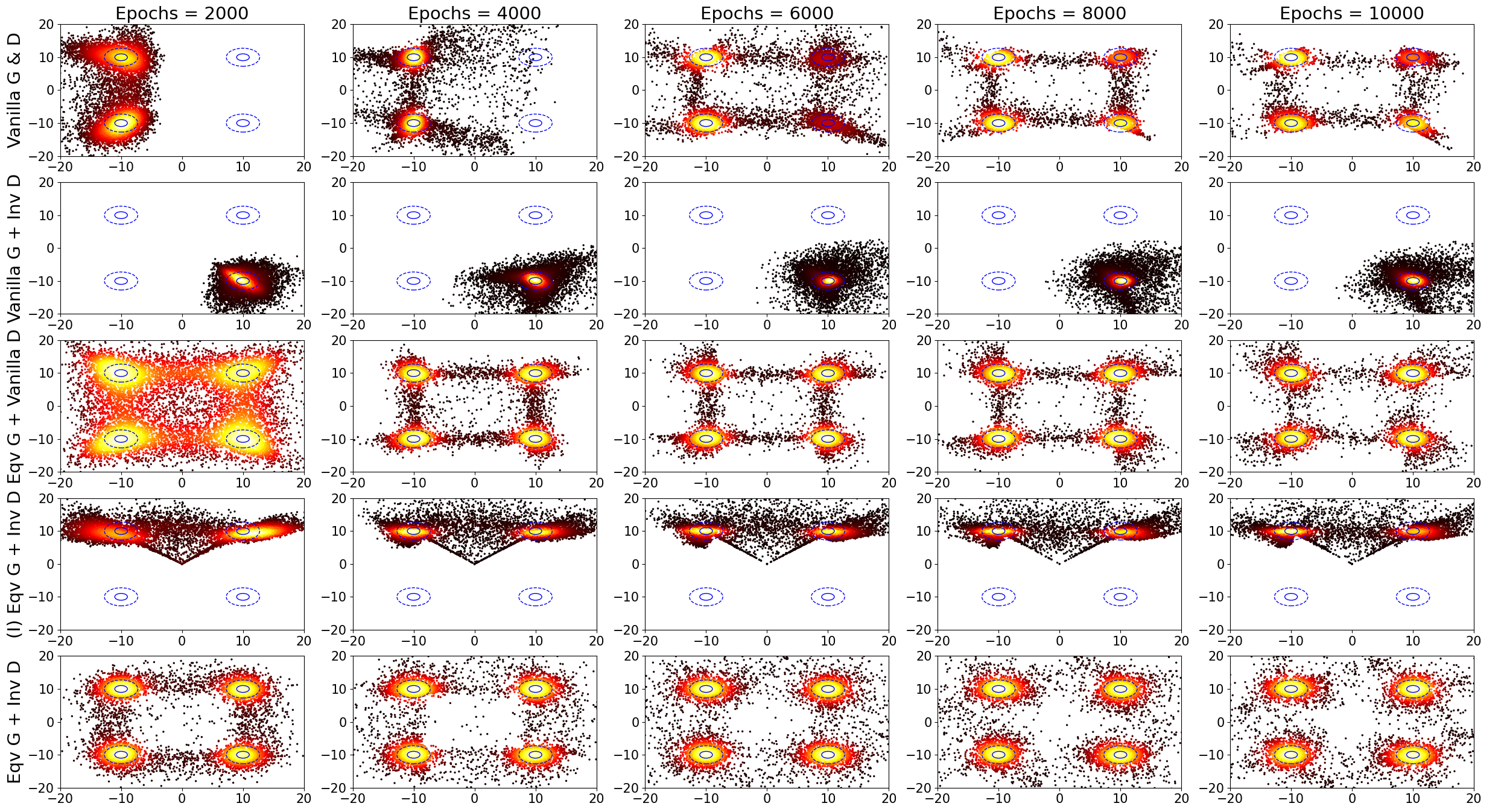}
    \caption{$D_\alpha^L$-GANs, $\alpha = 5$.}
    \label{fig:toy_200_alpha_5_2d}
  \end{subfigure}
  \begin{subfigure}[b]{1\textwidth}
    \centering
    \includegraphics[width=\textwidth]{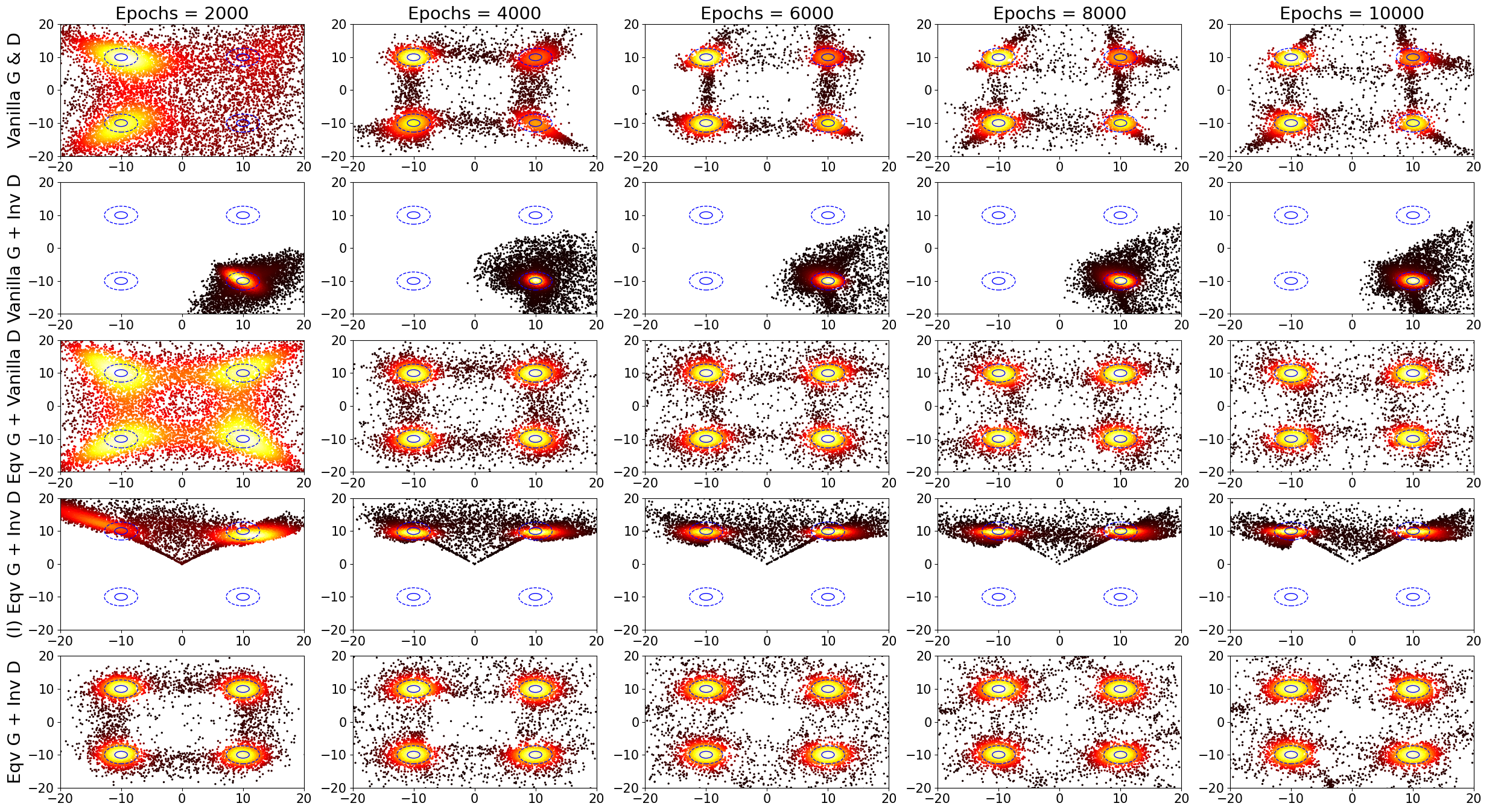}
    \caption{$D_\alpha^L$-GANs, $\alpha = 10$.}
    \label{fig:toy_200_alpha_10_2d}
  \end{subfigure}
  \caption{2D projection of the $D_{\alpha}^L$-GAN generated samples onto the support plane of the source distribution $Q$ [cf. \cref{sec:toy}]. Each column shows the result after a given number of training epochs. The rows correspond to different settings for the generators and discriminators. The solid and dashed blue ovals mark the 25\% and 50\% probability regions, respectively, of the data source $Q$, while the heat-map shows the generator samples. Models are trained on \textbf{200} training points. Panel (a): $\alpha=5$. Panel (b): $\alpha=10$.}
  \label{fig:toy_alpha_2_2d_different_Nsample_app}
\end{figure}

\begin{figure}[H]
    \centering
    \includegraphics[width=.9\textwidth]{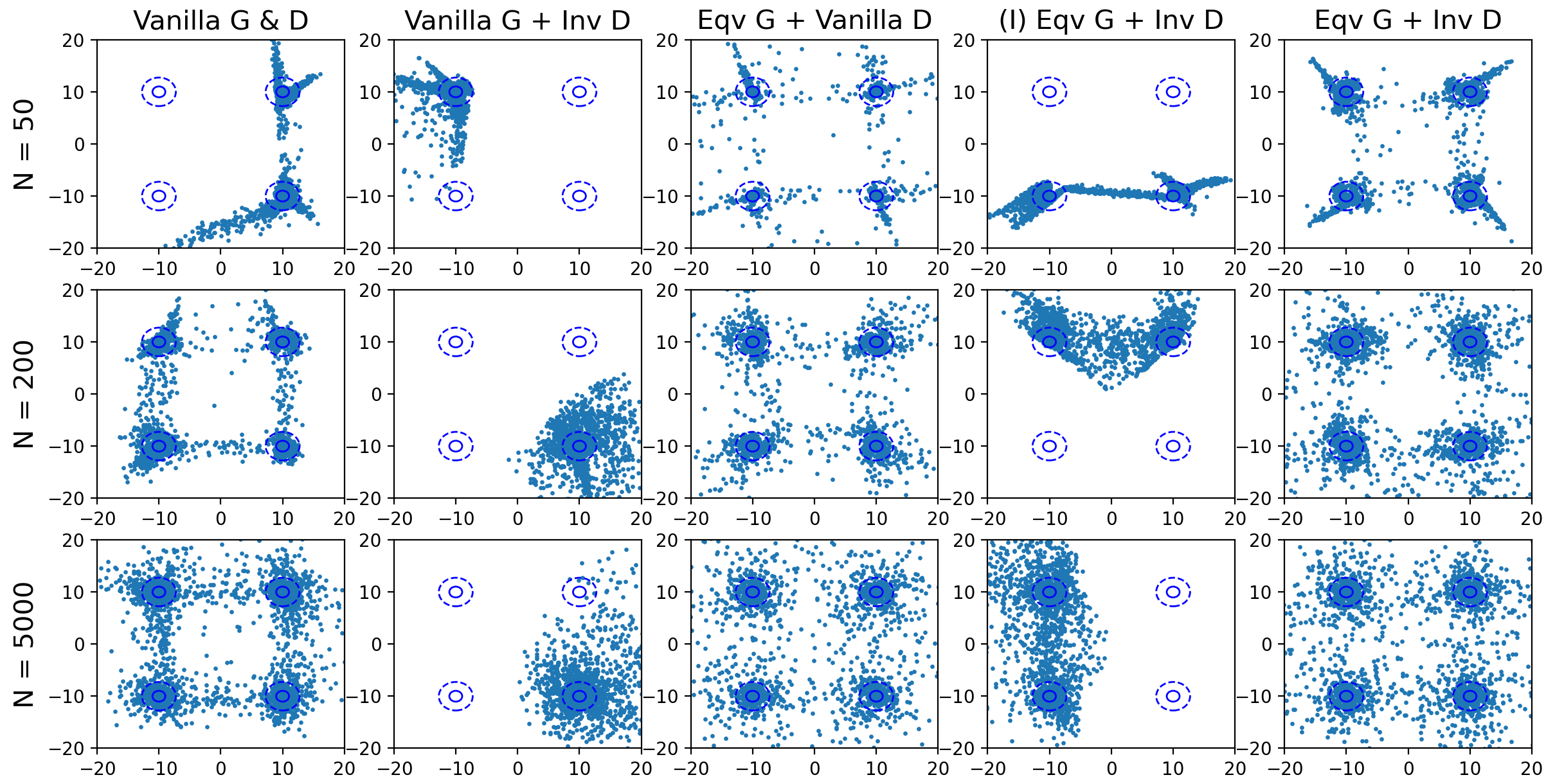}
  \caption{{ 2D projection of the $D_{2}^L$-GAN generated samples (3000 for each setting) onto the support plane of the source distribution $Q$ [cf. \cref{sec:toy}]. Each GAN is trained for 10000 epochs. The rows correspond to the number of training points $N =$ 50, 200, or 5000.
  The columns correspond to different settings for the generators and discriminators. The solid and dashed blue ovals mark the 25\% and 50\% probability regions, respectively, of the data source $Q$. Compared to Figure~\ref{fig:toy_alpha_2_2d_different_Nsample_app}, heat maps are suppressed in this figure for easier examination of the sample quality.}}
  \label{fig:real_samples}
\end{figure}

\begin{figure}[H]
    \centering
    \begin{subfigure}[b]{.26\columnwidth}
    \includegraphics[width=1\columnwidth]{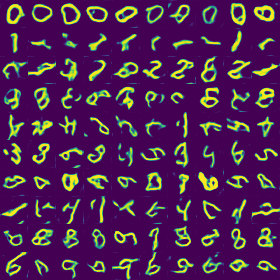}  
    \caption{\texttt{CNN G\&D}}
        
    \end{subfigure}
    ~
    \begin{subfigure}[b]{.26\columnwidth}
    \includegraphics[width=1\columnwidth]{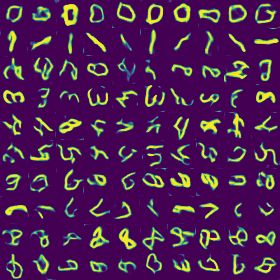}  
    \caption{\texttt{Eqv G} + \texttt{CNN D}, $\Sigma=C_4$}
        
    \end{subfigure}    
    ~
    \begin{subfigure}[b]{.26\columnwidth}
    \includegraphics[width=1\columnwidth]{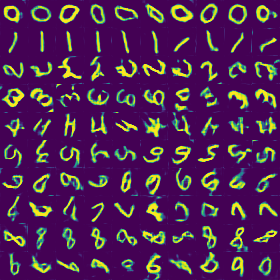}  
    \caption{\texttt{CNN G} + \texttt{Inv D}, $\Sigma=C_4$}
        
    \end{subfigure}
    ~
    \begin{subfigure}[b]{.26\columnwidth}
    \includegraphics[width=1\columnwidth]{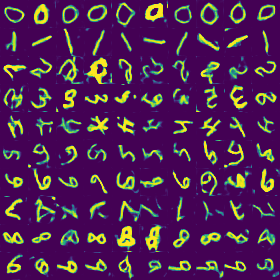}  
    \caption{\texttt{(I)Eqv G} + \texttt{Inv D}, $\Sigma=C_4$}
        
    \end{subfigure}    
    ~
    \begin{subfigure}[b]{.26\columnwidth}
    \includegraphics[width=1\columnwidth]{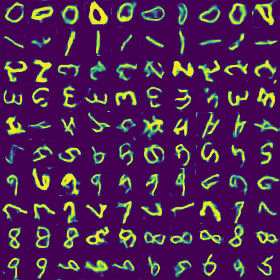}  
    \caption{\texttt{Eqv G} + \texttt{Inv D}, $\Sigma=C_4$}
        
    \end{subfigure}    
    ~
    \begin{subfigure}[b]{.26\columnwidth}
    \includegraphics[width=1\columnwidth]{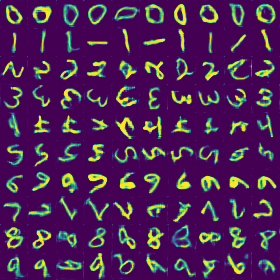}  
    \caption{\texttt{Eqv G} + \texttt{Inv D}, $\Sigma=C_8$}
        
    \end{subfigure}    
    \caption{Randomly generated digits by the $D_2^{L}$-GANs trained on RotMNIST  after 20K generator iterations with 1\% (600) training data.}
    \label{fig:mnist_alpha_random_samples_600}
\end{figure}

\begin{figure}[H]
    \centering
    \begin{subfigure}[b]{.26\columnwidth}
    \includegraphics[width=1\columnwidth]{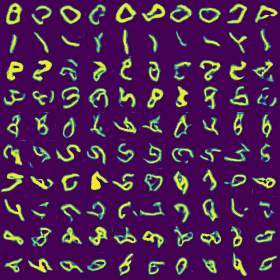}  
    \caption{\texttt{CNN G\&D}}
        
    \end{subfigure}
    ~
    \begin{subfigure}[b]{.26\columnwidth}
    \includegraphics[width=1\columnwidth]{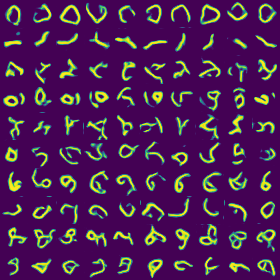}  
    \caption{\texttt{Eqv G} + \texttt{CNN D}, $\Sigma=C_4$}
        
    \end{subfigure}    
    ~
    \begin{subfigure}[b]{.26\columnwidth}
    \includegraphics[width=1\columnwidth]{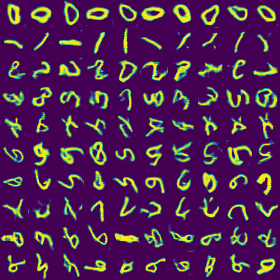}  
    \caption{\texttt{CNN G} + \texttt{Inv D}, $\Sigma=C_4$}
        
    \end{subfigure}   
    ~
    \begin{subfigure}[b]{.26\columnwidth}
    \includegraphics[width=1\columnwidth]{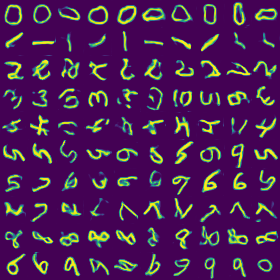}  
    \caption{\texttt{(I)Eqv G} + \texttt{Inv D}, $\Sigma=C_4$}
        
    \end{subfigure}    
    ~
    \begin{subfigure}[b]{.26\columnwidth}
    \includegraphics[width=1\columnwidth]{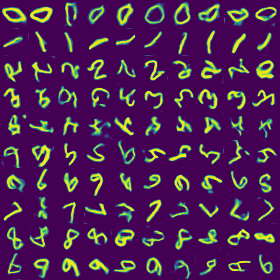}  
    \caption{\texttt{Eqv G} + \texttt{Inv D}, $\Sigma=C_4$}
        
    \end{subfigure}    
    ~
    \begin{subfigure}[b]{.26\columnwidth}
    \includegraphics[width=1\columnwidth]{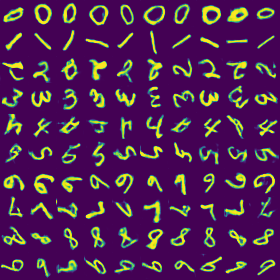}  
    \caption{\texttt{Eqv G} + \texttt{Inv D}, $\Sigma=C_8$}
        
    \end{subfigure}        
    \caption{Randomly generated digits by the RA-GANs trained on RotMNIST after 20K generator iterations with 1\% (600) training data.}
    \label{fig:mnist_rel_random_samples_600}
\end{figure}

% MNIST 200

\begin{figure}[H]
    \centering
    \begin{subfigure}[b]{.26\columnwidth}
    \includegraphics[width=1\columnwidth]{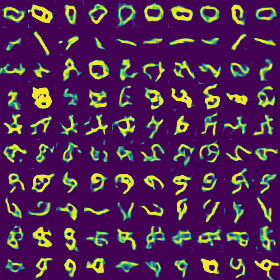}  
    \caption{\texttt{CNN G\&D}}
        
    \end{subfigure}
    ~
    \begin{subfigure}[b]{.26\columnwidth}
    \includegraphics[width=1\columnwidth]{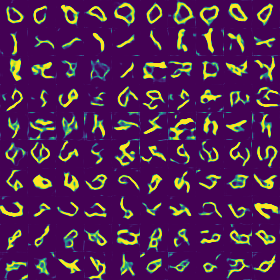}  
    \caption{\texttt{Eqv G} + \texttt{CNN D}, $\Sigma=C_4$}
        
    \end{subfigure}    
    ~
    \begin{subfigure}[b]{.26\columnwidth}
    \includegraphics[width=1\columnwidth]{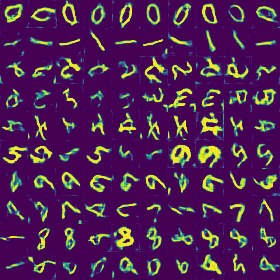}  
    \caption{\texttt{CNN G} + \texttt{Inv D}, $\Sigma=C_4$}
        
    \end{subfigure}
    ~
    \begin{subfigure}[b]{.26\columnwidth}
    \includegraphics[width=1\columnwidth]{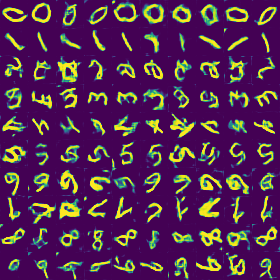}  
    \caption{\texttt{(I)Eqv G} + \texttt{Inv D}, $\Sigma=C_4$}
        
    \end{subfigure}    
    ~
    \begin{subfigure}[b]{.26\columnwidth}
    \includegraphics[width=1\columnwidth]{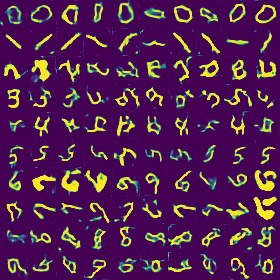}  
    \caption{\texttt{Eqv G} + \texttt{Inv D}, $\Sigma=C_4$}
        
    \end{subfigure}    
    ~
    \begin{subfigure}[b]{.26\columnwidth}
    \includegraphics[width=1\columnwidth]{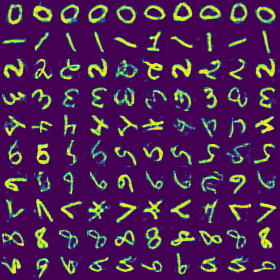}  
    \caption{\texttt{Eqv G} + \texttt{Inv D}, $\Sigma=C_8$}
        
    \end{subfigure}        
    \caption{Randomly generated digits by the $D_2^{L}$-GANs trained on RotMNIST  after 20K generator iterations with 0.33\% (200) training data. Our model \texttt{Eqv G} + \texttt{Inv D}, $\Sigma=8$ is the only one that can generate high-fidelity images in this setting. We note that the repetitively generated digits are inevitable in such a small data regime, as the models are forced to learn the empirical distribution of the limited training data (20 images per class).}
    \label{fig:mnist_alpha_random_samples_200}
\end{figure}

\begin{figure}[H]
    \centering
    \begin{subfigure}[b]{.26\columnwidth}
    \includegraphics[width=1\columnwidth]{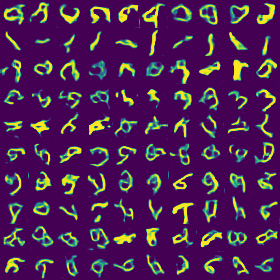}  
    \caption{\texttt{CNN G\&D}}
        
    \end{subfigure}
    ~
    \begin{subfigure}[b]{.26\columnwidth}
    \includegraphics[width=1\columnwidth]{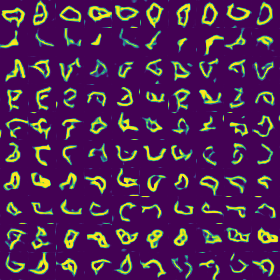}  
    \caption{\texttt{Eqv G} + \texttt{CNN D}, $\Sigma=C_4$}
        
    \end{subfigure}    
    ~
    \begin{subfigure}[b]{.26\columnwidth}
    \includegraphics[width=1\columnwidth]{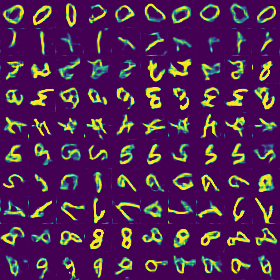}  
    \caption{\texttt{CNN G} + \texttt{Inv D}, $\Sigma=C_4$}
        
    \end{subfigure}   
    ~
    \begin{subfigure}[b]{.26\columnwidth}
    \includegraphics[width=1\columnwidth]{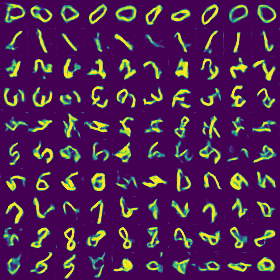}  
    \caption{\texttt{(I)Eqv G} + \texttt{Inv D}, $\Sigma=C_4$}
        
    \end{subfigure}    
    ~
    \begin{subfigure}[b]{.26\columnwidth}
    \includegraphics[width=1\columnwidth]{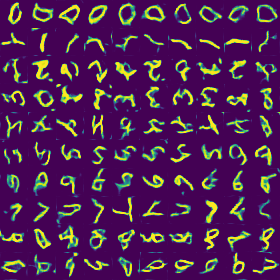}  
    \caption{\texttt{Eqv G} + \texttt{Inv D}, $\Sigma=C_4$}
        
    \end{subfigure}    
    ~
    \begin{subfigure}[b]{.26\columnwidth}
    \includegraphics[width=1\columnwidth]{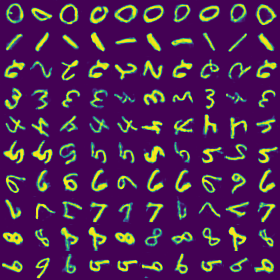}  
    \caption{\texttt{Eqv G} + \texttt{Inv D}, $\Sigma=C_8$}
        
    \end{subfigure}      
    \caption{Randomly generated digits by the RA-GANs trained on RotMNIST after 20K generator iterations with 0.33\% (200) training data.  Our model \texttt{Eqv G} + \texttt{Inv D}, $\Sigma=8$ is the only one that can generate high-fidelity images in this setting. We note that the repetitively generated digits are inevitable in such a small data regime, as the models are forced to learn the empirical distribution of the limited training data (20 images per class).}
    \label{fig:mnist_rel_random_samples_200}
\end{figure}

  \begin{table}[H]
    \caption{The median of the FIDs (lower is better), calculated every 1,000 generator update for 20,000 iterations, averaged over three independent trials. The number of the training samples used for experiments varies from 0.33\% (200) to 100\% (60,000) of the entire training set.}\label{table:3}
  \centering
  \begin{tabular}{m{1em}ccccccccc}
    \toprule
    Loss & Architecture & 0.33\%&  1\%& 5\%& 10\% & 25\% & 50\% & 100\% \\
    \midrule
    \rotatebox{90}{RA-GAN }& \makecell{\texttt{CNN G\&D} \\
    \texttt{Eqv G} + \texttt{CNN D}, $\Sigma = C_4$\\    
    \texttt{CNN G} + \texttt{Inv D}, $\Sigma = C_4$\\
    \texttt{(I)Eqv G} + \texttt{Inv D}, $\Sigma = C_4$\\
    \texttt{Eqv G} + \texttt{Inv D}, $\Sigma = C_4$ \\
    \texttt{Eqv G} + \texttt{Inv D}, $\Sigma = C_8$}&
\makecell{431\\865\\382\\360\\\textbf{190}\\313}&\makecell{295\\389\\223\\173\\\textbf{98}\\123}&\makecell{357\\333\\181\\141\\78\\\textbf{52}}&
    \makecell{348\\355\\188\\132\\89\\\textbf{51}}&
    \makecell{407\\325\\185\\124\\80\\\textbf{59}}&
    \makecell{403\\380\\177\\135\\84\\\textbf{52}}&
    \makecell{392\\393\\176\\130\\82\\\textbf{57}}
    \\
    \midrule
    \rotatebox{90}{ $D_{\alpha=2}^\Gamma$-GAN }& \makecell{\texttt{CNN G\&D} \\
    \texttt{Eqv G} + \texttt{CNN D}, $\Sigma = C_4$\\    
    \texttt{CNN G} + \texttt{Inv D}, $\Sigma = C_4$\\
    \texttt{(I)Eqv G} + \texttt{Inv D}, $\Sigma = C_4$\\
    \texttt{Eqv G} + \texttt{Inv D}, $\Sigma = C_4$ \\
    \texttt{Eqv G} + \texttt{Inv D}, $\Sigma = C_8$}&\makecell{423\\409\\511\\484\\352\\\textbf{293}}&\makecell{280\\253\\330\\273\\149\\\textbf{122}}&\makecell{261\\271\\208\\147\\99\\\textbf{55}}&\makecell{283\\251\\192\\133\\88\\\textbf{57}}&\makecell{290\\263\\190\\141\\80\\\textbf{53}}&\makecell{297\\274\\183\\124\\80\\\textbf{53}}&\makecell{293\\275\\173\\126\\81\\\textbf{51}}\\
    \bottomrule
  \end{tabular}
\end{table}

\begin{figure}[H]
    \centering
    \begin{subfigure}[b]{.48\columnwidth}
    \includegraphics[width=1\columnwidth]{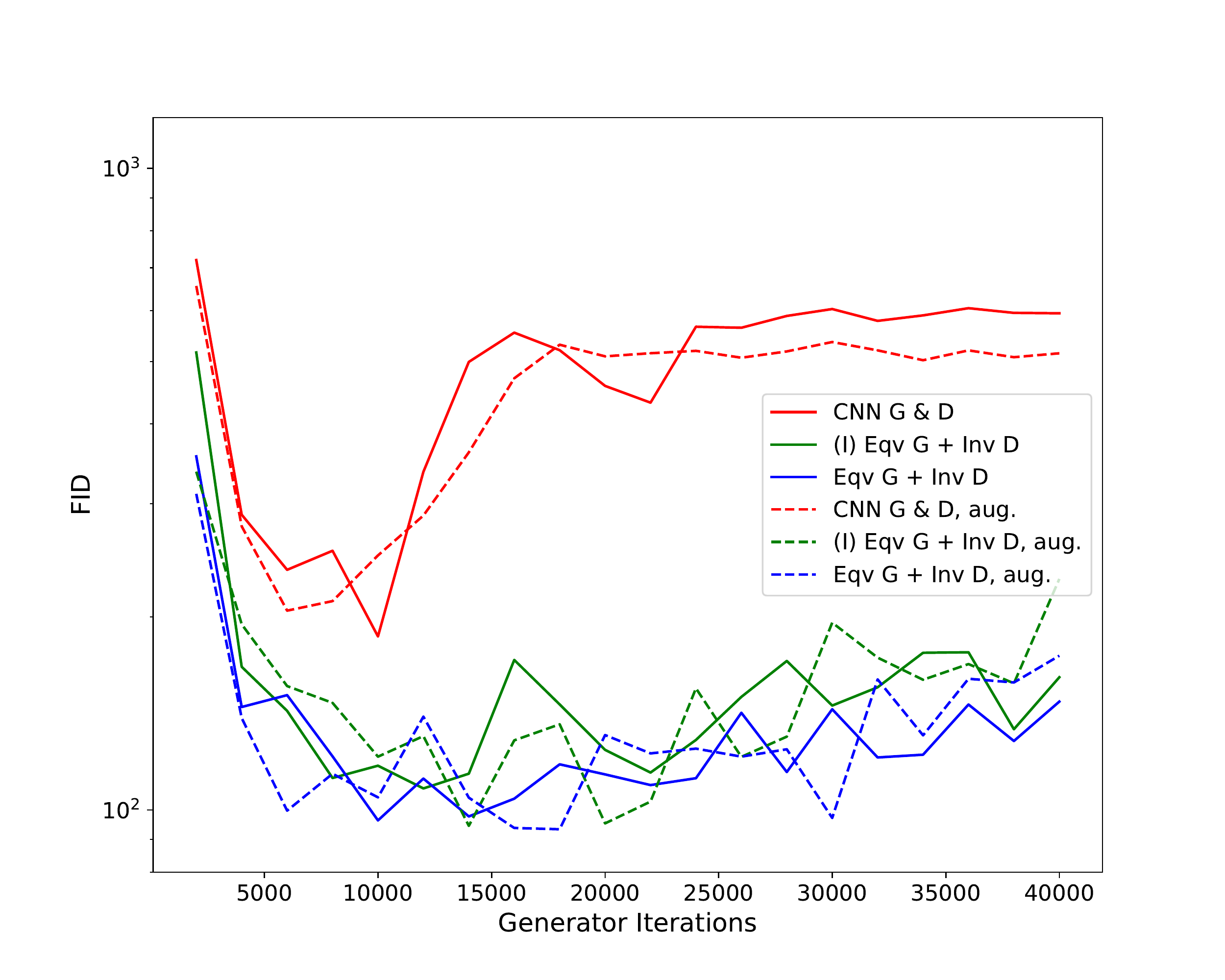}  
    \caption{ANHIR, RA-GAN}
        
    \end{subfigure}
    ~
    \begin{subfigure}[b]{.48\columnwidth}
    \includegraphics[width=1\columnwidth]{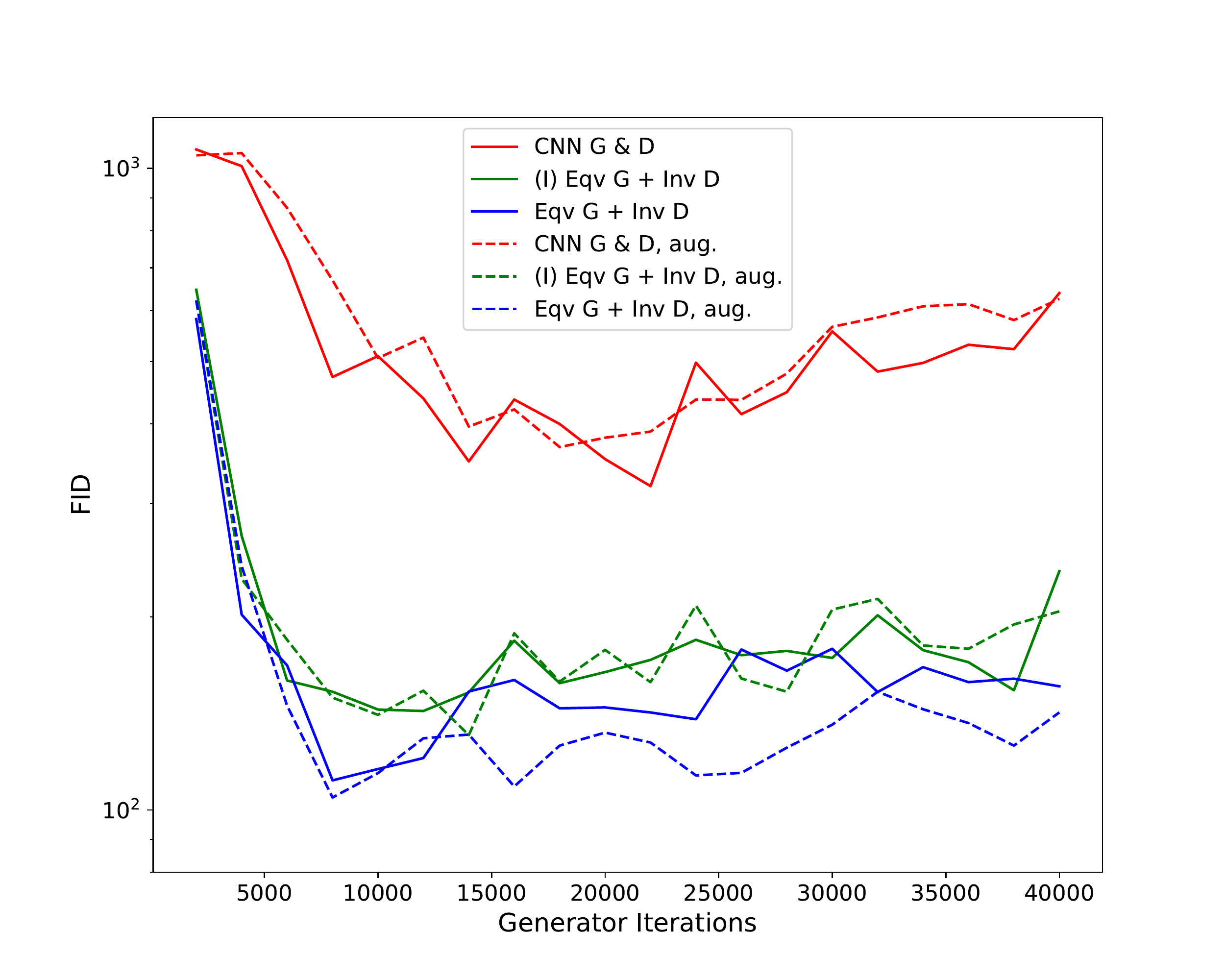}  
    \caption{ANHIR, $D_2^L$-GAN}
        
    \end{subfigure}        
    \begin{subfigure}[b]{.48\columnwidth}
    \includegraphics[width=1\columnwidth]{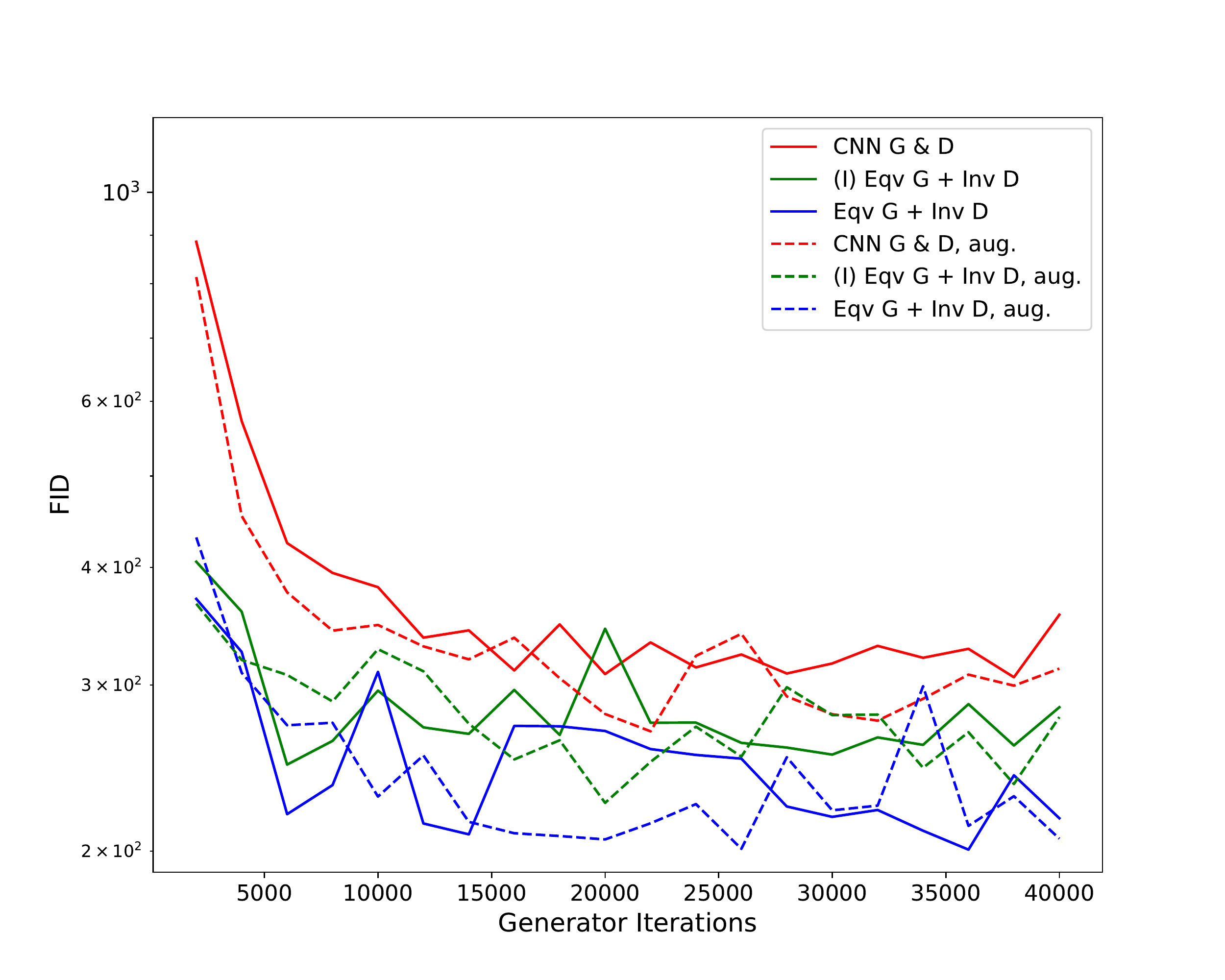}  
    \caption{LYSTO, RA-GAN}
        
    \end{subfigure}
    ~
    \begin{subfigure}[b]{.48\columnwidth}
    \includegraphics[width=1\columnwidth]{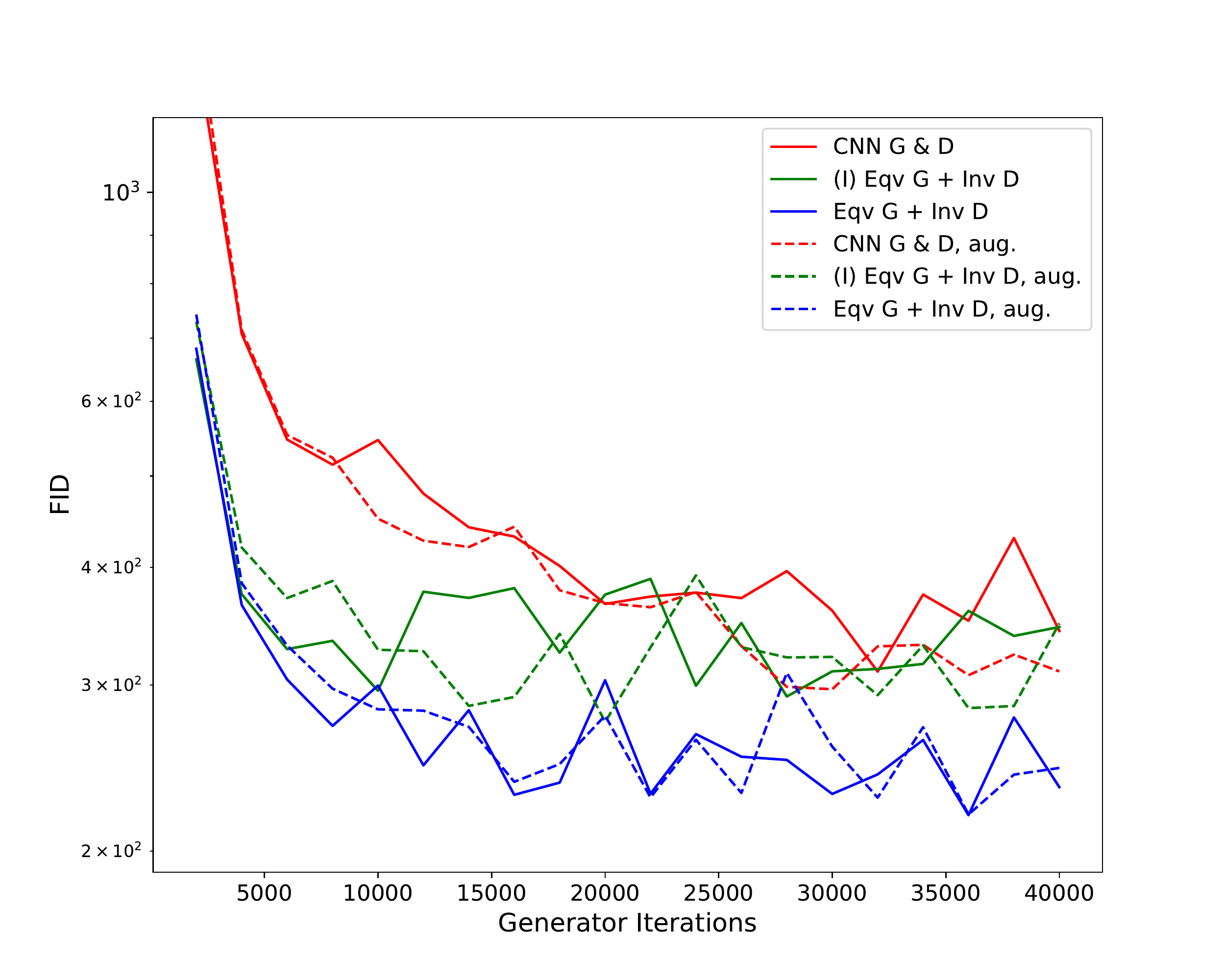}  
    \caption{LYSTO, $D_2^L$-GAN}
        
    \end{subfigure}    
    \caption{The curves of the Fr\'echet Inception Scores (FID), calculated after every 2,000 generator updates up to 40,000 iterations, averaged over three random trials on the medical data sets, ANHIR (top row) and LYSTO (bottom row). The symbol ``aug." in the legend denotes the presence of data augmentation during GAN training.}
    \label{fig:medical_fid_curve}
\end{figure}

\begin{figure}[H]
    \centering
    \includegraphics[width=.9\textwidth]{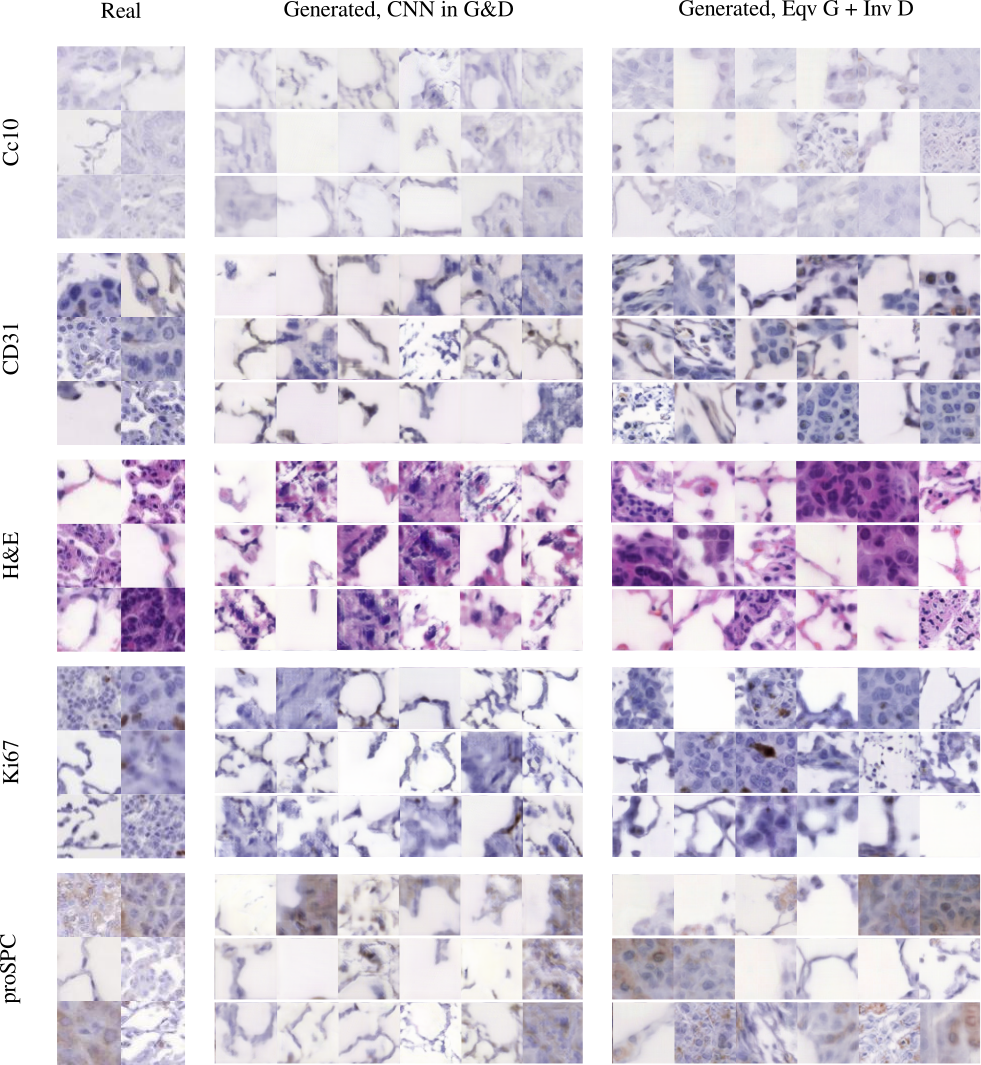}
    \caption{Real and GAN generated ANHIR images dyed with different stains. Left panel: real images. Middle and right panels: randomly selected $D_2^L$-GANs' generated samples after 40,000 generator iterations. Middle panel: \texttt{CNN G\&D}. Right panel: \texttt{Eqv G} + \texttt{Inv D}.}
    \label{fig:anhir_images_large2}
\end{figure}

\begin{figure}[H]
    \centering
    \includegraphics[width=.9\textwidth]{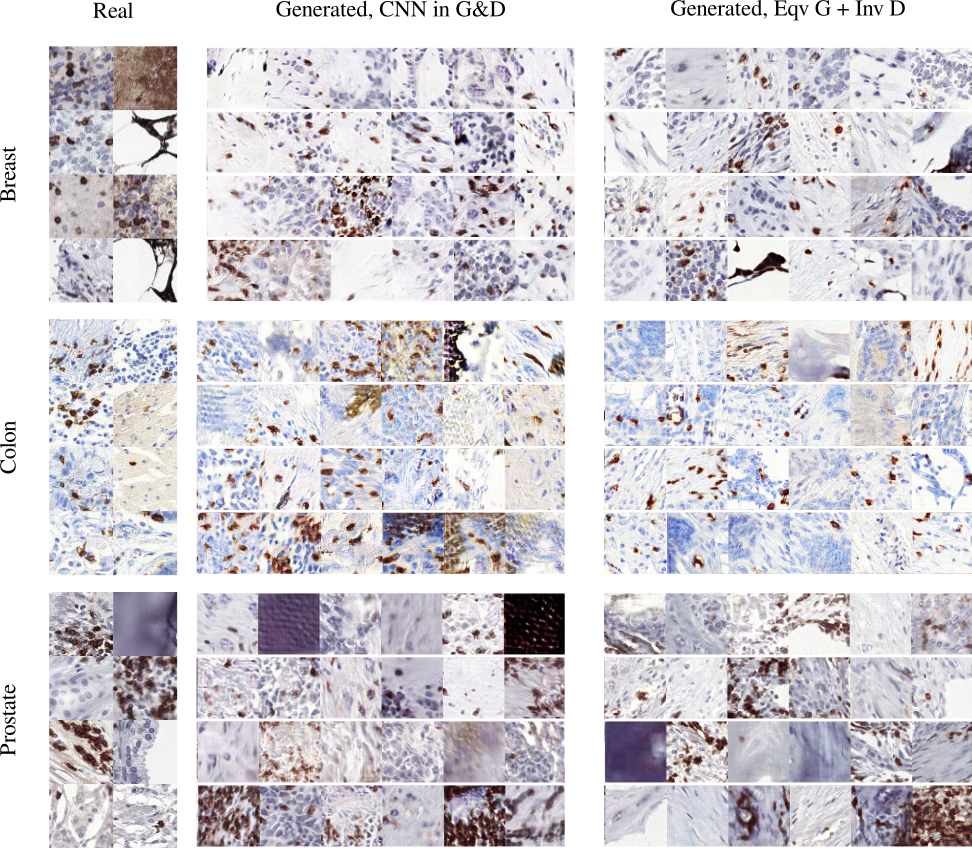}
    \caption{Real and GAN generated LYSTO images of breast, colon, and prostate cancer. Left panel: real images. Middle and right panels: randomly selected $D_2^L$-GANs' generated samples after 40,000 generator iterations. Middle panel: \texttt{CNN G\&D}. Right panel: \texttt{Eqv G} + \texttt{Inv D}.}
    \label{fig:anhir_images_large3}
\end{figure}

\begin{table}[H]
  \caption{The (min, median) of the FIDs over the course of training, averaged over three independent trials on the medical images, where the plus sign ``+" after the data set, e.g., ANHIR+, denotes the presence of data augmentation during training.
  %{\color{brown} Do we still need the (I) results here? We made that point earlier, in addition our improvements will look "cleaner" w/o (I).}
  } \label{tab:fid_medical_app}  
  \centering
  \begin{tabular}{m{1em}ccc}
    \toprule
    Loss & Architecture   & ANHIR & ANHIR+\\
    \midrule
    RA& \makecell{\texttt{CNN G\&D} \\
    \texttt{(I)Eqv G} + \texttt{Inv D} \\
    \texttt{Eqv G} + \texttt{Inv D}}
    & \makecell{(186, 523) \\
    (100, 142) \\
    \textbf{(78, 125)}}
    & \makecell{(184, 503) \\
    (88, 140) \\
    \textbf{(84, 118)}}
    \\
    \midrule
    $D_{2}^L$& \makecell{\texttt{CNN G\&D}\\
    \texttt{(I)Eqv G} + \texttt{Inv D} \\    
    \texttt{Eqv G} + \texttt{Inv D}}
    &\makecell{(313, 485)\\
    (120, 176) \\    
    \textbf{(97, 157)}}
    &\makecell{(347, 539)\\
    (119, 177) \\    
    \textbf{(90, 128)}}
    \\
    \bottomrule
    \toprule
    Loss & Architecture   & LYSTO & LYSTO+\\
    \midrule
    RA& \makecell{\texttt{CNN G\&D}  \\ 
    \texttt{(I)Eqv G} + \texttt{Inv D} \\    
    \texttt{Eqv G} + \texttt{Inv D}}
    & \makecell{(281, 340)  \\ 
    (218, 272) \\    
    \textbf{(175, 238)}}
    & \makecell{(250, 312)  \\ 
    (212, 271) \\    
    \textbf{(181, 227)}}
    \\
    \midrule
    $D_{2}^L$& \makecell{\texttt{CNN G\&D}\\
    \texttt{(I)Eqv G} + \texttt{Inv D} \\    
    \texttt{Eqv G} + \texttt{Inv D}}
    & \makecell{(289, 410)\\
    (253, 343) \\    
    \textbf{(205, 259)}}
    & \makecell{(265, 376)\\
    (244, 329)\\    
    \textbf{(192, 259)}}
    \\
    \bottomrule    
  \end{tabular}
\end{table}

\section{Implementation Details}
\label{app:implementation_details}
\subsection{Common experimental setup}
All models are trained using the Adam optimizer \cite{kingma2014adam} with  $\beta_1=0.0$ and $\beta_2=0.9$ \cite{zhang2019self}. Discriminators are updated twice after each generator update. An exponential moving average across iterations of the generator weights with $\alpha=0.9999$ is used when sampling images \cite{brock2018large}.

\subsection{RotMNIST}
For RA-GAN, the training is stabilized by regularizing the discriminator $\gamma\in \Gamma$ with a zero-centered gradient panelty (GP) on the real distribution $Q$ in the following form
\begin{align}
\label{eq:gp_r_1}
    R_1 = \frac{\lambda_1}{2}E_{x\sim Q}\|\nabla \gamma(x)\|_2^2.
\end{align}
We set the GP weight $\lambda_1=0.1$ according to \cite{EquivariantGAN}. For the $D_\alpha^L$-GAN, we use the one-sided GP as a soft constraint on the Lipschitz constant
\begin{align}
\label{eq:gp_r_2}
    R_2 = \lambda_2 E_{x\sim \rho_g}\max\{0, \|\nabla \gamma(x)\|^2-1\},
\end{align}
where $\rho_g\sim TX + (1-T)Y$ (with $X\sim P_g$, $Y\sim Q$, and  $T\sim \text{Unif}([0, 1])$ all being independent.) The one-sided GP weight is set to $\lambda_2=10$ according to \cite{Birrell:f-Gamma}. Unequal learning rates were set to $\eta_G = 0.0001$  and $\eta_D = 0.0004$ respectively. The neural architectures for the generators and discriminators are displayed in Table~\ref{tab:G_rotmnist} and Table~\ref{tab:D_rotmnist}.

\subsection{ANHIR and LYSTO}
Similar to RotMNIST, the GP weights are set to $\lambda_1=0.1$ for the RA-GAN in \eqref{eq:gp_r_1} and $\lambda_2=10$ for the $D_\alpha^L$-GAN in \eqref{eq:gp_r_2}, and we consider only the case $\alpha=2$. The learning rates were set to $\eta_G = 0.0001$  and $\eta_D = 0.0004$ respectively. ResNets instead of CNNs are used as baseline generators and discriminators, and the detailed architectural designs are specified in Table~\ref{tab:G_medical} and Table~\ref{tab:D_medical}.
\subsection{Architectures}
\label{app:architecture}
\begin{table}[ht]
    \caption{Generator architectures used in the RotMNIST experiment. ConvSN and $C_4$-ConvSN stand for spectrally-normalized 2D convolution and its $C_4$-equivariant counterpart. The incomplete attempt at building equivariant generators (\texttt{(I)Eqv G}) does not have the ``$C_4$-symmetrization" layer. The $C_8$-equivariant generator (\texttt{Eqv G}, $\Sigma=C_8$) is built by replacing ``$3\times 3$ $C_4$-ConvSN" with ``$5\times 5$ $C_8$-ConvSN" while adjusting the number of filters to maintain a similar number of trainable parameters.} 
    \label{tab:G_rotmnist}
    \vspace{1em}
    \centering
    {\renewcommand{\arraystretch}{1}
    \begin{tabular}{c}
        \toprule
         CNN Generator (\texttt{CNN G})\\
         \midrule
         Sample noise $z\in \R^{64}\sim \mathcal{N}(0,I)$\\
         Embed label class $y$ into $\hat{y}\in \R^{64}$\\
         Concatenate $z$ and $\hat{y}$ into $h\in \R^{128}$\\
         \midrule
         Project and reshape $h$ to $7\times 7\times 128$\\
        \midrule
        $3\times 3$ ConvSN, $128\to 512$\\
        \midrule
        ReLU; Up $2\times$\\
        \midrule
        $3\times 3$ ConvSN, $512\to 256$\\
        \midrule
        CCBN; ReLU; Up $2\times$\\
        \midrule
        $3\times 3$ ConvSN, $256\to 128$\\
        \midrule
        CCBN; ReLU\\
        \midrule
        $3\times 3$ ConvSN, $128\to 1$\\
        \midrule
        $\tanh()$\\
        \bottomrule
    \end{tabular}
    \hspace{2em}
    \begin{tabular}{c}
        \toprule
         $C_4$-Equivariant Generator (\texttt{Eqv G}, $\Sigma = C_4$)\\
         \midrule
         Sample noise $z\in \R^{64}\sim \mathcal{N}(0,I)$\\
         Embed label class $y$ into $\hat{y}\in \R^{64}$\\
         Concatenate $z$ and $\hat{y}$ into $h\in \R^{128}$\\
         \midrule
         Project and reshape $h$ to $7\times 7\times 128$\\
        \midrule
        $C_4$-symmetrization of $h$\\
        \midrule
        $3\times 3$ $C_4$-ConvSN, $128\to 256$\\
        \midrule
        ReLU; Up $2\times$\\
        \midrule
        $3\times 3$ $C_4$-ConvSN, $256\to 128$\\
        \midrule
        CCBN; ReLU; Up $2\times$\\
        \midrule
        $3\times 3$ $C_4$-ConvSN, $128\to 64$\\
        \midrule
        CCBN; ReLU\\
        \midrule
        $3\times 3$ $C_4$-ConvSN, $64\to 1$\\
        \midrule
        $C_4$-Max Pool\\
        \midrule
        $\tanh()$\\
        \bottomrule
    \end{tabular}
    }
\end{table}

\begin{table}[ht]
    \caption{Discriminator architectures used in the RotMNIST experiment. The $C_8$-invariant discriminator (\texttt{Inv D}, $\Sigma=C_8$) is built by replacing ``$3\times 3$ $C_4$-ConvSN" with ``$5\times 5$ $C_8$-ConvSN" while adjusting the number of filters to maintain a similar number of trainable parameters.}
    \label{tab:D_rotmnist}
    \vspace{1em}
    \centering
    {\renewcommand{\arraystretch}{1}
    \begin{tabular}{c}
        \toprule
         CNN Discriminator (\texttt{CNN D})\\
         \midrule
         Input image $x\in \R^{28\times 28\times 1}$\\
         \midrule
        $3\times 3$ ConvSN, $1\to 128$\\
        \midrule
        LeakyReLU; Avg. Pool\\
        \midrule
        $3\times 3$ ConvSN, $128\to 256$\\
        \midrule
        LeakyReLU; Avg. Pool\\
        \midrule
        $3\times 3$ ConvSN, $256\to 512$\\
        \midrule
        LeakyReLU; Avg. Pool\\
        \midrule
        Global Avg. Pool into $f$\\
        \midrule
        Embed label class $y$ into $\hat{y}'$\\
        \midrule        
        Project $(\hat{y}', f)$ into a scalar\\
        \bottomrule
    \end{tabular}
    \hspace{2em}
    \begin{tabular}{c}
        \toprule
         $C_4$-Invariant Discriminator (\texttt{Inv D}, $\Sigma = C_4$)\\
         \midrule
         Input image $x\in \R^{28\times 28\times 1}$\\
         \midrule
        $3\times 3$ $C_4$-ConvSN, $1\to 64$\\
        \midrule
        LeakyReLU; Avg. Pool\\
        \midrule
        $3\times 3$ $C_4$-ConvSN, $64\to 128$\\
        \midrule
        LeakyReLU; Avg. Pool\\
        \midrule
        $3\times 3$ $C_4$-ConvSN, $128\to 256$\\
        \midrule
        LeakyReLU; Avg. Pool\\
        \midrule
        $C_4$-Max Pool\\
        \midrule
        Global Avg. Pool into $f$\\
        \midrule        
        Embed label class $y$ into $\hat{y}'$\\
        \midrule        
        Project $(\hat{y}', f)$ into a scalar\\
        \bottomrule
    \end{tabular}    
    }
\end{table}

\begin{table}[ht]
    \caption{Generator architectures used in the ANHIR and LYSTO experiments. The generator residual block (ResBlockG) is a cascade of [CCBN, ReLU, Up $2\times$, $3\times 3$ ConvSN, CCBN, ReLU, $3\times 3$ ConvSN] with a short connection consisting of [Up $2\times$, $1\times 1$ ConvSN]. The equivariant residual block ($D_4$-ResBlockG) is built by replacing each component with its equivariant counterpart. The incomplete attempt at building equivariant generators (\texttt{(I)Eqv G}) does not have the ``$D_4$-symmetrization" layer.} 
    \label{tab:G_medical}
    \vspace{1em}
    \centering
    {\renewcommand{\arraystretch}{1}
    \begin{tabular}{c}
        \toprule
         CNN Generator (\texttt{CNN G})\\
         \midrule
         Sample noise $z\in \R^{128}\sim \mathcal{N}(0,I)$\\
         Embed label class $y$ into $\hat{y}\in \R^{128}$\\
         Concatenate $z$ and $\hat{y}$ into $h\in \R^{256}$\\
         \midrule
         Project and reshape $h$ to $4\times 4\times 128$\\
        \midrule
        ResBlockG, $128\to 256$\\
        \midrule
        ResBlockG, $256\to 128$\\
        \midrule
        ResBlockG, $128\to 64$\\
        \midrule
        ResBlockG, $64\to 32$\\        
        \midrule
        ResBlockG, $32\to 16$\\
        \midrule
        BN; ReLU\\
        \midrule
        $3\times 3$ ConvSN, $16\to 3$\\
        \midrule
        $\tanh()$\\
        \bottomrule
    \end{tabular}
    \hspace{2em}
    \begin{tabular}{c}
        \toprule
         Equivariant Generator (\texttt{Eqv G})\\
         \midrule
         Sample noise $z\in \R^{128}\sim \mathcal{N}(0,I)$\\
         Embed label class $y$ into $\hat{y}\in \R^{128}$\\
         Concatenate $z$ and $\hat{y}$ into $h\in \R^{256}$\\
         \midrule
         Project and reshape $h$ to $4\times 4\times 128$\\
        \midrule
        $D_4$-symmetrization of $h$\\
        \midrule
        $D_4$-ResBlockG, $128\to 90$\\
        \midrule
        $D_4$-ResBlockG, $90\to 45$\\
        \midrule
        $D_4$-ResBlockG, $45\to 22$\\
        \midrule
        $D_4$-ResBlockG, $22\to 11$\\        
        \midrule
        $D_4$-ResBlockG, $11\to 5$\\
        \midrule
        $D_4$-BN; ReLU\\
        \midrule
        $3\times 3$ $D_4$-ConvSN, $5\to 3$\\
        \midrule
        $D_4$-Max Pool\\
        \midrule
        $\tanh()$\\
        \bottomrule
    \end{tabular}
    }
\end{table}

\begin{table}[ht]
    \caption{Discriminator architectures used in the ANHIR and LYSTO experiments. The discriminator residual block (ResBlockD) is a cascade of [ReLU, $3\times 3$ ConvSN, ReLU, $3\times 3$ ConvSN, Max Pool] with a short connection consisting of [$1\times 1$ ConvSN, Max Pool]. The equivariant residual block ($D_4$-ResBlockD) is built by replacing each component with its equivariant counterpart.} 
    \label{tab:D_medical}
    \vspace{1em}
    \centering
    {\renewcommand{\arraystretch}{1}
    \begin{tabular}{c}
        \toprule
        CNN Discriminator (\texttt{CNN D})\\
        \midrule
        Input image $x\in \R^{64\times 64\times 3}$\\         
        \midrule
        ResBlockD, $3\to 16$\\
        \midrule
        ResBlockD, $16\to 32$\\
        \midrule
        ResBlockD, $32\to 64$\\
        \midrule
        ResBlockD, $64\to 128$\\        
        \midrule
        ResBlockD, $128\to 256$\\
        \midrule
        ReLU\\
        \midrule
        Global Avg. Pool into $f$\\
        \midrule
        Embed label class $y$ into $\hat{y}'$\\
        \midrule        
        Project $(\hat{y}', f)$ into a scalar\\
        \bottomrule
    \end{tabular}
    \hspace{2em}
    \begin{tabular}{c}
        \toprule
        Invariant Discriminator (\texttt{Inv D})\\
        \midrule
        Input image $x\in \R^{64\times 64\times 3}$\\         
        \midrule
        $D_4$-ResBlockD, $3\to 5$\\
        \midrule
        $D_4$-ResBlockD, $5\to 11$\\
        \midrule
        $D_4$-ResBlockD, $11\to 22$\\
        \midrule
        $D_4$-ResBlockD, $22\to 45$\\        
        \midrule
        $D_4$-ResBlockD, $45\to 90$\\
        \midrule
        ReLU\\
        \midrule
        $D_4$-Max Pool\\
        \midrule
        Global Avg. Pool into $f$\\
        \midrule
        Embed label class $y$ into $\hat{y}'$\\
        \midrule        
        Project $(\hat{y}', f)$ into a scalar\\
        \bottomrule
    \end{tabular}    
    }
\end{table}

%%%%%%%%%%%%%%%%%

%%%%%%%%%%%%%%%%%%%%%%%%%%%%%%%%%%%%%%%%%%%%%%%%%%%%%%%%%%%%%%%%%%%%%%%%%%%%%%%
%%%%%%%%%%%%%%%%%%%%%%%%%%%%%%%%%%%%%%%%%%%%%%%%%%%%%%%%%%%%%%%%%%%%%%%%%%%%%%%

\end{document}